%% file: main.tex
\documentclass{article}

\usepackage{microtype}
\usepackage{graphicx}
\usepackage{booktabs}
\input{pkgs}

\input{macros}

\let\ln\log
\usepackage{hyperref}

\usepackage[accepted]{icml2021}

\icmltitlerunning{Private Alternating Least Squares}

\begin{document}

\twocolumn[
\icmltitle{Private Alternating Least Squares: \\Practical Private Matrix Completion with Tighter Rates}
\icmlsetsymbol{equal}{*}

\begin{icmlauthorlist}
\icmlauthor{Steve Chien}{gooR}
\icmlauthor{Prateek Jain}{gooR}
\icmlauthor{Walid Krichene}{equal,gooR}
\icmlauthor{Steffen Rendle}{gooR}
\\
\icmlauthor{Shuang Song}{gooR}
\icmlauthor{Abhradeep Thakurta}{equal,gooR}
\icmlauthor{Li Zhang}{gooR}
\end{icmlauthorlist}

\icmlaffiliation{gooR}{Google Research}
\icmlcorrespondingauthor{Walid Krichene}{walidk@google.com}
\icmlcorrespondingauthor{Abhradeep Thakurta}{athakurta@google.com}

\icmlkeywords{Alternating Least Squares, Differential Privacy, Matrix Completion}
\vskip 0.3in
]


\printAffiliationsAndNotice{}  

\begin{abstract}
We study the problem of differentially private (DP) matrix completion under user-level privacy. We design a joint differentially private variant of the popular Alternating-Least-Squares (ALS) method that achieves: i) (nearly) optimal sample complexity for matrix completion (in terms of number of items, users), and ii) the best known privacy/utility trade-off both theoretically, as well as on benchmark data sets. In particular, we provide the first global convergence analysis of ALS with {\em noise} introduced to ensure DP, and show that, in comparison to the best known alternative (the Private Frank-Wolfe algorithm by \citet{jain2018differentially}), our error bounds scale significantly better with respect to the number of items and users, which is critical in practical problems. Extensive validation on standard benchmarks demonstrate that the algorithm, in combination with carefully designed sampling procedures, is significantly more accurate than existing techniques, thus promising to be the first practical DP embedding model.
\end{abstract}
\setlength{\textfloatsep}{16pt}
\input{introduction}
\input{background}
\input{privateALS}
\input{utility}

\input{practical}
\input{empirical}

\section{Conclusion}
We presented \dpals for solving low-rank matrix completion with user-level privacy protection. We show that \dpals provably converges to high accuracy outputs under standard assumptions and, with careful implementation, significantly outperforms existing privacy preserving matrix completion methods. In fact, \dpals achieves competitive metrics on benchmark data compared to non-private models and scales well with data set size.

The efficiency of \dpals shows that by taking advantage of the structure of the problem, one can achieve a much higher utility for privacy-preserving model training.
In this case, the alternating structure of ALS, along with the decoupling of the least squares solution, were essential in the design of an efficient method. These insights may be applicable to a broader class of problems and optimization algorithms.

\section*{Acknowledgments}
We would like to thank Om Thakkar and the anonymous reviewers for insightful comments and discussion.

\normalsize
\bibliography{reference}
\bibliographystyle{icml2021}
\appendix
\onecolumn
\input{app_proof}

\input{app_expt}

\end{document}

%% file: pkgs.tex
\usepackage{amsmath,amssymb,pifont}
\usepackage[algo2e,ruled,vlined]{algorithm2e}
\usepackage{multicol}
\usepackage{amstext}
\usepackage{amsthm}
\usepackage{multirow}
\usepackage{booktabs}
\usepackage{caption}
\usepackage{subcaption}
\usepackage{times}
\usepackage{lipsum}
\usepackage[shortlabels]{enumitem}
\usepackage{cancel}
\usepackage{wrapfig}
\usepackage{array}
\usepackage{siunitx}
\usepackage{csvsimple}
\usepackage[multidot]{grffile}
\usepackage{dblfloatfix}
\usepackage{hyperref}
\usepackage{makecell}
\usepackage{dsfont}
\usepackage{mathtools}
\usepackage{xcolor}
\usepackage{comment}

\usepackage[multiple]{footmisc}
\usepackage{mathrsfs}
\usepackage{todonotes}
\usepackage{tikz}

\usepackage{booktabs}
\usepackage{makecell}

%% file: macros.tex
\hypersetup{final}

\newcommand{\bolde}{\ensuremath{\boldsymbol{e}}}

\newcommand{\boldg}{\ensuremath{\boldsymbol{g}}}
\newcommand{\boldh}{\ensuremath{\boldsymbol{h}}}

\newcommand{\boldu}{\ensuremath{\boldsymbol{u}}}
\newcommand{\boldv}{\ensuremath{\boldsymbol{v}}}
\newcommand{\boldw}{\ensuremath{\boldsymbol{w}}}

\newcommand{\boldy}{\ensuremath{\boldsymbol{y}}}
\newcommand{\boldz}{\ensuremath{\boldsymbol{z}}}

\newcommand{\boldA}{\ensuremath{\boldsymbol{A}}}
\newcommand{\boldB}{\ensuremath{\boldsymbol{B}}}
\newcommand{\boldC}{\ensuremath{\boldsymbol{C}}}
\newcommand{\boldD}{\ensuremath{\boldsymbol{D}}}
\newcommand{\boldE}{\ensuremath{\boldsymbol{E}}}

\newcommand{\boldG}{\ensuremath{\boldsymbol{G}}}
\newcommand{\boldH}{\ensuremath{\boldsymbol{H}}}
\newcommand{\boldI}{\ensuremath{\boldsymbol{I}}}

\newcommand{\boldK}{\ensuremath{\boldsymbol{K}}}

\newcommand{\boldM}{\ensuremath{\boldsymbol{M}}}
\newcommand{\boldN}{\ensuremath{\boldsymbol{N}}}

\newcommand{\bR}{\ensuremath{\boldsymbol{R}}}

\newcommand{\boldU}{\ensuremath{\boldsymbol{U}}}
\newcommand{\boldV}{\ensuremath{\boldsymbol{V}}}
\newcommand{\boldW}{\ensuremath{\boldsymbol{W}}}
\newcommand{\boldX}{\ensuremath{\boldsymbol{X}}}

\newcommand{\calA}{\ensuremath{\mathcal{A}}}

\newcommand{\calN}{\ensuremath{\mathcal{N}}}

\newcommand{\calS}{\ensuremath{\mathcal{S}}}

\renewcommand{\Pr}{\mathop{\mathbf{Pr}}}

\newcommand{\E}{\mathop{\mathbf{E}}}

\newcommand{\R}{\mathbb{R}}

\newtheorem{lem}{Lemma}

\newtheorem{thm}[lem]{Theorem}

\newtheorem{claim}[lem]{Claim}
\newtheorem{defn}{Definition}
\newtheorem{assumption}[defn]{Assumption}
\newtheorem{definition}[defn]{Definition}
\makeatletter
\def\th@remark{%
  \thm@headfont{\bfseries}%
  \normalfont 
  \thm@preskip\topsep \divide\thm@preskip\tw@
  \thm@postskip\thm@preskip
}
\makeatother
\theoremstyle{remark}
\newtheorem{remark}{Remark}

\DeclareMathOperator*{\argmin}{arg\,min}

\makeatletter
\newcommand{\vast}{\bBigg@{4}}
\newcommand{\Vast}{\bBigg@{5}}
\makeatother

\newcommand{\ltwo}[1]{\left\|#1\right\|_2}
\newcommand{\lfrob}[1]{\left\|#1\right\|_F}

\newcommand{\snorm}[2]{\left\|#1\right\|_{#2}}

\DeclarePairedDelimiterX{\infdivx}[2]{(}{)}{%
  #1\;\delimsize\|\;#2%
}

\newcommand{\mypar}[1]{
	\noindent{\textbf{{#1}.}}}
	
\renewcommand{\epsilon}{\varepsilon}

\renewcommand{\tilde}{\widetilde}

\newcommand{\prA}{{\sf P}_\Omega}

\newcommand{\OmegaTrain}{\Omega^{\sf train}}
\newcommand{\OmegaValid}{\Omega^{\sf valid}}
\newcommand{\OmegaTest}{\Omega^{\sf test}}
\newcommand{\OmegaTestQuery}{\Omega^{\sf test\ query}}
\newcommand{\OmegaTestTarget}{\Omega^{\sf test\ target}}
\newcommand{\prApp}{{\sf P}_{\Omega''}}
\newcommand{\prAtest}{{\sf P}_{\Omega^{\sf test}}}

\newcommand{\hU}{\widehat{\boldU}}
\newcommand{\hV}{\widehat{\boldV}}

\newcommand{\alocal}{\calA_{\sf user}\,}
\newcommand{\aglobal}{\calA_{\sf item}\,}
\newcommand{\rclip}{\Gamma_{\boldu}}
\newcommand{\vclip}{\Gamma_{\boldM}}

\newcommand{\clip}[2]{{\sf clip}\left(#1,#2\right)}
\newcommand{\reg}{\lambda}
\newcommand{\aaltmin}{\calA_{\textsf{DPALS}}}

\newcommand{\av}{\tilde m}
\newcommand{\counts}{\ensuremath{\boldsymbol{c}}}
\newcommand{\countsP}{\ensuremath{\boldsymbol{\tilde c}}}

\newcommand{\head}{{\sf Frequent}\,}
\newcommand{\tail}{{\sf Infrequent}\,}

\newcommand{\dom}{\tau}

\newcommand{\Vt}{\boldV^{t}}
\newcommand{\Vht}{\widehat{\boldV}^{t}}
\newcommand{\Vh}{\widehat{\boldV}}
\newcommand{\Vto}{\boldV^{t+1}}
\newcommand{\Vhto}{\widehat{\boldV}^{t+1}}
\newcommand{\vhto}{\widehat{\boldv}^{t+1}}
\newcommand{\Ut}{\boldU^{t}}
\newcommand{\Uht}{\widehat{\boldU}^{t}}

\newcommand{\Uh}{\widehat{\boldU}}

\newcommand{\Vtot}{(\boldV^{t+1})^\top}
\newcommand{\Vhtot}{(\widehat{\boldV}^{t+1})^\top}
\newcommand{\Utt}{(\boldU^{t})^\top}
\newcommand{\Uhtt}{(\widehat{\boldU}^{t})^\top}

\newcommand{\Vo}{\boldV^*}
\newcommand{\Vot}{(\boldV^*)^\top}
\newcommand{\So}{\boldsymbol\Sigma^*}
\newcommand{\Uo}{\boldU^*}
\newcommand{\Uot}{(\boldU^*)^\top}

\newcommand{\so}{\sigma^*}
\newcommand{\vto}{\boldv^{t+1}}
\renewcommand{\u}{\boldu}
\newcommand{\ut}{\boldu^t}
\newcommand{\vt}{\boldv^t}
\newcommand{\uto}{\boldu^{t+1}}

\newcommand{\uo}{\boldu^* }
\newcommand{\uot}{(\boldu^*)^\top }
\newcommand{\vo}{\boldv^* }
\renewcommand{\v}{\boldv}
\newcommand{\vot}{(\boldv^*)^\top }

\newcommand{\uht}{\widehat{\boldu}^t}

\newcommand{\psd}[1]{\Pi_{\sf PSD}\left(#1\right)}
\newcommand{\ko}{\kappa}
\def\R{\mathbb{R}}
\def\Rnm{\R^{n\times m}}
\def\Rnr{\R^{n\times r}}
\def\Rmr{\R^{m\times r}}

\def\err{\mathrm{Err}}
\def\dpals{DPALS\xspace}
\def\dpfw{DPFW\xspace}
\long\def\commented#1{}

\newcommand{\wz}{\widetilde{\boldz}}
\newcommand{\projm}{\prA(\boldM)}
\newcommand{\cenm}{\boldN}
\newcommand{\wv}{\widetilde{\boldv}}
\newcommand{\cP}{\mathcal{P}}
\DeclareMathOperator{\sign}{sign}

%% file: introduction.tex
\section{Introduction}
\label{sec:intro}
Given $\boldM_{ij}, (i,j)\in \Omega$ where $\Omega\subseteq [n]\times [m]$ is a set of observed user-item ratings, and assuming $\boldM\approx \Uo \Vot \in \mathbb{R}^{n\times m}$ to be a nearly low-rank matrix, the goal of low-rank matrix completion (LRMC) is to efficiently learn  $\hU\in \mathbb{R}^{n\times r}$ and $\hV\in \mathbb{R}^{m\times r}$, such that $\boldM\approx \hU\hV{}^\top$.

LRMC, a.k.a. matrix factorization, is a cornerstone technique for building recommendation systems~\cite{koren,hu2008ials}, and though proposed over a decade ago, it remains highly competitive~\cite{rendle2019baselines}. In the recommendation setting, $\boldM$ represents a mostly unknown user-item ratings matrix and $\hU$ and $\hV$ capture the user and item embeddings. Using the learned ($\hU$, $\hV$), the system computes rating predictions $\widehat \boldM_{ij}=(\hU\hV{}^\top)_{ij}$ to recommend items for the users. To ensure good generalization, one would set the rank $r\ll \min(m,n)$. 

Such models, while highly successful in practice, have the risk of leaking users' ratings through  model parameters or their recommendations. The privacy risk of similar models has been well documented, and the protection against it has been intensively studied~\cite{dinur2003revealing,dwork2007price,korolova2010privacy,calandrino2011you,shokri2017membership,carlini2019secret,carlini2020attack,carlini2020extracting,thakkar2020understanding}. In this paper, we focus on learning user and item embeddings, and consequently user-item recommendations, while ensuring privacy of users' ratings.

We conform to the well-established formal notion of differential privacy (DP)~\cite{ODO,DMNS} to protect users' ratings. 
We operate in the setting of \emph{user-level} privacy~\cite{dwork2014algorithmic,jain2018differentially}, where we intend to protect \emph{all the ratings by the user}, a much harder task than protecting a single rating from the user (a.k.a. \emph{entry-level privacy})~\cite{hardt2013beyond,meng2018personalized}. 
Note that \emph{user-level} privacy is critical in this problem, as the ratings from a single user tend to be correlated and can thus be used to fingerprint a user~\cite{calandrino2011you}. 
As is standard in the user-level privacy literature~\cite{jain2018differentially}, we estimate the shared item embeddings $\hV$ while preserving privacy with respect to the users.
In contrast, each user {\em independently} computes their embedding (a row of $\hU$) as a function of their own ratings and the privacy preserving item embeddings $\hV$. 
Formally, this setup is called \emph{joint differential privacy}~\cite{kearns2014mechanism}, and it is well-established~\cite{hardt2012beating,hardt2013beyond} that such a relaxation is necessary to learn non-trivial recommendations while ensuring user-level privacy.

While several works have studied LRMC under joint-differential privacy \cite{mcsherry2009differentially, liu2015fast, jain2018differentially}, most of the existing techniques do not provide satisfactory empirical performance compared to the state-of-the-art (SOTA) non-private LRMC methods. Furthermore, these works either lack a rigorous performance analysis \cite{mcsherry2009differentially, liu2015fast} or provide  guarantees that are significantly weaker \citep{jain2018differentially} than that of non-private LRMC algorithms. Matrix factorization can also be solved using other first-order methods such as stochastic gradient descent~\cite{ge2016matrix} or alternating gradient descent~\cite{lu19pagd}, so one may apply the differentially private SGD (DPSGD) algorithm~\cite{song2013stochastic,BST14,abadi2016deep} to achieve privacy. However, applying DPSGD to LRMC is challenging as SGD typically requires many steps to converge, thus increasing privacy cost.

In this work, we design and analyze a differentially private version of the widely used alternating least squares (ALS) algorithm for LRMC \cite{koren2009matrix,jain2013low}. ALS alternates between optimizing over the user embeddings $\hU$ and the item embeddings $\hV$, each through least squares minimization. One important property of ALS is that when solving for one side, the optimization can be done independently for each user or item, which makes ALS highly scalable. Our key insight is that this decoupling of the solution is also useful for privacy-preserving computation, since there is no accumulation of noise when solving for the embeddings of different users (or items). Besides, ALS is known to require few iterations to converge in practice, making it particularly suitable for privacy preserving~LRMC.

Indeed, we present a differentially private variant of ALS, which we refer to as \dpals, and demonstrate that it enjoys much tighter error rates (see Table~\ref{tab:my_label}) and better empirical performance than the current SOTA, the differentially private Frank-Wolfe (\dpfw) method of \citet{jain2018differentially}. Furthermore, on the large scale benchmark of MovieLens 20M, \dpals produces the first realistic DP embedding model with competitive recall metric under moderate privacy loss.

\input{contrib}

%% file: contrib.tex
More specifically, our contributions are the following.

\mypar{Private alternating least squares for matrix completion} We provide the first differentially private version of alternating least squares (\dpals) for matrix completion with user-level privacy guarantee (Section~\ref{sec:privAlt}). The algorithm is conceptually simple, efficient, and highly scalable. We provide rigorous analysis on its privacy guarantee under the notion of Joint R\'enyi Differential Privacy.

\mypar{Initizlization via noisy power iteration} For convergence of DPALS algorithm, we need it to be initialized with a $\hV^{0}$ close to $\boldV^*$ in spectral norm. The standard approach based on private PCA~\cite{dwork2014analyze} would require $n=\widetilde\Omega\left(\frac{m\sqrt m}{\epsilon}\right)$ to achieve the initialization condition. Instead, we show that with a careful analysis, initializing with noisy power iteration only requires $n=\widetilde\Omega\left(\frac{m}{\epsilon}\right)$. Our analysis shows in particular that it suffices that the top-$r$ eigenspace of $\boldA := \prA(\boldM)^\top\prA(\boldM)$ be incoherent, and that there be a $\Omega(\log^2 m)$ gap between the top-$r$ eigenvalues and the rest. This result improves on~\cite{hardt2013noisy} which required all the eigenvectors of $\boldA$ to be incoherent, a condition that is hard to guarantee in our setting~\cite{dekel2011eigenvectors,vu2015random,rudelson2015delocalization}.

\mypar{Tighter privacy/utility/computation trade-offs} We prove theoretical guarantees on the sample complexity and the error bounds of \dpals under standard assumptions (Section~\ref{sec:util}). These bounds are much tighter than the current SOTA, the \dpfw method~\cite{jain2018differentially}. In particular, we show the following. First, \dpals requires only $O(\log^{O(1)} n)$ samples per user to guarantee its convergence. In contrast, \dpfw requires $\sqrt{m}$ ratings per user.
Second, to achieve a Frobenius norm error of $\zeta$, \dpals requires $n=\widetilde{\Omega}\left(\frac{m\sqrt{m}}{\zeta\epsilon}+m\right)$ users, which is nearly optimal in terms of $\zeta$ and $\epsilon$. In contrast,  DPFW's sample complexity is $n=\tilde\Omega\left(m^{5/4}/(\zeta^{5}\epsilon)\right)$; note a significant improvement in terms of $\zeta$. Finally, Private SVD~\cite{mcsherry2009differentially} is not even {\em consistent}, i.e., for a fixed $\epsilon, m, |\Omega|=n\sqrt{m}$, even if we scale $n\rightarrow \infty$, the Frobenius norm error bound does not converge to $0$ (see Theorem B.3 of \citet{jain2018differentially}).  
\begin{table}[t]
\centering
\caption{Sample complexity bounds for various algorithms, assuming constant Frobenius norm error. Here, $n$ is the number of users, $m$ is the number of items, and $\widetilde{\Omega}(\cdot)$ hides ${\sf polylog}(n,m,1/\delta)$. (*) assumes additional property of $\boldM$ being incoherent.
}
\label{tab:my_label}
\resizebox{\columnwidth}{!}{\begin{tabular}{ |c|c|c|c|} 
\hline
{\bf Algorithm} & {\bf Bound on} $n$ & {\bf Bound on} $|\Omega|/n$ & {\bf Iterations}\\
\hline
\hline
\makecell{Trace Norm (*) (non-priv.)\\\cite{candesrecht}} & $\widetilde{\Omega}(m)$ & $\widetilde{\Omega}(\log^2 n)$&  ${\sf poly}(n,m)$\\
\hline
ALS (*) (non-priv.)~\cite{jain2013low} & $\widetilde{\Omega}(m)$ & $\widetilde{\Omega}(\log^2 n)$& ${\sf polylog}(n,m)$\\ 
\hline
\hline
\makecell{Private SVD(*)\\\cite{mcsherry2009differentially}} & - & - & - \\
\hline
Private SGLD~\cite{liu2015fast} & - & - & - \\
\hline
Private FW~\cite{jain2018differentially} & $\widetilde{\Omega}(m^{5/4})$ & $\widetilde{\Omega}(\sqrt{m})$ & ${\sf poly}(n,m)$\\
\hline
{\textbf{Private ALS (*) (this work)}} & $\widetilde{\Omega}(m)$  & $\widetilde{\Omega}(\log^3 n)$ & ${\sf polylog}(n,m)$\\
\hline
\end{tabular}}
\end{table}

\mypar{Practical techniques to improve accuracy} One main difficulty in applying \dpals to practical problems comes from a heavy skew in the item distribution. We propose two heuristics to reduce the skew while preserving privacy (Section~\ref{sec:practCons}). Experiments on real-world benchmarks show that these techniques can significantly improve model quality.

\mypar{Strong empirical results using \dpals} 
We carry out an extensive study of \dpals on synthetic and real-world benchmarks. Aided by the aforementioned practical techniques, \dpals achieves significant gains over the current SOTA method. In particular, on the MovieLens 10M rating prediction benchmark, \dpals achieves the same error rate as the current SOTA even when trained on a fraction (23\%) of users. When trained on all users, it achieves a relative decrease in RMSE of at least 7\%. \dpals also achieves remarkably good performance on the MovieLens 20M item recommendation benchmark with modest privacy loss, and remains competitive even with non-private ALS, the first DP private embedding model to achieve such strong results.

%% file: background.tex
\section{Background}
\label{sec:background}

\subsection{Notation}
Let $[m]$ denote the set $\{1, 2, \cdots, m\}$. Let $\Rnm$ denote the set of $n\times m$ matrices. Throughout the paper, we use bold face uppercase letters to represent matrices and lowercase letters for vectors. For any matrix $\boldA=(\boldA_{ij})\in\Rnm$, let $\boldA_i$ be the $i$-th \emph{row} vector of $\boldA$. Denote by $\left\| \boldA \right\|_F, \left\| \boldA \right\|_\infty$ the Frobenius norm and the max norm of $\boldA$. For $\Omega\subseteq [n]\times[m]$, define the projection $\prA(\boldA)\in \Rnm$ as $\prA(\boldA)_{ij}=\boldA_{ij}$ if $(i,j)\in \Omega$ and $0$ otherwise. For $i\in [n]$, define $\Omega_i:=\{j:(i, j)\in\Omega\}$. Similarly, for $j\in[m]$, let $\Omega_j=\{i:(i,j)\in\Omega\}$. For $\boldu, \boldv\in\mathbb{R}^r$, we use $\boldu\cdot \boldv\in\mathbb{R}$ to denote their dot product, and $\boldu\otimes \boldv\in\mathbb{R}^{r\times r}$ for their outer product.

\subsection{Matrix Completion, Alternating Least Squares}\label{sec:als}

Let $\boldM\in\Rnm$ be a rank $r$ matrix, such that each entry $\boldM_{ij}$ ($i\in[n]$, $j\in[m]$) represents the preference/affinity of user $i$ for item $j$. Given a set of observed entries $\prA(\boldM)$, $\Omega\subseteq [n]\times[m]$, the goal of LRMC is to reconstruct $\boldM$ with minimal error.
This can be achieved by finding $\hU\in \Rnr$ and $\hV\in \Rmr$ such that the regularized squared error $\|\prA\big(\boldM-\hU\hV{}^\top\big)\|_F^2+\lambda\|\hU\|_F^2 + \lambda\|\hV\|_F^2$ is minimized. This minimization problem is NP-hard in general~\cite{hardt2014computational}. But the alternating least squares (ALS) algorithm has proved to work well in practice.

ALS alternatingly computes $\hU, \hV$ by minimizing the above objective while assuming the other embeddings fixed. Each step can be solved efficiently through the standard least squares algorithm with the following closed form solution.
{\begin{align}
\forall i\quad &\hU^{t}_i = (\lambda \boldI + \sum_{j\in\Omega_i} \hV^{t}_j \otimes \hV^{t}_j)^{-1} \sum_{j\in\Omega_i} \boldM_{ij} \hV^{t}_j, \label{algo:alsf1}\\
\forall j\quad &\hV^{t+1}_j = (\lambda\boldI + \sum_{i\in\Omega_j} \hU^{t}_i \otimes \hU^{t}_i)^{-1} \sum_{i\in\Omega_j} \boldM_{ij} \hU^{t}_i. \label{algo:alsf2}
\end{align}}%
While ALS does not guarantee convergence to the global optimum in general, it works remarkably well in practice and often produces $\hU$ and $\hV$ such that $\hU \hV{}^\top$ is a good approximation of $\boldM$. 
The practical success of ALS has inspired many theoretical analyses, which make the following additional assumptions on $\boldM$ and $\Omega$.
\begin{assumption}[$\mu$-incoherence]\label{def:incoherence}
Let $\boldM = \Uo \So \Vot$ be the singular value decomposition of $\boldM$, i.e. $\Uo\in\Rnr, \Vo\in\Rmr$ are orthonormal matrices, and $\So\in\mathbb{R}^{r\times r}$ is the diagonal matrix of the singular values of $\boldM$.\\
We assume that $\boldM$ is $\mu$-incoherent, that is, $\forall i\in[n]$, $\ltwo{\Uo_i}\leq \frac{\mu\sqrt{r}}{\sqrt{n}}$; and $\forall j\in[m]$, $\ltwo{\Vo_j}\leq\frac{\mu\sqrt{r}}{\sqrt{m}}$.
\end{assumption}

\begin{assumption}[Random $\Omega$]\label{def:random}
We assume that $\Omega$ are random observations with probability $p$, that is, $\Omega=\{(i,j) \in [n]\times[m] : \delta_{ij}=1\}$, where $\delta_{ij}\in\{0,1\}$ are i.i.d. random variables with $\Pr[\delta_{ij}=1]=p$.
\end{assumption}

\citet{jain2013low,hardt2014fast} showed that ALS converges to $\boldM$ with high probability if $\boldM$ is $\mu$-incoherent and $p=\tilde{\Omega}\left(\frac{\log n}{m}\right)$, where $n\geq m$ and $\tilde{\Omega}$ hides polynomial dependence on $\mu$, $r$, and the condition number of $\boldM$.
In this work, we make the same assumptions on $\boldM$ and $\Omega$. Our key theoretical contribution is a similar convergence result for \dpals, under the additional requirements of user-level differential privacy.

\subsection{Joint Differential Privacy}
\label{sec:diffPri}

Differential privacy~\cite{DMNS,ODO} is a widely adopted privacy notion. We use the variant of \emph{user-level} joint differential privacy (Joint DP). Intuitively, Joint DP requires any information which may cross different users to be differentially private, but allows each individual user to use her own private information to her full advantage, for example, when computing the embeddings for generating recommendations to herself. This notion was already implicit in~\cite{mcsherry2009differentially} and made formal in~\cite{kearns2014mechanism,jain2018differentially}. 

Let $D=\{d_1,\ldots,d_n\}$ be a data set of $n$ records, where each sample $d_i$ is drawn from a domain $\dom$ and belongs to individual $i$ (which we also refer to as a \emph{user}). Let $\calA:\tau^*\to\calS^n$ be an algorithm that produces $n$ outputs in some space $\calS$, one for each user $i$. 
Let $D_{-i}$ be the data set with the $i$-th user removed, and let $\calA_{-i}(D)$ be the set of outputs without that of the $i$-th user. 
Also, let $(d_i;D_{-i})$ be the data set obtained by adding $d_i$ (for user $i$) to the data set $D_{-i}$. Joint DP and its R\'enyi differential privacy~\cite{mironov2017renyi} (Joint RDP) variant are defined as follows.
\begin{definition}[Joint Differential Privacy~\cite{kearns2014mechanism}]
An Algorithm $\calA$ is $(\epsilon,\delta)$-jointly differentially private if for any user $i$, for any possible value of data entry $d_i,d'_i\in\dom$, for any instantiation of the data set for other users $D_{-i}\in\dom^{n-1}$, and for any set of outputs $S\subseteq\calS^n$, the following two inequalities hold simultaneously:
\begin{small}
\begin{align*}
&\Pr_\calA\left[\calA_{-i}((d_i;D_{-i}))\in S\right]\leq e^\epsilon\Pr_\calA[\calA_{-i}(D_{-i})\in S]+\delta\\
&\Pr_\calA\left[\calA_{-i}(D_{-i})\in S\right]\leq e^\epsilon\Pr_\calA[\calA_{-i}((d_i;D_{-i}))\in S]+\delta.
\end{align*}
\end{small}%
An algorithm $\calA$ is $(\alpha,\epsilon)$-joint R\'enyi differentially private (Joint RDP) if $D_{\alpha}\left(\calA_{-i}((d_i;D_{-i}))||\calA_{-i}(D_{-i})\right)\leq \epsilon$ and $D_{\alpha}\left(\calA_{-i}(D_{-i})||\calA_{-i}((d_i;D_{-i}))\right)\leq \epsilon$, where $D_{\alpha}$ is the R\'enyi divergence of order $\alpha$. 
\label{def:diffPriv}
\end{definition}

If we replace $\calA_{-i}$ with $\calA$ in the definition, we would recover the standard definition of DP and RDP. We note that the joint DP (resp. joint RDP) enjoys the same composability properties as the notion of DP (resp. RDP).

%% file: privateALS.tex
\section{\dpals: Private Alternating Least Squares}
\label{sec:privAlt}

We now provide the details of the \dpals algorithm and prove its privacy guarantee in the joint DP model.

\begin{algorithm}[t!]
\SetAlgoLined
\DontPrintSemicolon
\SetKwProg{myproc}{Procedure}{}{}
{\bf Required}: $\prA(\boldM)$: observed ratings, $\sigma$: noise standard deviation, $\rclip$: row clipping parameter, $\vclip$: entry clipping parameter, $T$: number of steps, $\reg$: regularization parameter, $r$: rank, $k$: maximum number of ratings per user in $\aglobal$, $\hV^{0}$: initial $\hV$. \\
\nl Clip entries in $\prA(\boldM)$ so that $\|\prA(\boldM)\|_\infty \leq \vclip$ \label{line:clipentry}\\
\For{$0\leq t\leq T$}{
    \For{$1\leq i\leq n$}{
        \nl $\Uht_i\leftarrow \alocal(\hV^t, \Omega_i, \prA(\boldM)_i, T, \reg, \rclip)$
    }
\nl $\hU^{t}\leftarrow [\Uht_1, \cdots, \Uht_n]^\top$\\
\nl $\hV^{t+1} \leftarrow \aglobal(\hU^t, \Omega, \prA(\boldM), k, \reg, \rclip, \vclip)$
}
\nl \KwRet $\hU^T, \hV^T$

\medskip
\myproc{$\aglobal$($\boldU$, $\Omega$, $\prA(\boldM)$, $k$, $\reg$, $\rclip$, $\vclip$)}{
\nl $\Omega'\leftarrow$ up to $k$ random samples of $(i, j)\in\Omega$, $\forall i\in[n]$. \label{line:sample}\\
\For{$1\leq j\leq m$}{
        \nl $\boldG_j \leftarrow \mathcal{N}_{\sf sym}\left(0,\rclip^4\cdot\sigma^2\right)^{r \times r}$\label{line:noise1}\\
        \nl $\boldg_j \leftarrow \mathcal{N}\left(0,\rclip^2\vclip^2\cdot\sigma^2\right)^r$\label{line:noise2}\\
        \nl $\boldX_j \leftarrow \reg \boldI + \sum_{i\in\Omega'_j} \boldU_i \otimes \boldU_i +\boldG_j$\label{line:LHS}\\
        \nl $\boldV_j\leftarrow \psd{\boldX_j}^+\left(\sum_{i\in\Omega'_j} \boldM_{ij}\cdot\boldU_i+\boldg_j\right)$ \label{alg:step psd}\label{line:Vsol}\\
    }
	\nl $\widetilde{\boldV}=[\boldV_1, \cdots, \boldV_m]^\top$\\
    \nl \KwRet $\boldV = \widetilde{\boldV}(\widetilde{\boldV}^\top \widetilde{\boldV})^{-1/2}$
  }
  \medskip
\myproc{$\alocal$($\boldV$, $\Omega_i$, $\prA(\boldM)_i$, $T$, $\reg$, $\rclip$)}{
  \nl $\Omega_i'\leftarrow$ random samples of $1/T$ fraction of $j\in \Omega_i$\label{line:usample}\\
  \nl $\small \boldu\leftarrow (\reg \boldI + \sum_{j\in\Omega_i'} \boldV_j \otimes \boldV_j) ^{-1} \sum_{j\in\Omega_i'} \boldM_{ij} \boldV_j$\label{line:Usol}\\
 \nl \KwRet $\clip{\boldu}{\rclip}$\label{line:clip}}
  \caption{\dpals: Private Matrix Completion via Alternating Minimization}\label{Alg:alt}
\end{algorithm}

\mypar{Notation} Let $\calN(0,\sigma^2)$ be the Gaussian distribution of variance $\sigma^2$, and
$\calN_{\sf sym}(0, \sigma^2)^{r\times r}$ be the distribution of symmetric matrices where each entry in the upper triangle is drawn i.i.d. from $\calN(0,\sigma^2)$.
For a symmetric $\boldA$, let $\psd{\boldA}$ be its projection to the positive semi-definite cone, obtained by replacing its negative eigenvalues with $0$.
Define $\clip{\boldu}{c} = \boldu \cdot \max(1, c/\|\boldu\|_2)$, i.e., the projection of $\boldu$ on an $\ell_2$ ball of radius $c$.
Let $\boldA^+$ be the pseudoinverse of~$\boldA$.

\begin{figure}
\centering
\includegraphics[width=0.45\textwidth]{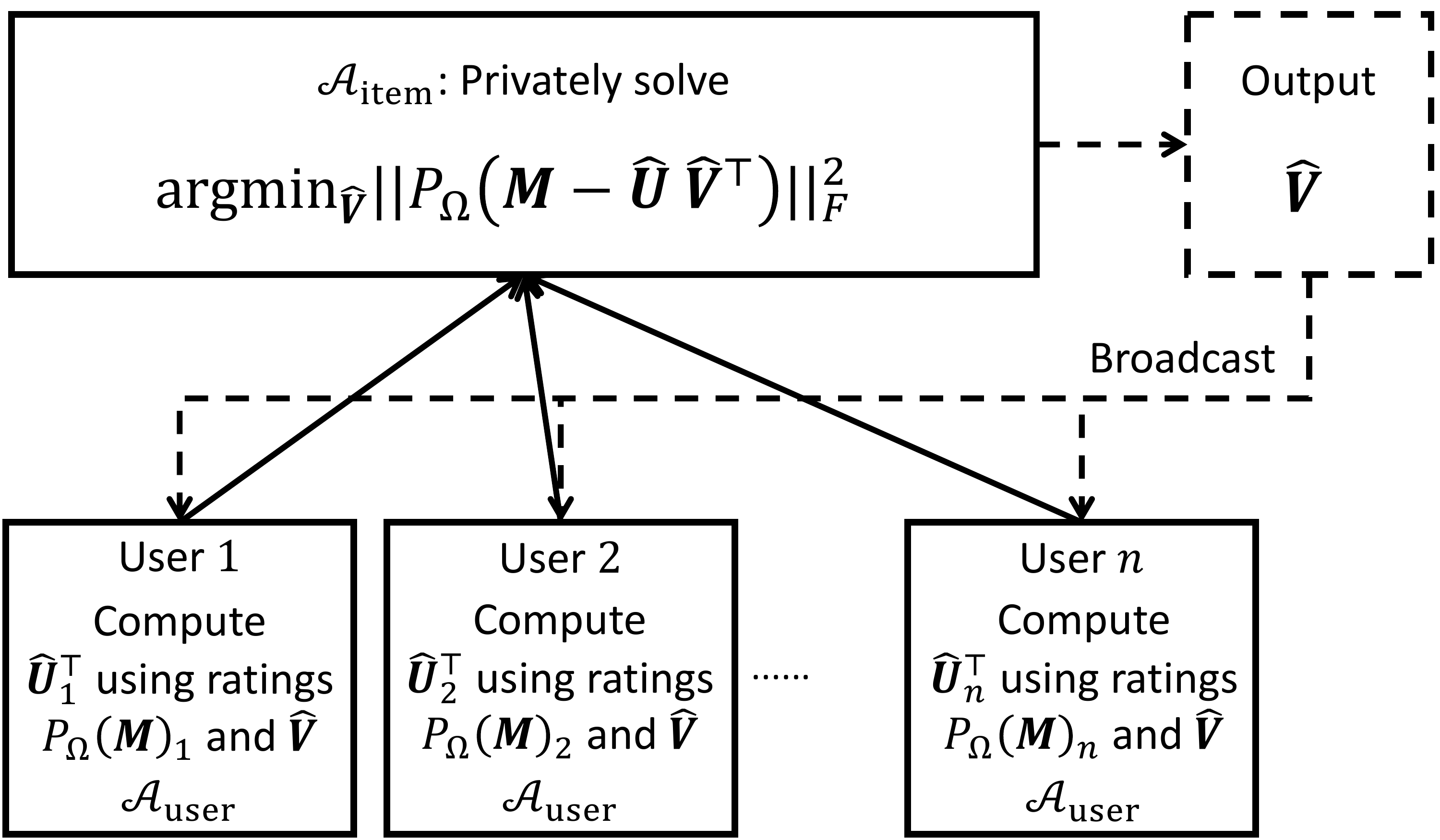}
\vspace{-0.3cm}
\caption{Block schematic of Joint differentially private alternating least squares algorithm. 
Solid lines and boxes represent privileged computations not visible to an adversary or other users.
Dashed boxes and lines are public information accessible to anyone.}
\label{fig:blockSchematic}
\end{figure}

\subsection{Algorithm}
The private alternating least squares algorithm, \dpals, is described in Algorithm~\ref{Alg:alt}.
It follows the standard ALS steps, i.e. it alternatingly solves the least squares problem to obtain $\hU$ and $\hV$ using \eqref{algo:alsf1} and \eqref{algo:alsf2}. 
To guarantee joint DP, we compute differentially private item embeddings $\hV{}^{t+1}$ (using procedure $\aglobal$) by solving a private variant of~\eqref{algo:alsf2}, and compute each row of
$\hU{}^{t+1}$ \emph{independently without any noise} using procedure $\alocal$. 
A block schematic of the algorithm is presented in Figure~\ref{fig:blockSchematic}. 

Here we describe how the privacy is guaranteed in $\aglobal$; see Theorem~\ref{thm:privG} for a formal statement. For a given $j\in[m]$, write $\boldH^t_j = \lambda \boldI + \sum_{i\in\Omega_j} \hU{}^t_i \otimes \hU{}_i^t$ and $\boldw^t_j = \sum_{i\in\Omega_j} \boldM_{ij} \hU{}^t_i$. Then the non-private update~\eqref{algo:alsf2} can be written as $\hV{}^{t+1}_j = \left(\boldH{}_j^t\right)^{-1} \boldw^t_j$. In the private version, we need to add noise to protect both $\boldH^t_j$ and $\boldw^t_j$.
To ensure sufficient noise, we limit the influence of each user by ``clipping'' each $\hU{}^t_i$ to a bounded $\ell_2$ norm $\rclip$ (Line~\ref{line:clip} in $\alocal$) and resampling $\Omega$ such that each user participates in at most~$k$ items' computation (Line~\ref{line:sample} in $\aglobal$). 
We then apply the Gaussian mechanism to $\boldH^t_j$ and $\boldw^t_j$ before using them to compute $\hV{}^{t+1}_j$ in $\aglobal$ (Lines 9--12).

While the above procedure is sufficient to guarantee privacy, we need a few additional modifications for utility analysis.

\mypar{Initialization} Random initialization has worked well for our empirical study. For the utility analysis, we need $\hV^0$ to be reasonably close to $\Vo$ (in terms of spectral norm). We show that by using the noisy power iteration process, we are able to obtain $\hV^0$ within the required bound with $n$ almost linear in $m$, an improvement compared to private PCA~\cite{dwork2014analyze}, which would require $n=\widetilde{\Omega}(m\sqrt{m})$. See Section~\ref{sec:initialization} for details.

\mypar{Sampling from $\Omega$} To ease the analysis, we require the observed values to be independent across different steps. This is achieved by resampling from $\Omega$ at the beginning of $\aglobal$ (Line~\ref{line:sample}) and $\alocal$ (Line~\ref{line:usample}). The sampling in $\aglobal$ is more important as it also limits the number of items per user, for privacy purposes. In practice, we omit the sampling in $\alocal$, and sample only once for $\aglobal$. The sampling distribution used in the latter has a significant impact in practice, as discussed in Section~\ref{sec:practCons-1}.

\subsection{Computational Complexity}
The computational complexity of \dpals is comparable to that of ALS. More precisely, the $V$ step of ALS involves computing $\boldH^t_j$ and $\boldw^t_j$, in $O(|\Omega'|r^2)$, then solving the $m$ linear systems $\hV{}^{t+1}_j = \left(\boldH{}_j^t\right)^{-1} \boldw^t_j$ in $O(mr^3)$, for a total complexity of $O(|\Omega'|r^2 + mr^3)$ (and similarly for the $U$ step). This scales linearly in the number of observations $|\Omega'|$ and the number of items $m$. In the private version ($\alocal$), the only additional operations are forming the noise matrices (Lines~\ref{line:noise1}--\ref{line:noise2}) in $O(mr^2)$, and projecting $\boldX_j$ (Line~\ref{line:Vsol}), in $O(mr^3)$, so the complexity per iteration is the same as ALS.
In comparison, the complexity of DPFW is $O(\Gamma(m + |\Omega'|))$, using Oja's method. The per-iteration complexity also scales linearly in $m$ and $|\Omega'|$. Even though the per-iteration complexity of DPFW and \dpals are comparable, \dpals converges in much fewer iterations (see Appendix~\ref{app:emp-experiments} for an example), which makes it more scalable in practice.

\subsection{Privacy Guarantee}
We now provide the privacy guarantee for \dpals. As each subroutine in \dpals is a variant of the Gaussian mechanism, we can apply the R\'enyi accounting~\cite{mironov2017renyi} and convert to $(\epsilon,\delta)$-DP. See Appendix~\ref{app:priv} for the proof.
\begin{thm}[Privacy guarantee]\label{thm:privG}
Excluding the initialization of $\hV^0$, 
Algorithm~\ref{Alg:alt} is $\left(\alpha,\alpha\rho^2\right)$-joint RDP with $\rho^2=\frac{kT}{2\sigma^2}$. Hence for any $\epsilon > 0$ and $\delta \in (0, 1)$, Algorithm~\ref{Alg:alt} is $(\epsilon, \delta)$-joint DP if we set $\sigma=\frac{\sqrt{(2kT)(\epsilon + \ln(1/\delta))}}{\epsilon}$. 

\end{thm}

The guarantee holds for all values of the parameters $\rclip$, $\vclip$, $T$, $\reg$, $r$, $k$. Note in particular that the scale of the noise (Lines~\ref{line:noise1}--\ref{line:noise2} in Algorithm~\ref{Alg:alt}) is normalized so that the expression of $\sigma$ in Theorem~\ref{thm:privG} does not depend on $\rclip$, $\vclip$.

In the above guarantee, we have excluded the initialization process. With random initialization, which we use in practice, there is no extra privacy cost. However, for the utility guarantee, we need an extra DP procedure, detailed in Section~\ref{sec:initialization}, so that $\hV^0$ is sufficiently close to $\Vo$. This can be done within the same privacy bound as the main procedure.

%% file: utility.tex
\section{Convergence Guarantee for \dpals}
\label{sec:util}

We now show that under standard low-rank matrix completion assumptions (Assumptions~\ref{def:incoherence} and~\ref{def:random}), Algorithm~\ref{Alg:alt}, initialized with the noisy power method, solves the matrix completion problem. We will first present the results assuming that $\hV^0$ is close to $\Vo$. We will then present the guarantee of the initialization procedure.

\begin{thm}[Utility guarantee]\label{thm:util}
Suppose that $\boldM$ is a $\mu$-incoherent rank-$r$ matrix, and $\Omega$ consists of random observations with probability $p$. Let $\sigma_1^\ast \geq \cdots \sigma_r^\ast>0$ be the singular values of $\boldM$ and $\kappa:=\sigma_1^*/\sigma_r^*$ its condition number.

There exists a universal constant $C>0$, such that for all $\delta \in (0, 1)$, $\epsilon \in (0,  \ln(1/\delta))$, if $p\geq \mu^{6}\kappa^{12} r^6 \cdot \frac{\log^3 n}{m}$ and $\sqrt{p}n\geq C\frac{\gamma\sqrt{\log (1/\delta)}}{\epsilon}$,  where $\gamma=C\ko^6\mu^3r^2\sqrt{m}\cdot \log^2 (\ko\cdot n)$, then \dpals, initialized with $\hV^{0}$ s.t. $\|(I-\Vo(\Vo)^\top)\hV^{0}\|\leq \frac{C}{\ko^2 r^2 \ln n}$, with parameters $k=C\cdot m\cdot p\log n$, $T=\log(\mu \kappa n/\epsilon)$, $\sigma=\frac{C \sqrt{kT \ln(1/\delta)}}{\epsilon}$, $\rclip=\frac{C \mu \sigma_1^\ast \sqrt{r}}{\sqrt{n}}$, $\vclip=\frac{\mu^2 r \sigma_1^\ast}{\sqrt{mn}}$ and $\lambda=0$, returns $\Uh^T$ and $\Vh^T$ such that the following holds: 
\setlist{nolistsep}
\begin{itemize}[leftmargin=*, noitemsep]
\item The distribution of $(\Uh^T, \Vh^T)$ satisfies $(\epsilon, \delta)$-joint DP.
\item $\|\boldM-\Uh^T (\Vh^T)^\top\|_F\leq C  \cdot \frac{\sqrt{m\log (1/\delta)}}{\epsilon\cdot n}\cdot \frac{\ko \gamma}{\sqrt{p}}\|\boldM\|_F$, with probability $\geq 1-1/n^{10}$.
\item Similarly, $\|\boldM-\Uh^T (\Vh^T)^\top\|_\infty \leq C  \cdot \frac{\sqrt{m\log(1/\delta)}}{\epsilon\cdot n}\cdot \frac{\ko \gamma}{\sqrt{p}}\cdot \frac{\mu^2 r\|M\|_2}{\sqrt{mn}}$, with probability $\geq 1-1/n^{10}$.
\end{itemize}
\label{thm:utilGM}
\end{thm}

\begin{remark} 
The choice of hyper-parameters in Theorem~\ref{thm:utilGM} assumes knowledge of certain quantities such as $r, \mu, \kappa$. In practice, these quantities are unknown, but one can use standard DP hyper-parameter search techniques~\cite{liu2019private} to search for optimal hyper-parameter values.
\end{remark}

\begin{remark} The number of samples needed per user is about $p\cdot m = O(\mu^6 \kappa^{12} r^6 \log^3 n)$ which is nearly optimal with respect to $m$ and $n$. This represents a significant improvement over the \dpfw algorithm in \cite{jain2018differentially} which requires $\Omega(\sqrt{m})$ samples per user.
\end{remark}

\begin{remark}
We did not optimize bounds for dependence on the rank $r$ and condition number $\kappa$. Prior work tends to focus on the dependence on the size ($m$ and $n$) and polynomial dependence on $r, \kappa$ is common even in the non-private setting, for example~\cite{jain2013low,sun2015guaranteed,ge2016matrix}. Our main goal is to provide a guarantee in the private setting that is competitive with the non-private setting, so we inherit the focus on the size $m, n$. Furthermore, dependence on $\kappa$ can be removed (up to log factors) by using a {\em stagewise} ALS method similar to~\cite{hardt2014fast}. However, this further complicates the proof and the practical performance of standard ALS is comparable to such stagewise methods.
\end{remark}

\begin{remark} Our Frobenius norm error bound is significantly smaller than the bound for the \dpfw algorithm, which is given by $\|\boldM-\Uh^T (\Vh^T)^\top\|_F\leq \big(\frac{m^{5/4}}{n\epsilon}\big)^{1/5}\|\boldM\|_F$. In particular, to ensure an error $\|\boldM-\Uh^T (\Vh^T)^\top\|_F \leq \zeta\|\boldM\|_F$, \dpals requires $n\geq \frac{C m}{\zeta\cdot \epsilon}$, while \dpfw requires $n\geq \frac{C m^{5/4}}{\zeta^5\cdot \epsilon}$, which is significantly worse in terms of $\zeta$. Furthermore, the DPFW bound is a generalization bound, i.e., there is an additional bias term which
can be large, and to the best of our knowledge, existing techniques (even in the non-private setting) require incoherence to control this term.
\end{remark}

\begin{remark}  Consider a set of $m$ linear regression problems in $r$-dimensions: $\big\{\boldy_{(i)}=\boldX\theta^*_{(i)}\big\}_{i=1}^m$, with $\boldX\in\mathbb{R}^{n\times r}$. One can use a single iteration of \dpals with ($\hU=\boldX$ and $\prA(\boldM)=[\boldy_{(1)},\ldots,\boldy_{(m)}]$) to solve these linear regression problems. Assuming the conditions on $\boldM$ are satisfied, we can obtain an excess empirical risk of $\widetilde{O}(\sqrt{m}/(\epsilon n))$. This matches the best known upper bound for solving a set of linear regressions with privacy~\cite{sheffet2019old,smith2017interaction}. So, a better convergence rate of \dpals would lead to a tighter bound on solving a set of linear regressions with a common feature matrix. For $m=O(1)$, we know that the lower bound for private linear regression is $\widetilde{\Omega}(1/\epsilon n)$~\cite{smith2017interaction}. Thus, we conjecture that the error for DPALS is tight w.r.t. $m$ and $\epsilon n$.
\end{remark}

\begin{remark}
Instead of using the perturbed objective function to estimate $\hV^t$ in DPALS, one can  use DPSGD~\cite{BST14} to do the same (solving a least 
squares problem with $\hU^t$ fixed). We leave the empirical comparison of this approach to future work. However, we know that for least-square losses, perturbing the objective is known to be theoretically optimal~\cite{smith2017interaction}.
\end{remark}

{\bf Proof sketch}: 
First, we show that under the assumptions in Theorem~\ref{thm:utilGM}, w.h.p., clipping and sampling operations in \dpals have no effect. Note, using $k\geq C p\cdot m \log n$, w.p. $\geq 1-1/n^{100}$, $\forall i, |\Omega_i|\leq k$. Furthermore, using Lemma~\ref{lem:incoh}, $\|\Uht_i\|\leq \rclip$. Similarly, using Lemma~\ref{lem:incoh}, $\sigma_{\min}(\boldX)\geq p/4-\|\boldG\|_2\geq p/4-\rclip^2 \sigma \sqrt{r}\geq p/8$. That is, $\boldX \succ 0$.

The above observation implies that, under the assumptions of the theorem, Algorithm~\ref{Alg:alt} is essentially performing the following iterative steps:\\
i) $\Uht=\arg\min\limits_{\Uh}\|\prA(\boldM-\Uh (\Vht)^\top)\|_F^2$, and\\
ii) 
$\Vhto_j=\Big(\boldI + \sum\limits_{i\in\Omega'_j} \Uht_i \otimes \Uht_i +\boldG \Big)^{-1} \Big(\sum\limits_{ i\in\Omega_j'} \boldM_{ij}\Uht_i+\boldg\Big)$.

Let $\Ut$ (resp. $\Vt$) be the Q part in the QR decomposition of $\Uht$ (resp. $\Vht$). Using Lemma~\ref{lem:disdec}, we get
$\err(\Vo, \Vto)\leq \frac{1}{4}\err(\Vo, \Vt)+\alpha$, 
where $\err(\Vo, \boldV)=\|(\boldI-\Vo\Vot)\boldV\|_F$ and $\alpha\leq \frac{C\ko^6\cdot \mu^3 r^2\sqrt{\log n}}{\sqrt{p}n}  \frac{\sqrt{m\log n}\cdot T\log 1/\delta}{\epsilon}$.
That is, after $T$ iterations, $\err(\Vo, \boldV^T)\leq 2\alpha$. The second claim of the theorem now follows from the above observation and Lemma~\ref{lem:incoh}. Similarly, the third claim follows by using the bound on $\err(\Vo, \boldV^T)$ and incoherence of $\boldU^T$, $\boldV^T$ (Lemma~\ref{lem:incoh}). See Appendix~\ref{app:proof} for a detailed proof.

\begin{lem}\label{lem:incoh}
Suppose the assumptions of Theorem~\ref{thm:utilGM} hold. Then, w.p. $\geq 1-5T/n^{100}$, we have: a) each iterate $\Uht$, $\Vht$ is $16\ko\mu$-incoherent, b) $1/2\leq \sigma_q(\Uht(\So)^{-1})\leq 2$ for all $q\in [r]$, c) $1/4 \leq \sigma_q(\frac{1}{p}\sum_{i:(i,j)\in \Omega^{v,t}} \uht_i (\uht_i)^\top)\leq 4$.
\end{lem}

\begin{lem}\label{lem:disdec}
Suppose the assumptions of Theorem~\ref{thm:utilGM} hold. Also, let $\Vt$ be $16 \ko \mu$-incoherent s.t. $\err(\Vo, \Vt)\leq \frac{1}{\ko^2 \log^2 n}$. Then, w.p. $\geq 1-5T/n^{100}$, we have 
$\err(\Uo, \Ut)\leq \frac{1}{2}\err(\Vo, \Vt)$, and $\err(\Vo, \Vto)\leq \frac{1}{2}\err(\Uo, \Ut)+\frac{C\ko^6\cdot \mu^3 r^2\sqrt{\log n}}{\sqrt{p}n}  \frac{\sqrt{m\log n}\cdot T\log 1/\delta}{\epsilon},$ 
where $\err(\Vo, \boldV)=\|(\boldI-\Vo\Vot)\boldV\|_F$. 
\end{lem}

\subsection{Noisy Power Iteration Initialization}\label{sec:initialization}

Theorem~\ref{thm:util} requires that \dpals be initialized with $\hV^{0}$ such that $\|(I-\Vo(\Vo)^\top)\hV^{0}\| =  O(1/\ln n)$. One may apply Algorithm 1 of \cite{dwork2014analyze}, i.e. compute the top-$r$ eigenvectors of $\boldA+\boldG$, where $\boldA :=\prA(\boldM)^\top \prA(\boldM)$ and $\boldG \sim \calN_{\sf sym}(0, \vclip^4\sigma^2)^{m\times m}$. This would require $n=\widetilde{\Omega}(m\sqrt{m}/\epsilon)$. However, this turns out to be suboptimal in our setting as it doesn't take advantage of the the sparsity of $\prA(\boldM)$. Instead, the noisy power iteration method, developed in~\cite{hardt2012beating,hardt2013beyond,hardt2013noisy} for per-entry privacy protection, turns out to be more suitable.

One difficulty in applying noisy power iteration is that prior work requires incoherence of $\boldA$, which may not hold in our setting. To overcome this difficulty, we show that it suffices to have incoherence of the top-$r$ eigenspace of $\boldA$, together with a (moderate) gap between the top eigenvalues and the rest, both of which we are able to establish. This gives a tighter analysis of noisy power iteration which may be of independent interest, detailed in Appendix~\ref{sec:inittighter}. We apply this result to our setting in the next theorem.

\begin{thm}[Initialization guarantee]
There exists constant $C_0, C_1, C_2>0$, such that for any $\delta\in(0,1), \epsilon\in (0, \log(1/\delta))$, if $p\geq C_1 \gamma \log^3 m/m$ and 
$\sqrt{p}n\geq C_2 \frac{\sqrt{\log(1/\delta)}}{\epsilon}\gamma\sqrt{m}\log^{5/2} m$, where $\gamma=(\mu\kappa r)^{C_0}$, the noisy power iteration method is $(\epsilon,\delta)$-differentially private, and with high probability returns a $\hV^0$ which is close to $\Vo$ as defined in Theorem~\ref{thm:util}.
\end{thm}

%% file: practical.tex
\section{Heuristic Improvements to \dpals}
\label{sec:practCons}

We introduce heuristics to improve the privacy/utility trade-off of \dpals in practice. We describe each heuristic, its motivation, and how to implement it differentially privately.

\subsection{Reducing Distribution Skew}
\label{sec:practCons-1}
The first heuristics are motivated by the observation that, in practice, the elements of $\Omega$ are not sampled uniformly at random~\cite{marlin07collaborative}. In particular, the number of observed ratings per item typically follows a power-law distribution, and is heavily skewed towards popular items. For example, Figure~\ref{fig:ml10m_count_dist} shows the fraction of observations vs. fraction of top movies in the MovieLens 10M data set. It shows, for instance, that the top 20\% of the movies account for more than 85\% of the observations.
\begin{figure}[t]
    \centering
    \includegraphics[width=0.31\textwidth]{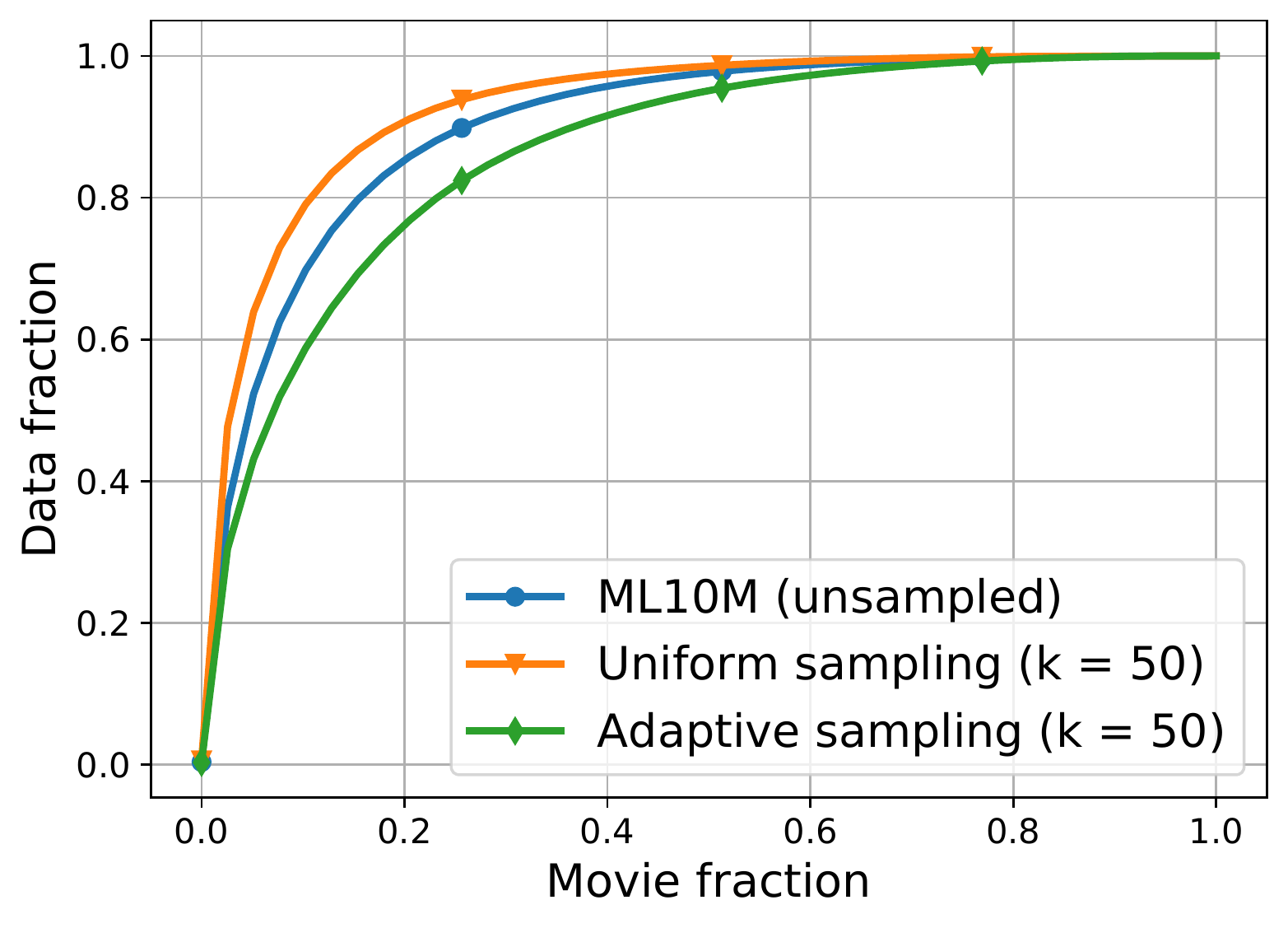}\vspace*{-10pt}
    \caption{Fraction of observations contributed by the top movies in MovieLens 10M. Adaptive sampling reduces popularity bias.}
    \label{fig:ml10m_count_dist}
\end{figure}

Due to this popularity bias, some items may have very few observations, and for such rare items $j$, the embedding $\boldV_j$ learned by \dpals may not be useful: The noise terms in Line~\ref{line:Vsol} of Algorithm~\ref{Alg:alt} do not scale with the number of observations $|\Omega'_j|$ -- for otherwise we may lose the protection on users who rated rare items -- thus, items with a smaller $|\Omega'_j|$ have a lower signal-to-noise ratio. In our experiments, we found that such noisy embeddings may have a further cascading effect and lead to quality degradation in the embeddings of other movies and users. To alleviate this issue, we propose two techniques.

\mypar{Learning on frequent items}
The first strategy is to partition the items into two sets, based on an estimate of the item counts, which we denote by $\countsP \in \mathbb R^n$. We introduce a hyper-parameter $\beta$ representing the fraction of movies to train on. Define the set $\head$ to be the $\lceil m\beta \rceil$ items with the largest $\countsP$, and let $\tail$ be its complement. We learn embeddings $\hV_j$ only for $j \in \head$, by running Algorithm $\aaltmin$ on those items. When making predictions for any missing entry $\boldM_{ij}$, if $j \in \head$, we use the dot product $\hU_i \cdot \hV_j$, and if $j \in \tail$ we use the average observed rating of $\prA(\boldM)_i$.

To compute $\countsP$ privately, notice that since each user contributes at most $k$ items, the exact item count $\counts$ has $\ell_2$ sensitivity $\sqrt{k}$. Thus, $\countsP := \counts + \calN(0,k\sigma^2)$ guarantees $\left(\alpha,{\alpha}/{2\sigma^2}\right)$-RDP.

\mypar{Adaptive sampling}
To further reduce the popularity bias, we propose to use an adaptive distribution when sub-sampling $\Omega$. Recall that in Line~\ref{line:sample} of $\aglobal$, we pick $k$ items per user in $\Omega$, in order to limit the privacy loss. We propose to sample rare items with higher probability, as follows. Given the count estimate $\countsP$, for each user $i$, we pick the $k$ items in $\Omega_i \cap \head$ with the lowest count estimates. This heuristic effectively reduces the distribution skew and gives a significant utility gain compared to uniform sampling, see Section~\ref{sec:practCons-3}. Figure~\ref{fig:ml10m_count_dist} illustrates the resulting distribution for a sample size of $k = 50$ per user. It's interesting to observe that under uniform sampling, the popularity bias is worse than in the unsampled data set, this is due to a negative correlation between user counts and item counts: conditioned on a light user, the probability to observe a rare item is lower; see Appendix~\ref{app:heuristics} for further discussion.

\subsection{Additional Heuristics}\label{sec:practCons-2}
A common heuristic, used for example by~\cite{mcsherry2009differentially}, is to center the observed matrix $\prA(\boldM)$, by subtracting an estimate of the global average, denoted by $\av$. To compute $\av$ privately, since $\|\prA(\boldM)\|_\infty \leq \vclip$ and each user contributes at most $k$ items,
publishing $\av=\frac{\sum_{(i,j)\in \Omega}M_{ij}+\calN(0,k\vclip^2\sigma^2)}{|\Omega|+\calN(0,k\sigma^2)}$ guarantees $\left(\alpha,\alpha/\sigma^2\right)$-RDP.

Another practice, commonly used in some benchmarks, is to modify the loss function in Section~\ref{sec:als} by adding the term $\lambda_0 \|\hU \hV{}^\top\|_F^2$, where $\lambda_0$ is a hyper-parameter. This is particularly important for item recommendation tasks, such as the MovieLens 20M benchmark. This modification introduces an additional term $\boldK := \lambda_0 \sum_{i \in [n]} \hU_i\otimes\hU_i$ to $\boldX$ in Line~\ref{line:LHS} of $\aglobal$. To maintain privacy, we use a noisy version $\tilde \boldK$ obtained by adding Gaussian noise to $\boldK$. Since $\boldK$ is independent of $j$, we reuse the same $\tilde \boldK$ for all $j\in[m]$, thus limiting the additional privacy loss due to this term.

Finally, we project the matrix $\boldX_j = \boldH_j + \boldG_j$ to the PSD cone (Line \ref{line:LHS}) to improve stability. In our analysis, we show that $\boldX_j$ is positive definite with high probability, but in practice, the projection improves performance.

We account for the privacy cost in the computation of $\av$,~$\countsP$, and $\tilde \boldK$, along with that in Theorem~\ref{thm:privG}, by standard composition properties of RDP~\cite{mironov2017renyi}. For completeness, the privacy accounting of the full algorithm including data pre-processing, is given in Appendix~\ref{app:heuristics}.

%% file: empirical.tex
\def\colwidth{0.3\textwidth}
\begin{figure*}[h!]
\centering
\includegraphics[width=0.95\textwidth]{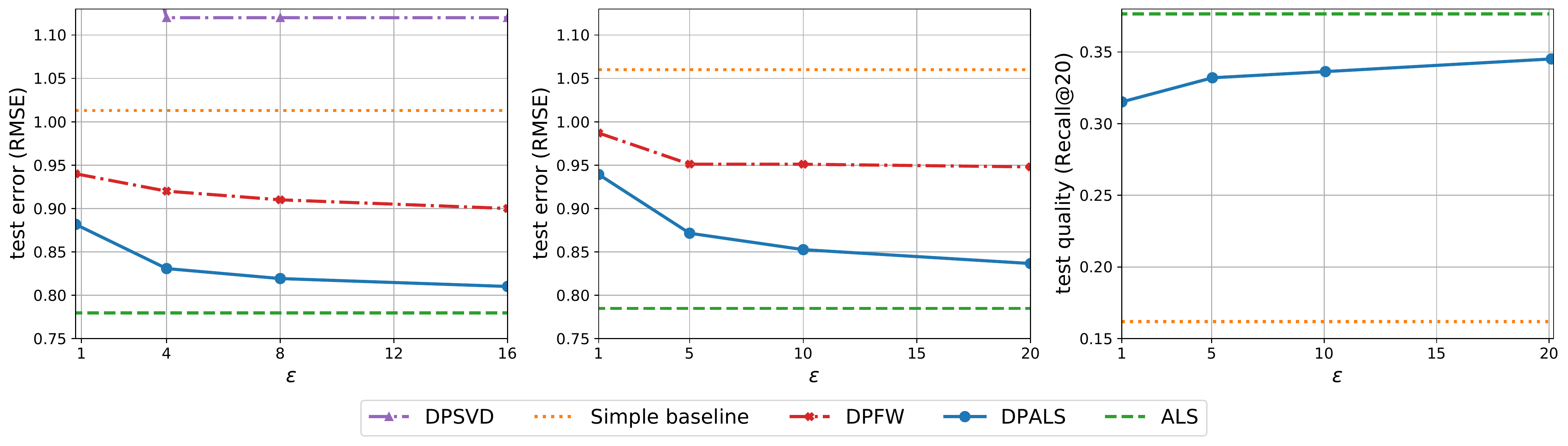}\vspace*{5pt}
\begin{subfigure}[b]{\colwidth}
\caption{ML-10M (top 400 movies)}
\label{fig:tradeoff_ml10mtop}
\end{subfigure}%
\begin{subfigure}[b]{\colwidth}
\caption{ML-10M}
\label{fig:tradeoff_ml10m}
\end{subfigure}%
\begin{subfigure}[b]{\colwidth}
\caption{ML-20M}
\label{fig:tradeoff_ml20m}
\end{subfigure}\vspace*{-10pt}
\caption{Privacy/utility trade-off of different methods. We observe that \dpals is significantly more accurate than \dpfw method, and the loss in accuracy for \dpals compared to ALS is  relatively small, especially for $\epsilon\geq 10$.}
\label{fig:tradeoff}
\end{figure*}

\section{Empirical Evaluation}
\label{sec:empEval}

We run experiments on synthetic data and two benchmark tasks on the widely used MovieLens data sets~\cite{harper16movielens}. The synthetic task follows the assumptions of our theoretical analysis, and serves to illustrate the guarantees of Theorem~\ref{thm:utilGM}. The MovieLens benchmark tasks serve as an evaluation of the empirical privacy/utility trade-off on a more realistic application, and to provide some practical insights into \dpals. We use current SOTA method \dpfw as the main baseline as it is already demonstrated to be more accurate than techniques like Private SVD  \cite{mcsherry2009differentially}. Similar to \cite{jain2018differentially}, we do not compare against  \cite{liu2019private} as the privacy parameters are unclear, and might require (exponential time) Markov chain based sampling methods to compute them.

\subsection{Metrics and Data Sets}
\mypar{Metrics} The quality of a learned model $(\hU, \hV)$ will be measured either using the RMSE or the Recall@k, depending on the benchmark. The RMSE is defined as $\text{RMSE} = \|\prAtest(\hU\hV{}^\top - \boldM)\|_F/\sqrt{|\Omega^{\text{test}}|}$, where $\Omega^{\sf test}$ is the set of test ratings held out from $\Omega$. Recall@k is defined as follows. For each user $i$, let $R_i$ be the set of $k$ movies with the highest scores, where the score of movie $j$ is $\hU_i \cdot \hV_j$. Then $\text{Recall@k} = \frac{1}{n}\sum_{i = 1}^n |R_i \cap \Omega^{\sf test}_i| / \min(k, |\Omega^{\sf test}_i|)$.

\mypar{Synthetic data} We generate a rank $5$ ground truth matrix as the product of two random orthogonal matrices $\Uo \in \mathbb R^{n\times 5}, \Vo \in \mathbb R^{m \times 5}$, where $m = 1000$, and $n \in \{5000, 10000, 20000, 50000\}$. We scaled the ground truth matrix such that the standard deviation of the observations is $1$, in other words, a trivial model which always predicts the global average has a RMSE of 1. The observed entries $\Omega$ are obtained by sampling each entry independently with probability $p = 20 \log(n)/m$.

\begin{figure}[t]
\centering
\includegraphics[width=0.42\textwidth]{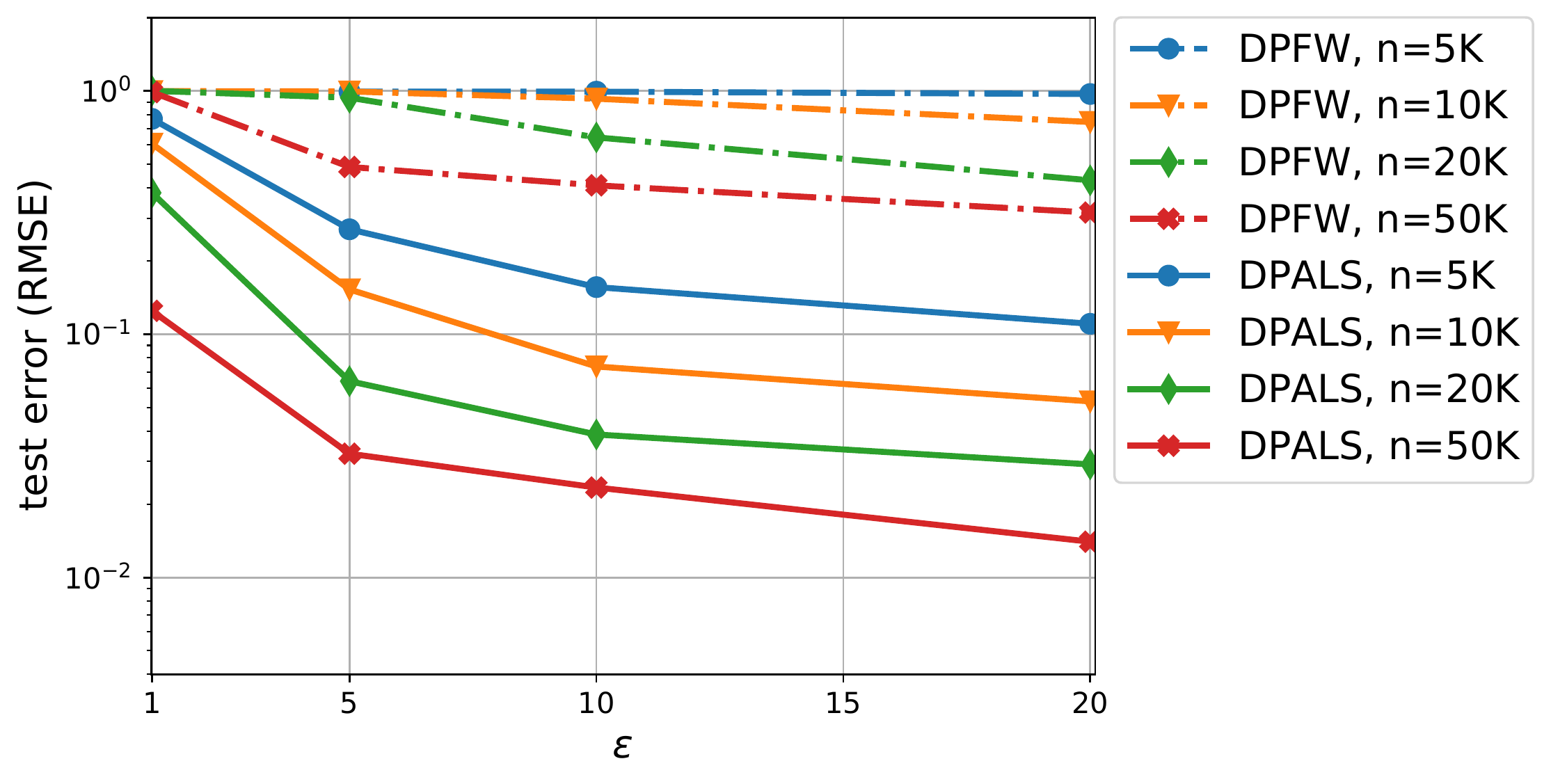}\vspace*{-10pt}
\caption{Comparison of \dpfw and \dpals on synthetic data with different number of rows/users $n$.}
\label{fig:tradeoff_synth}
\end{figure}

\mypar{MovieLens data sets} 
We apply our method to two common recommender benchmarks: 
(i) rating prediction on MovieLens 10M (ML-10M) following~\citet{lee13llorma}, where the task is to predict the value of a user's rating, and performance is measured using the RMSE, (ii) item recommendation on MovieLens 20M (ML-20M) following~\citet{liang18vae}, where the task is to select k movies for each user and performance is measured using Recall@k.
For comparison to \dpfw, we use a variant of the ML-10M task following~\citet{jain2018differentially}, where the movies are restricted to the 400 most popular movies (\dpfw did not scale to the full data set with all movies, unlike \dpals).

\mypar{Experimental protocol}
Each data set is partitioned into training, validation and test sets. Hyper-parameters are chosen on the validation set, and the final performance is measured on the test set. The privacy loss accounting is done using RDP, then translated to $(\epsilon, \delta)$-DP with $\delta = 10^{-5}$ for the synthetic data and ML-10M and $\delta=1/n$ for ML-20M. When training \dpals models on synthetic data, we use the basic Algorithm~\ref{Alg:alt}, without heuristics. When training on MovieLens, we use the heuristics described in Section~\ref{sec:practCons}. Note that even when training on $\head$ items (Section~\ref{sec:practCons-1}), evaluation is always done on the \emph{full set} of items, so that the reported metrics are comparable to previously published numbers. Additional details on the experimental setup are in Appendix~\ref{app:emp}, including a list of hyper-parameters and the ranges we used for each.

\subsection{Privacy-Utility Trade-Off}

\mypar{\dpals vs. \dpfw on synthetic data} 
On synthetic data (Figure~\ref{fig:tradeoff_synth}) we observe:
First, as expected, the trade-off of both algorithms improves as the number of users increases.
Second, for $\epsilon = 1$, the quality of the \dpfw models is no better than the trivial model (RMSE equal to 1), while \dpals has a lower RMSE, which significantly improves with larger~$n$. 
Third, for the largest data set ($n = 50K$), the relative improvement in RMSE between \dpals and \dpfw is at least 7-fold across all values of $\epsilon$. To further illustrate the difference between \dpals and \dpfw, we show in Appendix~\ref{app:emp-experiments} the RMSE against number of iterations, both for the private and non-private variants (Figure~\ref{fig:app:convergence}).

\mypar{\dpals vs. \dpfw on ML10M} 
Next, we compare the two methods on ML-10M-top400 (Figure~\ref{fig:tradeoff_ml10mtop}). For \dpfw and DPSVD, the numbers are taken directly from~\cite{jain2018differentially}.
For reference, we include the test RMSE of non-private ALS, and a simple baseline model that always predicts the global average rating. The performance of DPSVD is worse than that of the simple baseline. \dpals performs best, with a relative improvement in RMSE (compared to \dpfw) that ranges from 7\% to 11.6\%, and that increases with $\epsilon$. In Appendix~\ref{app:emp-experiments}, we show that \dpals achieves performance better than \dpfw even when trained on a small fraction of the users (23\%).

Finally, Figure~\ref{fig:tradeoff_ml10m} shows the privacy/utility trade-off on the full ML-10M data. In order to scale \dpfw to the the full data, we use the same procedure described in Section~\ref{sec:practCons}: \dpfw is trained on the top movies, and for remaining movies the model predicts the user's average rating. Compared to the restricted data set (ML-10M-top400), the privacy-utility trade-off is worse on the full data. This indicates that a smaller ratio between number of users and number of items makes the task harder -- a result that is in line with the theory.

The results on synthetic data and ML10M suggest that \dpals exhibits a much better privacy/utility trade-off than \dpfw, and a better dependence on the number of rows $n$, which is consistent with the theoretical analysis.

\mypar{\dpals on MovieLens 20M} Figure~\ref{fig:tradeoff_ml20m} shows the privacy/utility trade-off of \dpals on the ML-20M data set. We include as a reference the non-private ALS, and a simple baseline model that always returns the k most rated movies.

On this task, the performance of the private model is remarkably good. Indeed, the best previously reported Recall@20 numbers for \emph{non-private} models on this benchmark are 36.0\% for ALS~\cite{liang18vae} and 41.4\% using a sophisticated auto-encoder model~\cite{shenbin20recvae}. Our results show that \dpals can achieve performance comparable to the previously reported state of the art numbers for (non-private) matrix completion, and the utility does not significantly degrade, even at small $\epsilon$.

\subsection{Importance of Adaptive Sampling and Projection}
\label{sec:practCons-3}
In this section, we give additional insights into the effect of the heuristics introduced in Section~\ref{sec:practCons}.
We run a study on ML-10M for $\epsilon = 10$, $r=128$ and a sample size $k = 50$ (both correspond to the best overall model); other hyper-parameters are re-tuned. According to Section~\ref{sec:practCons-1}, we partition the set of movies into $\head$ and $\tail$ and train only on $\head$. The results are reported in Figure~\ref{fig:ml10m_adaptive}, where the movie fraction is simply $|\head|/n$.
\begin{figure}[t]
\centering
\includegraphics[width=0.33\textwidth]{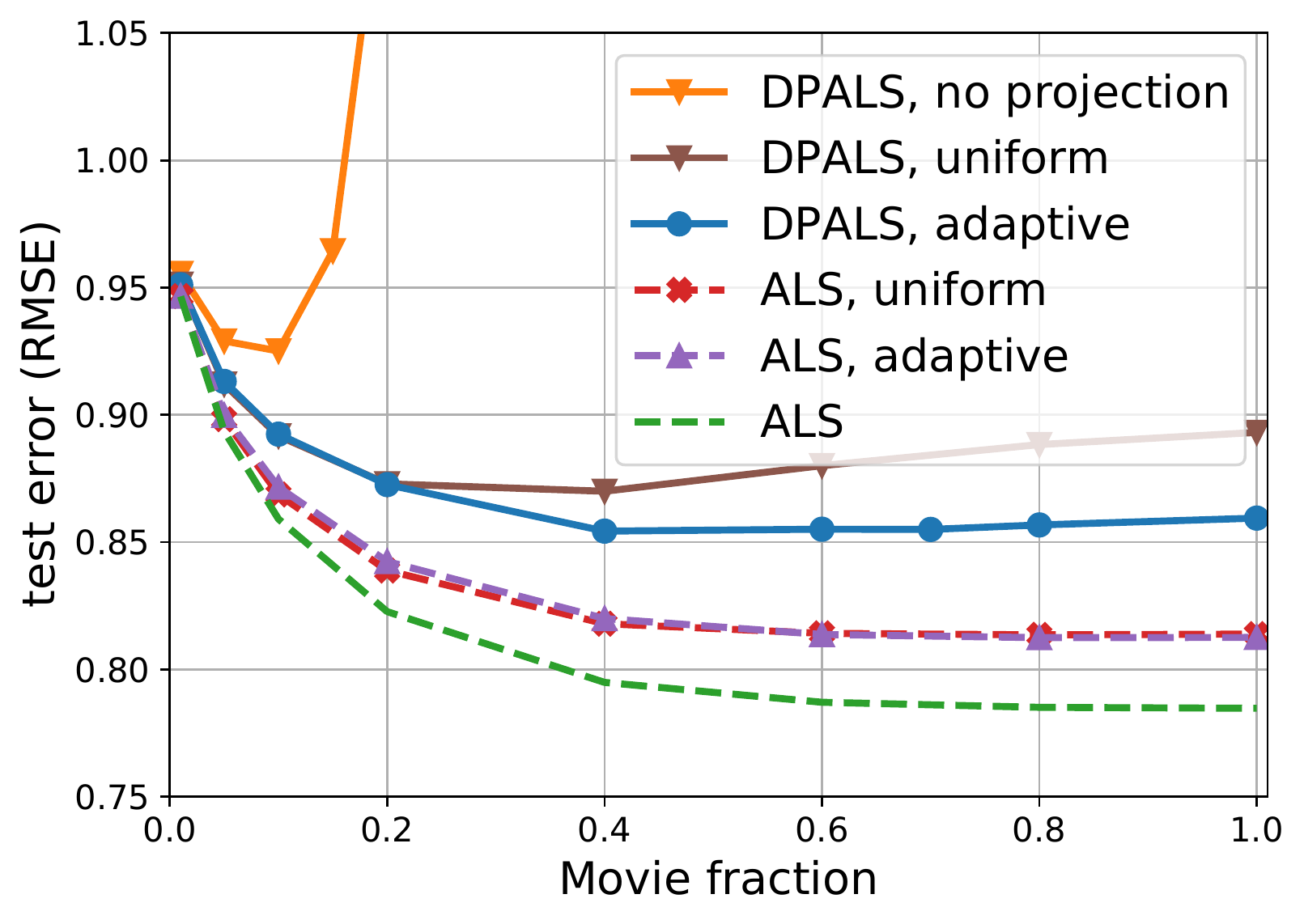}\vspace*{-10pt}
\caption{RMSE vs. movie fraction for $\epsilon = 10$ on ML-10M.}
\label{fig:ml10m_adaptive}
\end{figure}
We make the following observations. First, for non-private ALS, we get the highest RMSE by training on all movies, while there is a benefit for training on a subset of the movies for the private models. Second, when training the non-private model on sub-sampled data (red and purple lines), there is a considerable increase in RMSE, from $0.785$ to $0.812$. This gives an indication that part of the utility loss is due to sub-sampling, and not simply due to the addition of noise. Third, the sampling strategy has a significant impact on the performance of the private \dpals model: adaptive sampling improves the RMSE from 0.870 to 0.854, in contrast, the sampling strategy appears to have little effect on non-private models (i.e. models trained without noise). Finally, training the private model without PSD projection ($\Pi_{\text{PSD}}$ in Line~\ref{alg:step psd} of Algorithm~\ref{Alg:alt}) results in a terrible performance. We find that while the projection is not technically necessary for the theoretical analysis, it is essential in practice.

Training on a subset of the movies appears to  have only a marginal effect when combined with adaptive sampling. However, as detailed in the appendix, the effect is much more significant for smaller $\epsilon$, as well as on ML-20M.

Additional experiments are presented in Appendix~\ref{app:emp}, to explore the effect of other hyper-parameters, such as the rank and the regularization of the objective function.

%% file: app_proof.tex
\section{Proof of Theorem~\ref{thm:privG}}
\label{app:priv}

\begin{proof}
\def\L{\ln(1/\delta)}
\def\sqL{\sqrt{\L}}

To prove the guarantee for Algorithm~\ref{Alg:alt}, it suffices to show the following claim: that at each time step $t\in[T]$, the computations of $\boldX$ and $\sum_{i\in\Omega'_{j}}\boldM_{ij}\cdot\boldU_i$, for all $j\in[m]$ satisfy $\left(\alpha,\alpha\frac{ k}{2\sigma^2}\right)$-RDP. One can then compose the privacy losses via simple R\'enyi composition~\cite{mironov2017renyi} to obtain the overall RDP-cost to be $\left(\alpha,\alpha\frac{kT}{2\sigma^2}\right)$.

To prove the claim, notice that at each time step $t\in[T]$, there are $m$ computations of $\boldX$ and $\sum_{i\in\Omega'_{j}}\boldM_{ij}\cdot\boldU_i$. Since each user $i\in[n]$ can affect only $k$ of those computations, by Gaussian mechanism~\cite{ODO,mironov2017renyi} and the generalization of standard composition  property of RDP~\citep[Proposition 1]{mironov2017renyi} to the joint RDP, we have the required guarantee. 

We now translate joint-RDP to join-DP. By the first part of the theorem, Algorithm~\ref{Alg:alt} is $(\alpha, \alpha \rho^2)$-joint RDP with $\rho^2 = \frac{kT}{2\sigma^2}$. Thus by \citep[Proposition 3]{mironov2017renyi} it is $(\epsilon, \delta)$-joint DP with $\epsilon = \alpha \rho^2 + \frac{\L}{\alpha-1}$ for any $\alpha > 1$. The latter expression is minimized when $\alpha = 1 + \frac{\sqL}{\rho}$, which yields $\epsilon^{\min}(\rho) = \alpha \rho^2 + \frac{\L}{\alpha-1} = 2\sqL \rho + \rho^2$. Now fix $\epsilon>0, \delta \in (0, 1)$. To guarantee $(\epsilon, \delta)$-joint DP while minimizing the noise (which scales as $1/\rho$ by definition of $\rho$), it suffices to maximize $\rho$ subject to $\epsilon^{\min}(\rho) \leq \epsilon$, but since $\epsilon^{\min}$ is increasing in $\rho$, $\rho$ is maximized when $\epsilon^{\min}(\rho) = \epsilon$. This is a second-order polynomial in $\rho$, and it has a positive root at $\rho^+ = \sqrt{\L + \epsilon} - \sqL$. Therefore, setting
\[
\sigma = \frac{\sqrt{kT/2}}{\rho^+} = \frac{\sqrt{kT/2}}{\sqrt{\L + \epsilon} - \sqrt{\L}} = \frac{\sqrt{kT/2}(\sqrt{\L + \epsilon} + \sqrt{\L})}{\epsilon} \leq \frac{\sqrt{2kT(\L + \epsilon)}}{\epsilon}
\]
suffices to guarantee $(\epsilon, \delta)$-joint DP. This completes the proof.
\end{proof}

\section{Proofs from Section~\ref{sec:util}}\label{app:proof}

Recall the problem setting. $\boldM=\Uo \So \Vot$ where $(\Uo)^\top \Uo=\boldI$ and $\Vot \Vo=\boldI$. Also, $\Uo$ and $\Vo$ are  $\mu$-incoherent by assumption. That is, $\|\Uo_i\|_2\leq \mu\sqrt{r}/\sqrt{n}$ and $\|\Vo_j\|_2\leq \mu\sqrt{r}/\sqrt{m}$. The set of observations is $\Omega=\{(i,j)\ s.t.\ \delta_{ij}=1\}$, where $\delta_{ij}$ are i.i.d.  Bernoulli random variables with $\Pr[\delta_{ij}=1]=p$. 
We sample a new set of observations before every update.

We now present a basic lemma. 
\begin{lem}
    \label{lem:rsc}
    Let $\Uo$ and $\Ut$ be $\mu$ and $\mu_1$-incoherent, orthonormal matrices where $\mu_1\geq \mu$ and $n\cdot p\geq \mu\mu_1 r^2$. Then, the following holds for all $j\in[m]$ (w.p. $\geq 1-m\beta$): 
    $$\lfrob{\frac{1}{p}\sum_{i=1}^n \delta_{ij}\Uo_i (\Ut_i)^\top-\Uot\Ut}\leq C\sqrt{\frac{\mu_1^2 r}{n\cdot p}}\cdot \ln \frac{r}{\beta}.$$
\end{lem}
\begin{proof}
The proof follows from the matrix Bernstein inequality \citep[Theorem 6.1.1]{troppmatrix} and incoherence of $\Uo$, $\Ut$.
\end{proof}

\begin{lem}
    \label{lem:mat_bern}
    Let $\delta_{ij}$ be i.i.d. Bernoulli random variables with $\Pr[\delta_{ij}=1]=p$. Then, the following holds (w.p. $\geq 1-\delta$):
    $$\lfrob{\frac{1}{p}\prA(\boldM)-M}\leq C\left(\sqrt{\frac{n}{p}\ln \frac{1}{\beta}}+\frac{1}{p}\ln \frac{1}{\beta}\right)\cdot \|M\|_\infty.$$
\end{lem}
\begin{proof}
The lemma is similar to Theorem 7 of (Recht, 2011) and follows by the matrix Bernstein inequality \citep[Theorem 6.1.1]{troppmatrix}. 
\end{proof}

\subsection{Rank-$1$ Case}
Simplifying the notation, denote $\boldM=\so \uo \vot$ where $(\uo)^\top \uo=1$ and $\vot \vo=1$. 

Note that as $k=C \cdot p \cdot m \ln n$ w.p. $\geq 1-T/n^{100}$, we do not throw any tuples in Line~\ref{line:sample} of Algorithm~\ref{Alg:alt}. Similarly, using incoherence we have:   $\|\boldM\|_\infty\leq \vclip$. So, we do not clip any sample in Line~\ref{line:clipentry} of Algorithm~\ref{Alg:alt}. 

Now, we use mathematical induction to show the incoherence of resulting $\vt$ and $\uht$, and to show that the clipping operations do not really apply in our setting with the selected hyper-parameters. 

For the base case ($t=0$), initialization of $\v^0$ ensures that $\err(\vo, \v^0)\leq \frac{C}{\ln n}$. Now, using \citep[Lemma C.2]{jain2013low} that uses clipping only in the first step to ensure incoherence, we get that $\v^0$ is $16\mu$-incoherent. 

In the induction step, assuming the Lemma holds for $\vt$, we prove the claim for $\uht$ and $\vto$. 
Dropping superscripts of $\Omega'_i$ for notation simplicity and using $\lambda=0$, we have: $\uht=\arg\min\limits_{\u} \lfrob{\prA\left(\boldM-\u(\vt)^\top\right)}^2$. 
The update of $\ut=\uht/\|\uht\|_2$. So using \citep[Lemma 5.5, Lemma 5.7, Theorem 5.1]{jain2013low}, we get w.p. $\geq 1-1/n^{100}$: 
\begin{align}
&\ut\ \text{ is }\ 16\mu\text{-incoherent},\nonumber\\
&\|\uht\|_2\geq \so /16,\nonumber\\
&\err(\ut, \uo)\leq \frac{1}{4} \err(\vt, \vo).\label{eq:ruminone}
\end{align}

To complete the claim, we only need to study the update for $\vto$, which is a noisy version of the ALS update: 
$$\vhto=(\boldD+\boldG^t)^{-1}\left(\prA(M)\uht+\boldg\right),$$
where $\boldD$ and $\boldG^t$ are diagonal matrices s.t. $\boldD_{jj}=\sum_{(i,j)\in \Omega}(\uht_i)^2$ and $\boldG^t_{jj} \sim \rclip^2\sigma\cdot \mathcal{N}(0,1)$. 

We first prove that $\boldD+\boldG^t$ is indeed invertible, and has lower-bounded smallest eigenvalue. Using Lemma~\ref{lem:rsc}, and $pn\geq \mu^2 \log n \log (1/\delta)$, we have w.p. $\geq 1-m\delta$, $$\frac{1}{p} \boldD_{jj} \geq \|\uht\|_2^2 \left(1-\sqrt{\frac{1}{\log n}}\right).$$
Also, using maximum of Gaussians, we have w.p. $\geq 1-m\beta$, $$\frac{1}{p}\snorm{\boldG^t}{2} \leq\frac{\rclip^2\sigma\sqrt{\log (n/\beta)}}{p}\leq \frac{\mu^2(\so)^2\sigma\sqrt{\log (n/\beta)}}{np}\leq \frac{\|\uht\|^2_2}{16\times 256},$$
where the final inequality follows by the assumption on $p$. 

So, 
\begin{equation}\label{eq:dg_one_op}
\|(\boldD+\boldG^t)^{-1}\|_2\leq \frac{2}{ p\cdot \|\uht\|_2^2}.
\end{equation}

We now conduct error analysis for $\vhto$: 
$$\vhto=\alpha\cdot\vo-\boldE,$$
where $\alpha=\frac{\so\cdot \uot\ut}{\|\uht\|_2}$. 
Furthermore, for a matrix $\boldC$ with $\boldC_{jj}=\sum_{(i,j)\in \Omega}\uht_i\uo_i$, we have $\boldE=\boldE^1+\boldE^2$ with $$\boldE^1_j=(\boldD_{jj}+\boldG^t_{jj})^{-1} (\alpha \boldD_{jj}-\so \boldC_{jj})\vo_j,\ \ \boldE^2_j=(\boldD_{jj}+\boldG^t_{jj})^{-1} (\alpha \boldG^t_{jj}\vo_j-\boldg_j).$$
This step follows from the observation that $(\prA(\boldM)\uht)_j=\so \boldC_{jj} \vo_j$. 
(We note that $\boldE$ is a vector but we use upper case to be consistent with Section~\ref{app:proof rank r}.)

Note that $\E[\alpha \boldD_{jj}-\so \boldC_{jj}]=0$. Furthermore, using incoherence of $\vo$, $\snorm{\uht}{2}\geq \so/16$, and the Bernstein's inequality, we have: 
\begin{equation}
\label{eq:rone_lem_1}
\|\boldE^1\|_2\leq \frac{1}{64} \err(\ut,\uo). 
\end{equation}

Now, 
\begin{equation}
\label{eq:rone_lem_2}
\|\boldE^2\|_2 \leq \frac{2\sqrt{\log n}}{p\cdot \snorm{\uht}{2}^2}\cdot \left(16 \rclip^2 \sigma+\vclip \rclip \sigma \sqrt{m}\right)\leq \frac{C\mu^4\sqrt{\log n}}{np}\cdot \sigma. 
\end{equation}
Using \eqref{eq:rone_lem_1} and \eqref{eq:rone_lem_2}, $\snorm{\vhto}{2}\geq 3/4$. 

Thus, we get: 
$$\err(\vo,\vto)\leq \frac{1}{32}\err(\uo,\uto)+\frac{C\mu^4\sqrt{\log n}}{np}\cdot \sigma.$$

Similarly, by incoherence of $\vo$ and using bound on $\boldE^1_j$ and $\boldE^2_j$, we get: 
$$\|\vhto\|_\infty\leq 3\mu.$$

Therefore
 \begin{align}
&\vto\ \text{ is }\ 16\mu\text{-incoherent},\nonumber\\
&\err(\vto, \vo)\leq \frac{1}{32} \err(\ut, \uo)+\frac{C\mu^4\sqrt{\log n}}{np}\cdot \sigma.\label{eq:rvminone}
\end{align}
So, the inductive hypothesis holds. Furthermore, we get Theorem~\ref{thm:utilGM}, by combining the error terms of $\ut$ and $\vto$. 

\subsection{Rank-$r$ Case}
\label{app:proof rank r}
\subsubsection{Proof of Lemma~\ref{lem:incoh}}
Note that as $k=C \cdot p \cdot m \ln n$, w.p. $\geq 1-T/n^{100}$, we do not throw any tuples in Line~\ref{line:sample} of Algorithm~\ref{Alg:alt}. 
Similarly, using incoherence we have:   $\|\boldM\|_\infty\leq \vclip$. So, we do not clip any samples in Line~\ref{line:clipentry} of Algorithm~\ref{Alg:alt}. 

Now, we use mathematical induction to show the incoherence of resulting $\Vht$ and $\Uht$, and to show that the clipping operations do not really apply in our setting with the selected hyperparamters.

For the base case ($t=0$), initialization of $\Vh^0$ ensures that $(\Vh^0)^\top \Vh^0 = \boldI$ and $\err(\Vo, \Vh^0)\leq \frac{C}{\ko^2 r^2 \ln n}$. Now, using \citep[Lemma C.2]{jain2013low} that uses clipping only in the first step to ensure incoherence, we get that $\Vh^0$ is $16\mu\sqrt{r}$-incoherent. 

For the induction step, assuming the Lemma holds for $\Vht$, we prove the claim for $\Uht$ and $\Vhto$. 
Dropping superscripts of $\Omega'_i$ for notation simplicity and using $\lambda=0$, we have: $\Uht=\arg\min\limits_{\boldU} \lfrob{\prA\left(\boldM-\boldU(\Vt)^\top\right)}^2$. 
That is, the update of $\Uht=\Ut \bR_U$, with $\boldU^t$ being the Q part of QR-decomposition, is identical to the standard non-noisy ALS. So using \citep[Lemma 5.5, Lemma 5.7, Theorem 5.1]{jain2013low}\footnote{Lemma 5.5 of \cite{jain2013low} has a redundant $\sqrt{r}$ term in incoherence claim}, we get w.p. $\geq 1-1/n^{100}$,
\begin{align}
&\Ut\ \text{ is }\ 16\ko \mu\text{-incoherent},\nonumber\\
&\|\So \bR_U^{-1}\|_2\leq 16 \ko,\ \ \text{i.e.},\ \ \|\bR_U^{-1}\|_2\leq \ltwo{(\So)^{-1}}\|\So \bR_U^{-1}\|_2\leq 16\|(\So)^{-1}\|_2\ko,\nonumber\\
&\err(\Ut, \Uo)\leq \frac{1}{4} \err(\Vt, \Vo).\label{eq:rumin}
\end{align}
That is, now to complete the claim we only need to study the update for $\Vhto$, which is a noisy version of the ALS update.  

Now consider, 
\begin{align}
    \boldX^t_j&=\Uhtt \left(\sum_{i:(i,j)\in \Omega} \bolde_i \bolde_i^\top\right) \Uht + \boldG^t_j=p\cdot \bR_U \left(\Utt  \left(\frac{1}{p}\sum_{i:(i,j)\in \Omega} \bolde_i \bolde_i^\top\right) \Ut + \boldN^t_j\right) \bR_U,
\end{align}
where $\boldG^t$ is the noise added in Line~\ref{line:LHS} of Algorithm~\ref{Alg:alt} at time step $t$, $\boldD^t_j=\Utt  \left(\frac{1}{p}\sum_{i:(i,j)\in \Omega} \bolde_i \bolde_i^\top\right) \Ut$ and $\boldN^t_j = \frac{1}{p} \bR_U^{-1} \boldG^t_j \bR_U^{-1}$. Note that using Gaussian eigenvalue bound \cite{vershynin} and Weyl's inequality \cite{bhatia}, we have w.p. $\geq 1-1/n^{100}$,
\begin{align}
\sigma_{\min}(\boldD^t_j+\boldN^t_j)\geq \left(1-C\sqrt{\frac{\mu^2\ko^2 r}{n\cdot p}}\cdot \ln n-\frac{2\rclip^2\sigma \sqrt{r}}{p\cdot \sigma_{\min}(\bR_U)^2} \right)\geq \frac{1}{2}, \label{eq:dnt}
\end{align}
where the last inequality follows from: $np\geq C \mu^2 \ko^2 r \ln^2 n$ and $n\sqrt{p}\geq C \mu^2 \ko^6 r\sqrt{r} \cdot \frac{\sqrt{m\ln n}\cdot( T\ln(1/\delta))}{\epsilon}$. 

Next, we argue that $\boldX^t_j$ is PSD. Observe that
\begin{align}
    \sigma_{\min}(\boldX^t_j)&\geq  \frac{1}{2} p\cdot \sigma_{\min}(\bR_U)^2 \geq C\frac{p\cdot \sigma_{\min}(\So)^2}{\ko^2}>0,
\end{align}
where the last inequality follows from \eqref{eq:rumin}.

This shows that $X$ used in update of $\Vhto$ is {\em PSD}, and hence the update for $\Vhto$ is given by: 
\begin{align}
&\bR_U\Vhtot_j\\
=&(\boldD_j+\boldN^t_j)^{-1} \left(\boldC_j\So\Vot_j+\bar{\boldg}^t_j\right)\\
=&\Utt\Uo\So \Vot_j\\
&-(\boldD_j+\boldN^t_j)^{-1}(\boldD_j\Utt\Uo - \boldC_j)\So\Vot_j-(\boldD_j+\boldN^t_j)^{-1} (\boldN^t_j\Utt\Uo\So \Vot_j-\bar{\boldg}^t_j),
\end{align}
where $\boldD_j=\Utt\left(\frac{1}{p} \sum_{i:(i,j)\in \Omega} \bolde_i \bolde_i^\top\right) \Ut$, $\boldC_j=\Utt\left(\frac{1}{p} \sum_{i:(i,j)\in \Omega} \bolde_i \bolde_i^\top\right) \Uo$, and $\bar{\boldg}^t_j = \frac{1}{p} \bR_U^{-1} \boldg^t_j$. 

That is, 
\begin{align}
    \label{eq:vup0}
&    \Vhto \bR_U = \Vo \So \Uot \Ut - \boldE^\top,\ \ \boldE_j=\boldE^1_j + \boldE^2_j,\nonumber\\
&\boldE^1_j=(\boldD_j+{\boldN}^t_j)^{-1}(\boldD_j\Utt\Uo - \boldC_j)\So\Vot_j,  \ \ \boldE^2_j=(\boldD_j+{\boldN}^t_j)^{-1} ({\boldN}^t_j\Utt\Uo\So \Vot_j-\bar{\boldg}^t_j).
\end{align}
Let $\Vhto=\Vto \bR_V$. Then,  
\begin{align}
    \label{eq:vup1}
&    \Vto \bR_V \bR_U = \Vo \So \Uot \Ut - \boldE^\top,\ \ \boldE_j=\boldE^1_j + \boldE^2_j,\nonumber\\
&\boldE^1_j=(\boldD_j+{\boldN}^t_j)^{-1}(\boldD_j\Utt\Uo - \boldC_j)\So\Vot_j,  \ \ \boldE^2_j=(\boldD_j+{\boldN}^t_j)^{-1} ({\boldN}^t_j\Utt\Uo\So \Vot_j-\bar{\boldg}^t_j).
\end{align}
Using the technique of \citep[Lemma 5.6]{jain2013low} and the bound on $\sigma_{\min}(\boldD^t_j+\boldN^t_j)$ (see \eqref{eq:dnt}), we get:  
\begin{equation}
    \lfrob{(\So)^{-1} \boldE^1}\leq \frac{C}{\ko} \err(\Ut, \Uo). \label{eq:soe1}
\end{equation}

Similarly, w.p. $\geq 1-1/n^{100}$: 
\begin{align}
    \|(\So)^{-1} \boldE^2\|_F&\leq \frac{2}{\sigma_{\min}(\So)}\cdot \left(\frac{\rclip^2\sigma}{p\sigma_{\min}(\bR_U)^2}\frac{\sigma_{\max}(\So)\mu\sqrt{r^2m\ln n}}{\sqrt{m}}+\frac{\rclip \vclip \sigma}{p\sigma_{\min}(\bR_U)}\cdot \sqrt{mr \ln n} \right),\nonumber\\
    &\leq  \frac{C\sigma\sqrt{\ln n}}{pn} \cdot \left(\ko^5\cdot \mu^3 r^2+ \mu^3 r^2 \ko^3 \right)\leq \frac{C\ko^5\cdot \mu^3 r^2\sqrt{\ln n}}{\sqrt{p}n}\cdot  \frac{\sqrt{m\ln n}\cdot T\sqrt{\ln( 1/\delta)}}{\epsilon}.\label{eq:soe2}
\end{align}
Let $\beta=\frac{C\ko^5\cdot \mu^3 r^2\sqrt{\ln n}}{\sqrt{p}n}  \frac{\sqrt{m\ln n}\cdot T\ln (1/\delta)}{\epsilon}$. Now, 
\begin{align}
    \sigma_{\min}(\bR_V \bR_U)\geq \sigma_{\min}(\So)\left(1- 2 \err(\Ut, \Uo)-\ko \beta\right)\geq \frac{\sigma_{\min}(\So)}{2}, \label{eq:rvru}
\end{align}
where the last inequality holds because: 
$$\sqrt{p}n\geq C \ko^6 \mu^3 r^2 \sqrt{m} \frac{T\ln n\sqrt{\ln( 1/\delta)}}{\epsilon}.$$

Using \eqref{eq:vup1}, we have:
\begin{align}
\max_j \snorm{\Vtot_j}{2}\leq \frac{2\mu\ko\sqrt{r}}{\sqrt{m}}+\frac{4\mu\ko\sqrt{r}}{\sqrt{m}}+\frac{2\mu \ko r\rclip^2\sigma}{\sqrt{m}p\sigma_{\min}(\bR_U)^2}+\frac{2\rclip \vclip\sigma \sqrt{r}}{p\sigma_{\min}(\So)\sigma_{\min}(\bR_U)}\leq \frac{16 \mu\ko\sqrt{r}}{\sqrt{m}},
\end{align}
where the last inequality follows from the assumption that $\sqrt{p}n\geq C \ko^6 \mu^3 r^2 \sqrt{m} \frac{T\ln n\sqrt{\ln (1/\delta)}}{\epsilon}$. This concludes the proof.

\subsubsection{Proof of Lemma~\ref{lem:disdec}}
The proof for this key Lemma follows technique similar to the above proof. That is, using previous lemma, the clipping operations do not have any effect, and hence we get noisy ALS updates. Now, using \eqref{eq:soe1}, \eqref{eq:soe2}, \eqref{eq:rvru}, and Lemma~\ref{lem:err1}, we have:
\begin{equation}
    \err(\Vo,\Vto)\leq \frac{1}{4} \err(\Uo, \Ut)+4\ko \beta. 
\end{equation}
This proves the lemma.

\subsection{Proof of Theorem~\ref{thm:utilGM}}
Using Lemma~\ref{lem:err1}, we have: 
$$\err(\Vo, \Vt)\leq \frac{1}{4}+\err(\Vo, \Vto)+\alpha,$$
where $\alpha\leq \frac{C\ko^6\cdot \mu^3 r^2\sqrt{\ln n}}{\sqrt{p}n}  \frac{\sqrt{m\ln n}\cdot T\ln 1/\delta}{\epsilon}$. 
So, after $T=\ln \frac{\err(\Vo, \boldV^0)}{\alpha}$ iterations, $\err(\Vo, \boldV^T)\leq 2\alpha$. 

As $\widehat{\boldU}^T=\arg\min_{\widehat{\boldU}}\|\boldM-\widehat{\boldU}^T (\boldV^T)^\top\|_F$, we have: 
\begin{align}
    \|\boldM-\widehat{\boldU}^T (\boldV^T)^\top\|_F\leq \|\boldM-\Uo \So (\boldV^T)^\top\|_F\leq \|\boldM\|_F \|\Vo-\boldV^T\|_2 \leq 2 \alpha \|\boldM\|_F,
\end{align}
where last inequality follows from the fact that $\|\Vo-\boldV^T\|_2\leq 2 \err(\Vo, \boldV)$. 

This shows the second claim of the theorem. The third claim follows similarly while using incoherence of $\boldV^T$. 

\begin{lem}\label{lem:err1}
	Let $\widehat{\boldU}=\Uo \So \boldW + \boldE$ and $\boldU=\widehat{\boldU} \bR^{-1}$ where $\So$ is a diagonal matrix, $\boldW\in \mathbb{R}^{r\times r}$, and $\bR^2=\widehat{\boldU}^\top \widehat{\boldU}$. Then, assuming $\sigma_{\min}(\So)\sigma_{\min}(\boldW)>\|\So\|_2\|\boldE(\So)^{-1}\|_2$, the following holds: 
	$$\|(\boldI-\Uo(\Uo)^\top)\boldU\|_2\leq \frac{\|\boldE \cdot (\So)^{-1}\|_2}{\frac{\sigma_{\min}(\So)}{\|\So\|_2} \sigma_{\min}(\boldW)-\|\boldE(\So)^{-1}\|_2}.$$
	That is, 
		$$\|(\boldI-\Uo(\Uo)^\top)\boldU\|_2\leq \frac{\kappa \|\boldE\|_2}{\sigma_{\min}(\So)\sigma_{\min}(\boldW)-\kappa \|\boldE\|_2}.$$
\end{lem}
\begin{proof}
	\begin{align*}
	\|(\boldI-\Uo(\Uo)^\top)\boldU\|_2\leq \|\boldE \cdot \bR^{-1}\|_2 \leq \|\boldE(\So)^{-1}\|_2 \|\So \bR^{-1}\|_2. 
	\end{align*}
	Furthermore, $\|\So \bR^{-1}\| \leq \|\So\|_2 \|\bR^{-1}\|_2$. Now, 
	\begin{align*}
		\frac{1}{\|\bR^{-1}\|_2}=\sigma_{\min}(\bR)\geq \sigma_{\min}(\So) \sigma_{\min}(\boldW)-\|\So\|_2\|\boldE(\So)^{-1}\|_2. 
	\end{align*}
	That is, 
		\begin{align*}
		\|(\boldI-\Uo(\Uo)^\top)\boldU\|_2\leq \frac{\|\boldE \cdot (\So)^{-1}\|_2}{\frac{\sigma_{\min}(\So)}{\|\So\|_2} \sigma_{\min}(\boldW)-\|\boldE\cdot (\So)^{-1}\|_2}. 
	\end{align*}
\end{proof}
 
 \commented{
 \subsection{Initialization}
 \label{sec:init}
 In this section, we describe the initialization routine used by our method. At a high level, similar to \cite{jain2013low}, we use the top eigen-vectors of $\boldA=\prA(\boldM)^\top \prA(\boldM)$ to obtain $\hV^0$. However, to ensure privacy, we need to add noise to $\boldA$. That is, we compute top-$r$ eigenvectors of $\boldW=\boldA+\boldG$ where $\boldG$ is a symmetric Gaussian matrix with standard deviation $\sigma \vclip^2$.

 Now using Theorem 2 of \cite{dwork2014analyze} which is similar to applying Davis-Kahan theorem to $\boldW$, we get: 
 $$\err(\bar{\boldV}, \boldV^{(0)})\leq \frac{2\|\boldG\|_2}{\lambda_r(\boldA)-\lambda_{r+1}(\boldA)}\leq \frac{2\sqrt{m}\sigma \vclip^2}{\lambda_r(\boldA)-\lambda_{r+1}(\boldA)}.$$
 
 Furthermore, $\|\frac{1}{p}\prA(\boldM)-\boldM\|_2\leq \frac{\|\boldM\|_2\mu^2 r}{\sqrt{pm}}$. That is, $\prA(\boldM)=p \boldM + p \boldE$ where $\|\boldE\|_2\leq \frac{\|\boldM\|_2\mu^2 r}{\sqrt{pm}}$. This implies, $\boldA=p^2 \boldM^\top \boldM + \bar{\boldE}$ where $\|\bar{\boldE}\|_2\leq p^2 \|\boldM\|_2 \|\boldE\|_2\leq p^2 \frac{\|\boldM\|_2^2\mu^2 r}{\sqrt{pm}}$. 
 
 Using Weyl's inequality: $\lambda_r(\boldA)-\lambda_{r+1}(\boldA)\geq p^2 (\so_r)^2 - 2\| \bar{\boldE}\|_2\geq (1-\frac{1}{\log n})p^2 (\so_r)^2 $ due to assumption on $p$. So, using the bound above, we get:  $$\err(\bar{\boldV}, \boldV^{(0)})\leq \frac{2\sqrt{m}\sigma \vclip^2}{p^2 (\so_r)^2}\leq \frac{2  \sqrt{m} \sigma \|\boldM\|_2^2}{p^2mn}.$$

Similarly using Davis-Kahan on $\boldA=p^2 \boldM^\top \boldM + \bar{\boldE}$, we get: 
$$\err(\Vo, \bar{\boldV})\leq \frac{2\|\bar{\boldE}\|_2}{p^2 (\so_r)^2}\leq 2\ko^2 \frac{\mu^2 r}{\sqrt{pm}}.$$

That is, 
$$\err(\Vo, \boldV^{(0)})\leq \frac{2  \sqrt{m} \sigma \|\boldM\|_2^2}{p^2mn}+2\ko^2 \frac{\mu^2 r}{\sqrt{pm}}.$$
The initialization condition is now matched by the assumption on $p$  (Theorem~\ref{thm:utilGM}) combined with the assumption that $n\geq \widetilde{\Omega}(m\sqrt{m\log 1/\delta}/\epsilon)$.}

\input{init}

%% file: init.tex
\subsection{Noisy Power Iteration Initialization}
 \label{sec:inittighter}
 
In this section, we derive a tighter initialization routine through the noisy power iteration procedure. We will show that it only requires $n=\widetilde{O}(m)$ and can succeed with high probability.

Prior work on noisy power iteration requires $\boldA=\projm^\top\projm$ to be incoherent~\cite{hardt2012beating,hardt2013beyond,hardt2013noisy}. But whether this is true under our sparsity condition is a difficult problem. For example, related bounds, first conjectured in~\cite{dekel2011eigenvectors}, have only been shown to hold for constant $p$~\cite{vu2015random,rudelson2015delocalization} but still open for $p=O(\log^c m/m)$ for constant $c>0$, the range interesting to us. To overcome this difficulty, we show that we actually do not need $\boldA$ to be fully incoherent. Instead, we just need $\boldA$'s top-$r$ eigenspaces to be incoherent, and the existence of a (moderate) gap between the top eigenspaces and the rest, both of which we are able to establish. Given these two conditions, we then add the proper amount of noise, with a magnitude in between the top-$r$ eigenvalues and the rest, such that i) it does not interfere with the ``boosting'' of the top-$r$ eigenspace; and ii) it ``randomizes'' the remaining eigenvectors such that their incoherence is preserved through the power iteration.

For simplicity, we will present a detailed proof in the rank-$1$ case. Here, we say a vector $\boldw\in\mathbb{R}^m$ is $\mu$-incoherent if $\left\|\boldw\right\|_\infty\leq\frac{\mu}{\sqrt{m}}$.
 
 \hspace{.1\linewidth}\begin{minipage}{.8\linewidth}
\centering
\begin{algorithm}[H]
\SetAlgoLined
\setcounter{AlgoLine}{0}
\DontPrintSemicolon
\SetKwProg{myproc}{Procedure}{}{}
{\bf Required}: $\prA(\boldM)\in\mathbb{R}^{n\times m}$, number of iterations: $T$, incoherence parameter: $\nu$, $s$: threshold for maximum number of ratings per user,  entry clipping parameter: $\vclip$.\\
\nl $\boldw_1\leftarrow$ Random unit vector in $m$-dimensions.\\
\For{$1\leq t\leq T$}{
\nl {\bf If} $\boldw_t$ is not $\nu$-incoherent, {\bf then} report {failure} and {stop}.\\
\nl Compute $\boldz_t=\left(\prA(\boldM)^\top\prA(\boldM)\right)\cdot\boldw_t$, and $\wz\leftarrow\boldz_{t}+\boldg_t$, where $\boldg_t\sim\calN\left(0,\sigma^2\cdot\mathbb{I}\right)$.\\
\nl Normalize $\wz_{t}$ to obtain $\boldw_{t+1}$. 
}
\KwRet $\boldw_{T+1}$.
\medskip

\caption{Noisy power iteration.}
\label{Alg:init_tighter}
\end{algorithm}
\end{minipage}

\begin{thm}[Privacy guarantee]
Algorithm~\ref{Alg:init_tighter} satisfies $(\alpha,\alpha\rho^2)$-RDP, where $\rho^2=\frac{Ts^3\vclip^4\nu^2}{2m\sigma^2}$.
\label{thm:priv_inittighter}
\end{thm}

The proof follows immediately from $\ell_2$-sensitivity analysis and the RDP guarantee for Gaussian mechanism~\cite{mironov2017renyi}.

For the utility guarantee, consider the case where $\Omega$ is randomly sampled with probability $p$ (so setting $s\approx mp$ is sufficient). Furthermore, since we are in the rank-$1$ case, we will assume $\boldM=\sqrt{mn}\cdot \boldu\otimes \boldv$, where both $\boldu$ and $\boldv$ are $\mu$-incoherent so $\vclip$ can be set as $\mu^2$. Below we assume $\mu,\vclip=O(1)$ for simplicity. Now we will show that we can set $p=O(\log^3 m/m)$ and $n=O(m\log m\sqrt{\log(1/\delta)}/\epsilon)$, with proper choices of $\nu,\sigma$, such that $\rho^2 \leq (\epsilon + \log(1/\delta))/\epsilon^2$, and the above procedure returns a vector $w$ such that $|\boldw\cdot \boldv| > 0.6$ with probability $1-o(1)$, which can be boosted to high probability by the standard method. 

\begin{thm}[Utility guarantee]
There exists constant $C_1, C_2>0$, such that for any $\delta\in(0,1), \epsilon\in (0, \log(1/\delta))$, if $p\geq \frac{ C_1\log^3 m}{m}$ and 
$\sqrt{p}n\geq C_2\frac{\sqrt{\log(1/\delta)}}{\epsilon}\cdot\sqrt{m}\log^{5/2} m$, then we can choose settings of incoherence parameter $\nu$, noise standard deviation $\sigma$, and number of time steps $T$ in Algorithm~\ref{Alg:init_tighter} s.t., w.p. $1-o(1)$, we have $|\boldw_{T+1}\cdot\boldv|>0.6$, where $\boldv$ is the right singular vector of $\boldM$, and $\rho=\frac{\epsilon}{2\sqrt{\log(1/\delta)}}$, i.e. Algorithm~\ref{Alg:init_tighter} satisfies $(\epsilon,\delta)$-differential privacy. The probability guarantee can be boosted to high probability $1-m^{-c}$ for any $c>0$ with the private selection algorithm~\cite{liu2019private,zhu2020improving}.
\label{thm:util_inittighter}
\end{thm}

\begin{proof}
We will prove this theorem through a sequence of claims.  Write $\lambda = p^2 mn$. The following claim is from previous work, e.g.~\cite{recht2011simpler}.

\begin{claim}\label{claim:approx}
For $n=\Omega(m), p=\Omega(\frac{\log{m}}{m})$, with high probability, $\boldA=\lambda (\boldv\otimes \boldv) + \boldB$, where $\ltwo{\boldB}=O(np\sqrt{\log{m}})$. Hence, if $\lambda_1, \boldh_1$ are the principal eigenvalue and eigenvector, respectively, of $A$, then $\lambda_1 = (1\pm o(1)) \lambda$ and $|\boldh_1\cdot \boldv|= 1- o(1)$.
\end{claim} 

One key fact we need is that $\boldh_1$ is not only close to $\boldv$, but also incoherent. The proof essentially follows the arguments in the proof of Theorem~2.16 in~\cite{erdHos2013spectral}. One difference in our case is that the entries in $A$ are not independent because $A=\projm^\top \projm$. This difficulty was overcome in~\cite{jain2015fast}(Lemma 6) by using the resampling technique. Here we present a direct argument, which might be of independent interest.
\begin{claim}\label{claim:incoherence}
For $n=\Omega(m), p=\Omega(\frac{\log^2{m}}{m})$, with high probability, the principal eigenvector of $\boldA$ is $C$-incoherent for some absolute constant $C>0$.
\end{claim}
\begin{proof}
We treat $\projm$ as the adjacency matrix of a random bipartite graph and apply the techniques similar to~\cite{erdHos2013spectral,jain2015fast}. Recall $\boldh_1$ is the principal eigenvector of $\boldA$ with the eigenvalue $\lambda_1$. Denote by $\cenm = \projm - p\sqrt{mn} (\boldu\otimes \boldv)$. We will first show that when $np \gg \log m$, with high probability, there exists $\boldw$, where $\|\boldw\|_{\infty} = O(1/\sqrt{m})$, such that
\begin{equation}\label{eq:h1}
\boldh_1 = (1\pm o(1))\left(\mathbb{I} - \frac{1}{\lambda_1} \cenm^\top \cenm\right)^{-1} \boldw\,.
\end{equation}

Since $\lambda_1 = (1\pm o(1))\lambda$, $\ltwo{\cenm^\top \cenm} = o(\lambda_1)$, we can expand the above equation to:
\[\boldh_1 = (1\pm o(1)) \sum_{k\geq 0} \left(\frac{1}{\lambda_1} \cenm^\top  \cenm\right)^k \boldw\,.\]

We then apply the method in~\cite{erdHos2013spectral,jain2015fast} to show that there exists constant $c_0<1$ such that with high probability:
\begin{equation}\label{eq:x}
\left\|\left(\frac{1}{\lambda_1} \cenm^\top  \cenm\right)^k w \right\|_\infty \leq c_0^k \|\boldw\|_\infty\,.
\end{equation}

These would imply that $\boldh_1$ is $O(1)$-incoherent. We first prove~(\ref{eq:h1}). Since $\cenm=\projm - p\sqrt{mn} (\boldu\otimes \boldv)$, we can write
\[\boldA=\projm^\top  \projm = \cenm^\top  \cenm + p\sqrt{mn} (\boldv\otimes \boldu) \projm + p\sqrt{mn} \projm^\top (\boldu\otimes \boldv) - p^2 mn (\boldv\otimes \boldv)\,.\]

We first observe that with high probability $\wv = \frac{1}{p\sqrt{mn}}\projm^\top \boldu$ satisfies that
\begin{equation}\label{eq:v}
    |\wv_j - \boldv_j| = O\left(\sqrt{\frac{\log m}{pn}}|\boldv_j|\right)\quad\mbox{for all $j\in[m]$.}
\end{equation}

Since $|\boldh_1\cdot \boldv| = 1-o(1)$, we have that when $pn\gg \log m$, $|h_1 \cdot \wv_j - \boldh_1 \cdot \boldv|=o(1)$. Hence
\begin{equation}\label{eq:t1}
p\sqrt{mn}(\boldv\otimes \boldu)\projm \boldh_1 = p\sqrt{mn}(\projm^\top  \boldu \cdot \boldh_1) \boldv = p^2 mn ((\boldh_1 \cdot v) \pm o(1)) \boldv\,.
\end{equation}

In addition,
\begin{equation}\label{eq:t2}
p\sqrt{mn}\projm^\top(\boldu\otimes \boldv) \boldh_1 = p^2 mn (\boldh_1 \cdot v) \wv\,.
\end{equation}

Applying~(\ref{eq:t1}) and~(\ref{eq:t2}), we have that
\begin{align*}
&{}\boldA\boldh_1\\
=&{} (\cenm^\top \cenm + p\sqrt{mn} (\boldv\otimes \boldu) \projm + p\sqrt{mn} \projm^\top (\boldu\otimes \boldv) - p^2 mn (\boldv\otimes \boldv))h_1\\
=&{}\cenm^\top \cenm \boldh_1 + p^2 mn ((\boldh_1 \cdot \boldv) \pm o(1)) \boldv + p^2 mn (\boldh_1 \cdot \boldv) \wv - p^2 mn (\boldh_1 \cdot \boldv) \boldv\\
=&{}\cenm^\top \cenm \boldh_1 + \lambda (\boldh_1\cdot \boldv)(\wv \pm o(1) \boldv)\,.
\end{align*}

Let $\boldw = (\boldh_1\cdot v)(\wv \pm o(1) \boldv)$. Recall $\|\boldv\|_\infty = O(1/\sqrt{m})$. When $pn\gg \log m$, using~(\ref{eq:v}), we have $\|\wv\|_\infty = O(1/\sqrt{m})$ too. Hence $\|\boldw\|_\infty = O(1/\sqrt{m})$. Since $Ah_1 = \lambda_1 \boldh_1$, we have that $\cenm^\top \cenm \boldh_1 +\lambda \boldw = \lambda_1 \boldh_1$, hence $(\lambda_1\mathbb{I} - \cenm^\top\cenm)h_1 = \lambda \boldw$, which implies (\ref{eq:h1}).

Now we prove~(\ref{eq:x}). Since we $\lambda_1 = (1\pm o(1))\lambda$, it suffices to prove~(\ref{eq:x}) by replacing $\lambda_1$ with $\lambda$ instead. Furthermore, since $\|\cenm^\top \cenm\|_2=o(\lambda)$, it suffices to consider $k=O(\log m)$. Let $\cenm' = \frac{1}{p\sqrt{mn}}\cenm$. Write $w'=\left(\frac{1}{\lambda} \cenm^\top  \cenm\right)^k \boldw=(\cenm'^\top \cenm')^k \boldw$. 
Following the proof of Lemma~7.10 in \cite{erdHos2013spectral}, we will bound the $q$-th moment of $|w'_j|$ and apply the Markov inequality. We treat $\cenm'$ as the adjacency matrix of a bipartite graph $[n]\times [m]$ where each edge is labeled with a random variable $\xi_{ij} = (\frac{1}{p}\chi_{ij} - 1)u_i v_j$ where $\chi_{ij}$'s are independent Bernoulli random variables taking value $1$ with probability $p$. By its definition $\E[\xi_{ij}]=0$, and for $r\geq 2$,
\begin{equation}\label{eq:moments}
\E[|\xi_{ij}|^r] = p((\frac{1}{p} - 1)|u_i v_j|)^r + (1-p)(|u_i v_j|)^r = O\left(p\frac{1}{(p\sqrt{mn})^r}\right)\,.
\end{equation}

Let $G=[n]\times[m]$ denote the complete bipartite graph with each edge $ij$ labeled with $\xi_{ij}$. For $j_1, j_2\in [m]$, let $\cP_k(j_1, j_2)$ denote all the length $2k$ paths in $G$ starting from the node $j_1$ and ending at node $j_2$. Then,
\[w'_j = \sum_{j_1} w_{j_1} \sum_{P\in\cP_k(j_1, j)} \prod_{e\in P}\xi_e\,.\]

And
\[{w'_j}^q = \sum_{j_1, j_2, \cdots, j_q} w_{j_1} w_{j_2}\cdots w_{j_q} \sum_{\forall \ell\, P_\ell\in\cP_k(j_\ell, j)} \prod_{e\in \cup P_\ell} \xi_e\,.\]

Here $\cup_\ell P_\ell$ is understood as a multiple set. Hence
\begin{align*}
\E[{w'_j}^q] &= \sum_{j_1, j_2, \cdots, j_q} w_{j_1} w_{j_2}\cdots w_{j_q} \sum_{\forall \ell\, P_\ell\in\cP_k(j_\ell, j)} \E[\prod_{e\in \cup_\ell P_\ell} \xi_e] \\
&\leq \|w\|_\infty^q \sum_{j_1, j_2, \cdots, j_q} \sum_{\forall \ell\, P_\ell\in\cP_k(j_\ell, j)} |\E[\prod_{e\in \cup_\ell P_\ell} \xi_e]|\,.
\end{align*}

We will now bound 
\begin{equation}\label{eq:s}
\sum_{j_1, j_2, \cdots, j_q} \sum_{\forall \ell\, P_\ell\in\cP_k(j_\ell, j)} |\E[\prod_{e\in \cup_\ell P_\ell} \xi_e]|\,.
\end{equation}

Since $\E[\xi_{ij}] = 0$ and $\xi_{ij}$'s are independent random variables, for $\E[\prod_{e\in \cup_\ell P_\ell} \xi_e]\neq 0$, it must be that every edge in $\cup_\ell P_\ell$ appears at least twice. The bound on $\E[|w'_j|^q]$ is by counting the number of such paths and applying the moments bound (\ref{eq:moments}).

The argument follows~\cite{erdHos2013spectral}. We give it here for completeness. Let $P$ denote the set of edges, without multiplicity, in $\cup_\ell P_\ell$. Write $t=|P|$. Then $t\leq kq$ since every edge in $\cup_\ell P_\ell$ has to appear at least twice. In addition, the edges in $P$ form a connected component because all the paths are connected to node $j$. Hence there are at most $t+1$ vertices. Since the set must include $j$, there are at most ${m+n\choose t}$ choices of the set of vertices. Among $t+1$ vertices, we can make at most ${t+2k\choose 2k} (2k)!$ different paths of length $2k$. Hence the total number of paths is bounded by
\[{m+n\choose t}\left({t+2k\choose 2k} (2k)!\right)^q\leq n^t (c_1 kq)^{2kq}\,,\]
for some $c_1>0$. In the last inequality, we used that $n=\Omega(m)$ and $t\leq kq$.

Suppose that $P=\{e_1, e_2, \cdots, e_t\}$, and the multiplicity of these edges in $\cup_\ell P_\ell$ are $s_1, s_2, \cdots, s_t$ respectively. So $\sum_i s_i = 2kq$. Using (\ref{eq:moments}),
\begin{align*}
\E\left[\prod_{e\in \cup_\ell P_\ell} |\xi_e|\right] &= \E[|\xi_{e_1}|^{s_1}]\E[|\xi_{e_2}|^{s_2}]\cdots\E[|\xi_{e_t}|^{s_t}]\\
&= \left(p\frac{1}{(p\sqrt{mn})^{s_1}}\right) \cdots \left(p\frac{1}{(p\sqrt{mn})^{s_t}}\right)\\
&=p^t \frac{1}{(p\sqrt{mn})^{2kq}}\,.
\end{align*}

Hence the contribution to (\ref{eq:s}) by the case of $|P|=t$ is bounded by:
\[n^t (c_1 kq)^{2kq} p^t \frac{1}{(p\sqrt{mn})^{2kq}}=\frac{(c_1 kq)^{2kq}}{(pm)^{kq} (pn)^{kq-t}}\,.\]

Since $t\leq kq$, (\ref{eq:s}) is bounded by
\[O\left(\frac{(c_1kq)^{2kq}}{(pm)^{kq}}\right)\,.\]

Fix $c_0 = 1/2$. For any $c>0$ and $k=O(\log m)$, we can choose $q=\log m/(2k)$ and $p\geq c' \log^2 m/m$ for some $c'>0$ such that (\ref{eq:s}) is bounded by $c_0^{kq} m^{-c}$. Applying Markov inequality, we have that with probability $1-m^{-c}$, $|w'_j|=O(c_0^k \|w\|_\infty)$. Since $c$ can be chosen arbitrarily, we have that with high probability this holds for all $k=O(\log m)$. This completes the proof.

\end{proof}

The following summarizes the above claims and the conditions we need in our proof.
\begin{claim}
Assume $n=\Omega(m), p=\Omega(\frac{\log^3 m}{m})$. Let $\boldA=\sum\limits_i \lambda_i (\boldh_i \otimes \boldh_i)$ be the eigen-decomposition of $A$ where $\lambda_1\geq \lambda_2 \geq \ldots \lambda_m \geq 0$. Then $\lambda_1=(1\pm o(1)) p^2 mn$, and for $i\geq 2$, $\lambda_i=O(np\sqrt{\log{m}}) = o(\lambda_1/\log^2 m)$. In addition $\boldh_1 \cdot \boldv=1-o(1)$, and $\|\boldh_1\|_\infty \leq C /\sqrt{m}$.
\end{claim}

Now we show that
\begin{claim}
There exists $c_2, c_3>0$, such that if $\sigma \geq c_2 \frac{\lambda}{\sqrt{m} \log^{3/2} m}$, then $\forall t\in[T]$, $\boldw_{t}$'s are $c_3\sqrt{\log m}$-incoherent.
\end{claim}

\begin{proof}
For notation convenience, we set $\boldw_0=0$, hence $\boldw_1$ is a random unit vector. Write $\boldw_t= \sum_i \alpha_{ti} \boldh_i$. Then $Aw_t=\sum_i \lambda_i \alpha_{ti} \boldh_i$ and $\wz_t = \boldA\boldw_t + \boldg_t = \sum_i (\lambda_i \alpha_{ti} + g_{ti}) \boldh_i$, where $g_{ti} = g_t \cdot h_i$. Hence $\alpha_{(t+1)i} = (\lambda_i \alpha_{ti} + g_{ti})/ \|\wz_t\|_2$. Since $\boldg_t$ is sampled from $\mathcal{N}(0, \sigma^2\mathbb{I})$, $g_{ti}$'s are i.i.d Gaussian from $\mathcal{N}(0, \sigma^2)$. In addition $\alpha_{0i}=0$. By induction, we have that $\E[\alpha_{ti}]=0$, and the signs $\sign(\alpha_{t1}), \cdots, \sign(\alpha_{tm})$ are uniformly distributed in $\{-1, 1\}^m$, independent of their values, by the same argument in~\cite{hardt2013noisy} Lemma 4.13. We will now first show, by induction, that with high probability,  $\max_{i\geq 2}|\lambda_i\alpha_{ti}| = O(\sigma)$.

When $t=0$, this is clearly true. Now suppose this holds for $w_t$. All the following statements hold with high probability. Note in the following $\alpha_{ti}$'s are  random variables. By $\|\wz_t\|^2 = \sum_i (\lambda_i \alpha_{ti} + g_{ti})^2$ and $\max_{i\geq 2}|\lambda_i\alpha_{ti}| = O(\sigma)$ with high probability, we have that with high probability,
\begin{equation}\label{eq:norm}
    \|\wz_t\|_2^2 = \Omega(\lambda_1^2 \alpha_{t1}^2 + m\sigma^2)\,.
\end{equation}

By induction hypothesis for any $i\geq 2$,  $|\lambda_i\alpha_{ti}| = O(\sigma)$, hence $|\lambda_i \alpha_{ti} + g_{ti}|= O(\sqrt{\log m} \sigma)$. Since $\|\wz_t\|=\Omega(\sqrt{m}\sigma)$, $|\alpha_{(t+1),i}| = |\lambda_i \alpha_{ti} + g_{ti}|/\|\wz_t\|=O(\frac{\sqrt{\log m}}{\sqrt{m}})$, hence  $|\lambda_i\alpha_{(t+1),i}| = O(\frac{\lambda}{\log^2 m}\cdot \frac{\sqrt{\log m}}{\sqrt{m}})=O(\frac{\lambda}{\log^{3/2} m\sqrt{m}})=O(\sigma)$. We can clearly choose $c_3$ large enough to make sure the induction goes through.

The $j$-th coordinate of $\wz_t$ is $\wz_t(j) = \sum_i (\lambda_i \alpha_{ti} + g_{ti}) h_i(j)$. Consider $\wz_t(j)'=\sum_{i\geq 2} (\lambda_i \alpha_{ti} + g_{ti}) h_i(j)$. Clearly $g_{ti}$'s are independent Gaussian variables. In addition, $\sign(\alpha_{t1}),\cdots, \sign(\alpha_{tm})$ are uniformly distributed over $\{-1, 1\}^m$. Since $|\lambda_i\alpha_{ti}|=O(\sigma)$ for $i\geq 2$, and $\sum_i h_i(j)^2 \leq 1$, by applying the concentration bound, we have that with high probability 
\[|\wz_t(j)'| = O\left(\sqrt{\log m \textstyle\sum_{i\geq 2} \sigma^2 h_i(j)^2}\right) = O( \sigma\sqrt{\log m})\,.\]

Hence, with high probability, 
\begin{equation}\label{eq:ub}
\wz_t(j)^2 = O(\lambda_1^2\alpha_{t1}^2h_1(j)^2 + (\log m)\sigma^2) = O(\lambda_1^2\alpha_{t1}^2/m + (\log m)\sigma^2)\,.
\end{equation}

The last inequality is by $\|h_1\|_\infty\leq C /\sqrt{m}$. Combining with Equation~(\ref{eq:norm}), by distinguishing the cases of $|\lambda_1\alpha_{t1}|\leq \sqrt{m}\sigma$ and $|\lambda_1\alpha_{t1}|\geq \sqrt{m}\sigma$, we have that $\frac{|\wz_t(j)|}{\|\wz\|_2}=O(\frac{\sqrt{\log m}}{\sqrt{m}})$, i.e $w_t$ is $O\left(\sqrt{\log m}\right)$-incoherent.
\end{proof}

Now, we show that we need only $O(\log m)$ round to get a constant approximation to $v$. Note that here we cannot get high probability bound because we need the initial $|\alpha_{01}|$ to be $\Omega(1/\sqrt{m\log m})$ to bootstrap the process. But it does happen with probability $1-o(1)$. 
\begin{claim}
There exists $c_4, c_5>0$, such that with $n,p,\sigma,\nu$ as set above, if $T\geq c_4 \log m$, $|w_T\cdot v|\geq 1-c_5/\log m$ w.p. $1-o(1)$.
\end{claim}
\begin{proof} 
It suffices to show that $|\alpha_{T1}|=\Omega(1)$. Because $\lambda_2 = O(\lambda_1/\log^2 m)$ and $\sqrt{m\log m}\sigma=O(\lambda_1/\log m)$, once $|\alpha_{T1}|=\Omega(1)$, with one more round we would have $|\alpha_{T+1,1}| = 1 - O(1/\log m)$. 

Suppose that $|\alpha_{01}|=\Omega(1/\sqrt{m\log m})$, which happens with probability $1-o(1)$. Then $|\lambda_1\alpha_{01}| = \Omega(\log m) \sigma$. We can prove by induction that $|\lambda_1\alpha_{t1}| = \Omega(\log m)\sigma$ with high probability. By $\alpha_{(t+1)1} = (\lambda_1\alpha_{t1} + g_{t1})/\|\wz_t\|_2$, we have $|\alpha_{(t+1)1}|\geq \frac{1}{2} |\lambda_1\alpha_{t1}|/\|\wz_t\|_2$ with high probability. By Equation~(\ref{eq:ub}), we can bound
\[\|\wz_t\|_2^2 = O\left(\lambda_1^2\alpha_{t1}^2 + (m\log m)\sigma^2\right)\,.\]

Hence,
\begin{equation}
|\alpha_{(t+1)1}| \geq \frac{1}{2} |\lambda_1\alpha_{t1}|/\|\wz_t\|_2=\Omega(|\lambda_1\alpha_{t1}|/(|\lambda_1\alpha_{t1}| + \sqrt{m\log m}\sigma))\,.
\end{equation}

Now if $|\lambda_1\alpha_{t1}|\geq \sqrt{m\log m}\cdot{\sigma}$, then $|\alpha_{(t+1)1}| = \Omega(1)$ so we are done. Otherwise, $|\alpha_{(t+1)1}|=\Omega(|\lambda_1\alpha_{t1}|/(\sqrt{m\log m}\cdot\sigma))=\Omega(\log m)|\alpha_{t1}|$, by $\sigma = O\left(\frac{\lambda}{\sqrt{m}\log^{3/2} m}\right)$. Hence within $O(\log m)$ rounds, $\alpha_{t1}\geq c$ for some constant $c>0$. 
\end{proof}

Now, gather the assumption $\vclip=\mu^2=O(1)$ and the conditions $\nu=O(\sqrt{\log m}), \sigma=\Omega(\frac{p^2 mn}{\sqrt{m}\log^{3/2} m}), T=O(\log m)$ and plug them into the formula in Theorem~\ref{thm:priv_inittighter}. We have that
\begin{equation}
\rho^2 = \frac{T (mp)^3\vclip^4 \nu^2}{2m \sigma^2} = O\left(\frac{m\log^5 m}{p n^2}\right)\,.
\end{equation}
Hence we can set $n=O\left(\frac{\sqrt{\log(1/\delta)}}{\epsilon}\sqrt{m/p}\log^{5/2} m\right)$, such that $\rho^2=\frac{\epsilon^2}{4\log(1/\delta)}$. This completes the proof. We note that when $p=O(\log^3 m/m)$, $n=\widetilde{O}(m\log m)$ which is nearly optimal.

Note that the only reason that prevents the high probability guarantee is due to the choice of the initial random vector. It can be boosted to high probability guarantee by running the process $O(\log m)$ times and privately releasing the vector $w$ with $\|\projm w\|^2$ above a threshold. This can be done through the private selection algorithm~\cite{liu2019private} with an extra constant factor in $\epsilon$.
\end{proof}

%% file: app_expt.tex
\newpage
\section{Additional Details on Heuristic Improvements}
\label{app:heuristics}

Algorithm~\ref{Alg:preprocessing} summarizes the data pre-processing and sampling heuristics described in Section~\ref{sec:practCons}.

\hspace{.15\linewidth}\begin{minipage}{.7\linewidth}
\centering
\begin{algorithm}[H]
\SetAlgoLined
\setcounter{AlgoLine}{0}
\DontPrintSemicolon
\SetKwProg{myproc}{Procedure}{}{}
{\bf Required}: $\prA(\boldM)$: Observed ratings, $\vclip$: entry clipping parameter, $k$: maximum number of ratings per user, $\sigma_p$: standard deviation of the pre-processing noise, $\beta$: fraction of movies to train on. \\
\nl Clip entries in $\prA(M)$ so that $\|\prA(M)\|_\infty \leq \vclip$ \\
\nl Uniformly sample $\Omega'$:\\
\For{$1\leq i\leq n$}{
$\Omega_i' \leftarrow $ sample $k$ items from $\Omega_i$ uniformly.
}
\nl Compute movie counts $\countsP \leftarrow \text{Counts}(\Omega')$.\label{line:app:count1}\\
\nl Partition movies:\\
Let $\head$ be the $\lceil\beta m\rceil$ movies with the largest $\countsP$, and let $\tail$ be the rest. \label{line:app:partition}\\
\nl Adaptively sample $\Omega''$:\\
\For{$1\leq i\leq n$}{
$\Omega_i'' \leftarrow $ the $k$ items in $(\Omega_i \cap \head)$ with the lowest count $\countsP$. \label{line:app:adaptive}
}
\nl Recompute movie counts $\countsP \leftarrow \text{Counts}(\Omega'')$\label{line:app:count2}\\
\nl Center the data $\prApp(M)\leftarrow \prApp(M) - \av$, where
$\av=\frac{\sum_{(i,j)\in \Omega''}M_{ij}+\calN(0,k\vclip^2\sigma_p^2)}{|\Omega''|+\calN(0,k\sigma_p^2)}$\\
\KwRet $\prApp(\boldM), \countsP$
\medskip
\medskip
\\
\myproc{$\text{Counts}(\Omega)$}{
\For{$1\leq j\leq m$}{
$\countsP_j \leftarrow |\Omega_j| + \calN(0,\sigma_p^2)$\\
\KwRet $\countsP$
}
}
\caption{Data pre-processing heuristics.}
\label{Alg:preprocessing}
\end{algorithm}
\end{minipage}

First, we compute differentially private movie counts (Line~\ref{line:app:count1}) using a uniform sample $\Omega'$, and use it to partition the movies (Line~\ref{line:app:partition}) and to perform adaptive sampling (Line~\ref{line:app:adaptive}). The final subset used for training is $\Omega''$, which consists only of $\head$ movies. Finally, to have a more accurate estimate of the counts, we recompute $\countsP$ on $\Omega''$ (Line~\ref{line:app:count2}). We redo this computation as the counts are also used during optimization, as described in the next section. Note that in both computations of $\countsP$, we use a subset of $\Omega$ that contains at most $k$ movies per user, in order to guarantee user-level differential privacy.

\mypar{Privacy accounting} As we saw in Theorem~\ref{thm:privG}, Algorithm~\ref{Alg:alt} with random initialization satisfies $\big(\alpha,\frac{\alpha(kT)}{2\sigma^2}\big)$-joint RDP. The data processing heuristics in Algorithm~\ref{Alg:preprocessing} satisfy $\big(\alpha,\frac{\alpha(2k+2)}{2\sigma_p^2}\big)$-RDP. So, by standard composition of RDP, we have the total privacy cost at any order $\alpha>1$ to be: $\big(\alpha,\alpha\cdot\big(\frac{kT}{2\sigma^2}+\frac{k+1}{\sigma_p^2}\big)\big)$. We can obtain the final $(\epsilon,\delta)$-joint differential privacy guarantee by optimizing for $\alpha$, similarly to Appendix~\ref{app:priv}.

\mypar{Loss function}
We minimize the following loss in practice.
\begin{equation}
\label{eq:objective}
f(\Uh, \Vh) = \|\prA\big(\boldM-\hU\hV{}^\top\big)\|_F^2 + \lambda_0\|\hU\hV{}^\top\|_F^2 + \lambda \sum_{i = 1}^n \frac{\counts_i^{\nu}}{Z} \|\hU_i\|^2 + \lambda \sum_{j = 1}^m \frac{\countsP_j^\mu}{Z'} \|\hV_j\|^2,
\end{equation}
where $\lambda_0$, $\lambda$, $\mu$, and $\nu$ are hyper-parameters. The loss function used in the description of Algorithm~\ref{Alg:alt} is a special case of~\eqref{eq:objective} where $\lambda_0 = \mu = \nu = 0$. The additional terms in~\eqref{eq:objective} do not change the essence of the algorithm, but we find that they make a significant difference in practice.

First, the term $\lambda_0 \|\hU\hV{}^\top\|_F^2$ is often used in problems with implicit feedback, as in~\cite{hu2008ials}. In such problems, the observed entries are often binary, and minimizing the objective $\|\prA\big(\boldM-\hU\hV{}^\top\big)\|_F^2$ can yield a trivial solution -- the matrix of all ones. The addition of the second term penalizes non-zero predictions outside of $\Omega$, leading to better generalization. One of the benchmarks we use is an implicit feedback task, in which the use of the second term is necessary. As described in Section~\ref{sec:practCons-2}, this results in an additional term $\boldK$ in Line~\ref{line:LHS} of $\aglobal$, and care is needed when adding privacy protection to this term, since it involves a sum over all user embeddings. 
The key observation is that this term is constant for all items, so we only need to compute a noisy version of $\boldK$ once and use it for all items, thus limiting the privacy loss it incurs.

Second, we use a weighted $\ell_2$ regularization, where the weights are defined as follows. The weight of movie $j$ is $\countsP_j^\mu / Z'$, where $\countsP$ is the vector of approximate counts (computed in Algorithm~\ref{Alg:preprocessing}), $\mu$ is a non-negative hyper-parameter and $Z'$ is the normalizing constant $Z' = \frac{1}{m}\sum_{j=1}^m \countsP_j^\mu$. When $\mu$ is positive, this corresponds to applying heavier regularization to more frequent items, and we found in our experiments that this can significantly help generalization. The weights for the users are defined similarly, with one main difference: instead of using approximate counts $\countsP$, we use the exact counts $\counts$, as this term only affects the solution in $\alocal$, which is a privileged computation as illustrated in Figure~\ref{fig:blockSchematic}.

\begin{figure}[h]
\centering
\begin{subfigure}[b]{0.29\textwidth}
\includegraphics[width=\textwidth]{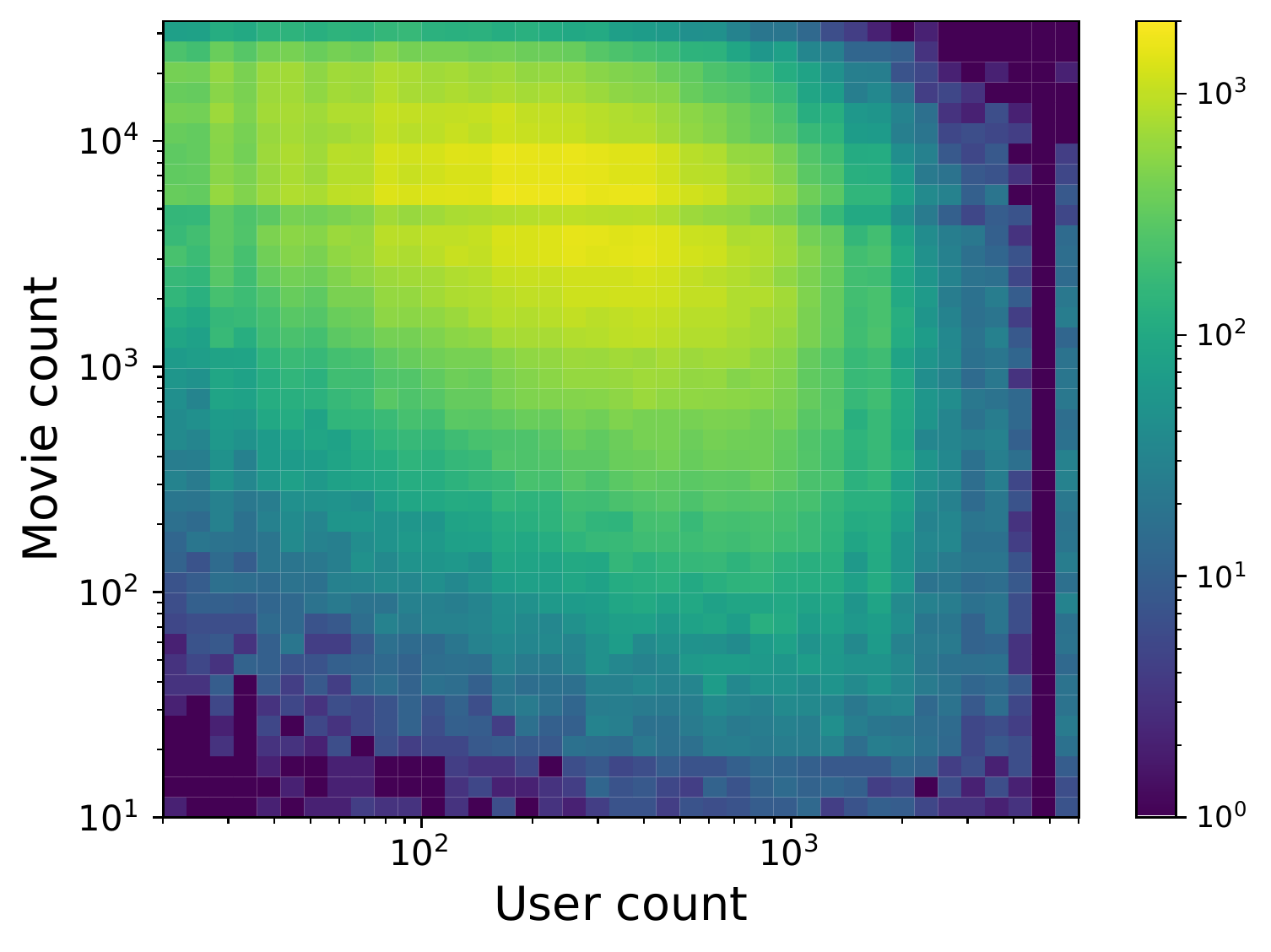}
\vspace{.13in}
\caption{ML-10M (unsampled)}
\label{fig:app:counts_hist_1}
\end{subfigure}\hfill
\begin{subfigure}[b]{0.31\textwidth}
\includegraphics[width=\textwidth]{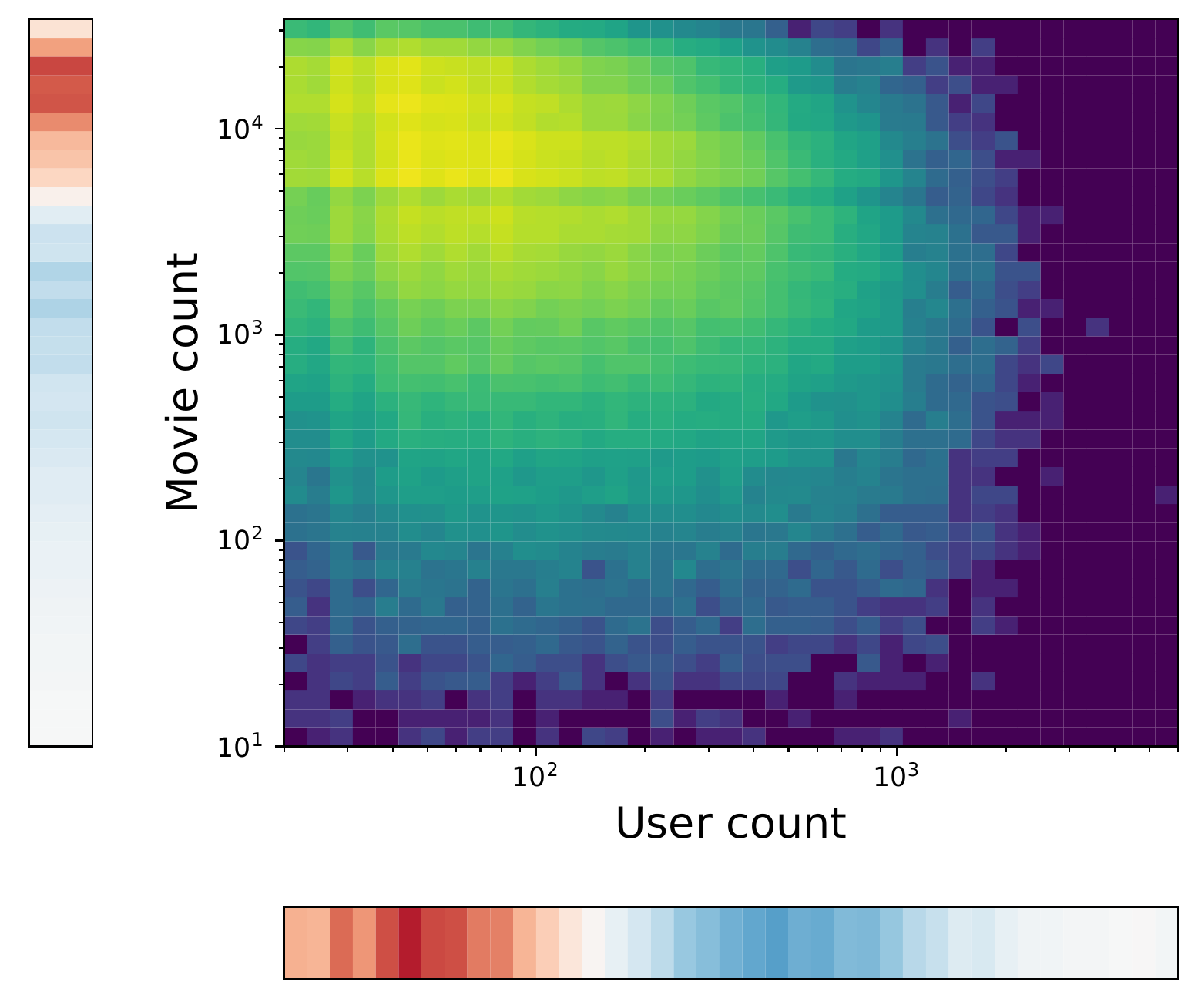}
\caption{Uniform sampling (k = 50)}
\label{fig:app:counts_hist_2}
\end{subfigure}\hfill
\begin{subfigure}[b]{0.31\textwidth}
\includegraphics[width=\textwidth]{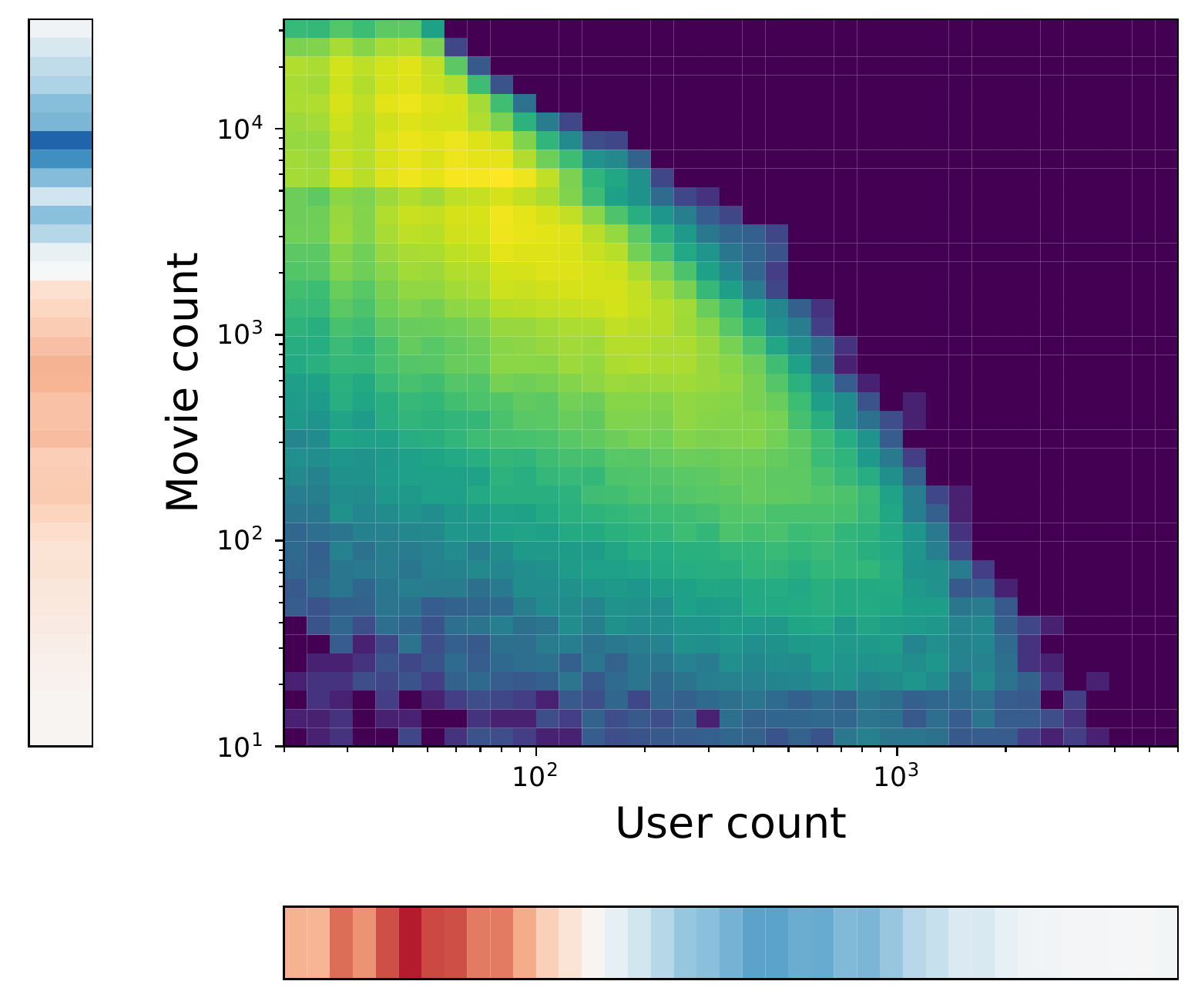}
\caption{Adaptive sampling (k = 50)}
\label{fig:app:counts_hist_3}
\end{subfigure}
\caption{Histogram of user and movie counts in ML-10M, in the original data, and under uniform and adaptive sampling. The color bars in Figures~\ref{fig:app:counts_hist_2} and~\ref{fig:app:counts_hist_3} show the difference in marginal probability compared to the original data in~\ref{fig:app:counts_hist_1}. Red indicates an increase in marginal probability, while blue indicates a decrease. Note that the probability of frequent movies increases under uniform sampling, and decreases under adaptive sampling.}
\label{fig:app:counts_hist}
\end{figure}

\mypar{Effect of uniform and adaptive sampling}
As observed in Figure~\ref{fig:ml10m_count_dist}, the movie count distribution of the MovieLens data set is heavily skewed. We also observed that, perhaps surprisingly, uniformly sampling $k$ items per user tangibly increases the skew. This can be explained by a negative correlation between user counts and movie counts; we computed a correlation coefficient of $-0.243$. This is also visible in Figure~\ref{fig:app:counts_hist_1}, which shows the joint histogram of $\{(c_i, c_j), (i, j) \in \Omega\}$, where $c_i = |\Omega_i|$ is the user count (the number of ratings this user produced) and $c_j = |\Omega_j|$ is the movie count (the number of ratings the movie received). The figure illustrates that infrequent users are more likely to rate frequent movies than the average user. By uniformly sampling a constant number of movies per user (Figure~\ref{fig:app:counts_hist_2}), we are, by definition, increasing the probability of infrequent users, hence increasing the probability of frequent movies (due to the negative correlation). This is made clear by the color bar left of Figure~\ref{fig:app:counts_hist_2}, which shows the change in movie count probability with respect to the original data set. This increase in the probability of frequent movies aggravates the skew of the movie distribution, as seen in Figure~\ref{fig:ml10m_count_dist}.

Adaptive sampling has the opposite effect: Figure~\ref{fig:app:counts_hist_3} shows that the probability of frequent movies \emph{decreases} under adaptive sampling, while that of infrequent movies increases. This leads to a decrease in bias toward frequent movies, as shown in Figure~\ref{fig:ml10m_count_dist}, and results in a significant improvement in the privacy/utility trade-off as discussed in Section~\ref{sec:practCons-3}.

\section{Additional Details on Experiments}\label{app:emp}
\subsection{Details on the Experimental Setup}
Table~\ref{tbl:movielens} shows the statistics of the MovieLens data sets.
\begin{table}[h]
\centering
\footnotesize
\caption{Statistics of the experiment data sets.}
\label{tbl:movielens}
\begin{tabular}{cccccc} 
\toprule
 &  ML-10M-top400 & ML-10M & ML-20M \\ [0.5ex]
\midrule
$n$ (number of users)  & 69,692 &  69,878 & 136,677 \\
$m$ (number of items)  & 400 & 10,677 & 20,108 \\
$|\Omega|$ (number of observations) & 4.49M & 10M & 9.99M \\
\bottomrule
\end{tabular}
\end{table}

For each data set, we partition the set of observations $\Omega$ into $\Omega = \OmegaTrain \sqcup \OmegaValid \sqcup \OmegaTest$. Hyper-parameter tuning is performed on $\OmegaValid$, and the final results are reported on $\OmegaTest$. The pre-processing described in Algorithm~\ref{Alg:preprocessing} is only applied to $\OmegaTrain$.

In the ML-10M benchmark, we follow the setup of~\cite{lee13llorma} and use a 80-10-10 split (random uniform over $\Omega$). In the ML-10M-top400 benchmark, we follow the setup of~\cite{jain2018differentially} and use a 98-1-1 split (random uniform over $\Omega$). In the ML-20M benchmark, we follow the setup of~\cite{liang18vae} and partition the set by \emph{users}, that is, a set of 20K random users are held-out, half of which are used for validation, and the other half for testing. Note that since held-out users are never seen in training, the protocol is to further split each user's observations $\OmegaTest_i$ (uniformly at random) into $\OmegaTestQuery_i \sqcup \OmegaTestTarget_i$. At test time, the model is allowed access to $\OmegaTestQuery_i$ to compute a user embedding and make a prediction for the user, and $\OmegaTestTarget_i$ is used as the ground truth target. The user embedding $\Uh_i$ is computed at test time simply by minimizing the loss in Eq.~\eqref{eq:objective} given the learned movie embeddings $\Vh$, that is,
\[
\Uh_i = \argmin_{u \in \mathbb R^r} \|{\sf P}_{\OmegaTestQuery_i}\big(\boldM_i-u\hV{}^\top\big)\|_F^2 + \lambda_0\|u\hV{}^\top\|_F^2 + \lambda \frac{\counts_i^{\nu}}{Z'} \|u\|^2.
\]
The resulting $\Uh_i$ is used to generate predictions for user $i$. Note that this procedure is consistent with the Joint-DP setting: the computation of $\Uh_i$ corresponds to one step of $\alocal$ in Algorithm~\ref{Alg:alt}, and is considered privileged (see Figure~\ref{fig:blockSchematic}). Besides, since the resulting embedding is not further used for training, it is unnecessary to clip the embedding norm. Avoiding norm clipping at test time could result in better predictions.

Finally the recall for user $i$ is computed as follows. Let $\Omega^{\sf prediction}_i$ be the top $k$ items that are not in $\OmegaTestQuery_i$. Then
$
\text{Recall}@k = \frac{|\Omega^\text{prediction}_i \cap \OmegaTestTarget_i|}{\min(k, |\OmegaTestTarget_i|)}
$.

\subsection{Hyper-Parameter Description and Ranges}
Table~\ref{tab:hyperparams} summarizes the complete list of hyper-parameters used in Algorithm~\ref{Alg:alt}, Algorithm~\ref{Alg:preprocessing}, and in the loss function~\eqref{eq:objective}, and specifies the ranges used in our experiments.
\begin{table}[h]
\caption{Hyper-parameter description and ranges.}
\label{tab:hyperparams}
\centering
\begin{tabular}{l l c}
\toprule
Symbol &Description & Range \\
\midrule
Model and training parameters & & \\
\hline
$r$ & rank & [2, 128] \\
$\reg$ & $\ell_2$ regularization coefficient & [0.1, 100]\\
$\reg_0$ & coefficient of the global penalty term & [0.1, 5]\\
$\mu$ & item regularization exponent & \{0, 0.5, 1\}\\
$\nu$ & user regularization exponent & \{0, 0.5, 1\}\\
$T$ & number of steps & [1, 5] \\
\midrule
Privacy parameters & & \\
\hline
$\rclip$ & row clipping parameter & 1 \\
$\vclip$ & entry clipping parameter & \{1, 5\} \\
$k$ & maximum number of ratings per user & [20, 150]\\
$\sigma$ & noise standard deviation & see remark below \\
\midrule
Pre-processing parameters & & \\
\hline
$\beta$ &fraction of items to train on & [0, 1]\\
$\sigma_p$ & standard deviation of pre-processing noise & [10, 200] \\
\bottomrule
\end{tabular}
\end{table}
We make several remarks about hyper-parameters:
\begin{itemize}
\itemsep0em
\item[--] In the non-private baselines, only the model and training parameters are tuned.
\item[--] Pre-processing (Algorithm~\ref{Alg:preprocessing}) is not used in synthetic experiments. Indeed, these heuristics are designed to deal with the non-uniform distribution of observations in practice. In synthetic experiments, the distribution is uniform by design.
\item[--] In the MovieLens experiments, the maximum value in $\boldM$ is known by definition of the task: In ML-10M, entries represent ratings in the interval $[0.5, 5]$, and in ML-20M the entries are binary. Thus, we simply set $\vclip$ to this value without tuning.
\item[--] We find that carefully tuning the model parameters, including the regularization coefficients $\lambda, \lambda_0$ and the exponents $\mu, \nu$ is important and can have a significant effect.
\item[--] For the rating prediction tasks (ML-10M and ML-10M-top400), we find that setting $\lambda_0$ to a positive number is detrimental, so we always use $0$. For the item recommendation task (ML-20M), using a non-zero $\lambda_0$ is important.
\item[--] The partitioning of the movies into $\head$ and $\tail$ is important for the private models, especially at lower values of $\epsilon$ (see Figure~\ref{fig:app:effect_of_frequent}), but does not help for the non-private baselines.
\item[--] To set the standard deviation $\sigma$, we use the simple observation that when all hyper-parameters except $\sigma$ are fixed, $\epsilon$ is a decreasing function of $\sigma$ that can be computed in closed form. Therefore, in each experiment, we set a target value of $\epsilon$ and do a binary search over $\sigma$ to select the smallest value that achieves the target $\epsilon$.
\item[--] Finally, note that in Algorithm~\ref{Alg:alt}, the parameter $\sigma$ determines the standard deviation of two noise terms: $\boldG$ in Line~\ref{line:noise1} and $\boldg$ in Line~\ref{line:noise2}. While this is sufficient for the analysis, we find in practice that the model is often more sensitive to $\boldg$, thus it can be advantageous to use different scales of noise. We will use the symbols $\sigma_G, \sigma_g$ to specify the scales of each term.
\end{itemize}

The optimal hyper-parameter values for each experiment and each value of $\epsilon$ are given in Table~\ref{tab:optimal-parameters}. These values are obtained through cross-validation. We do not include the privacy loss of hyper-parameter search because our main objective is to give insights into the choice of hyper-parameters at different privacy budgets. In practice, this can be accounted for, for example by the method in \cite{liu2019private}.

\begin{table}[H]
\caption{Optimal hyper-parameter values for the experiments in Figure~\ref{fig:tradeoff}. The clipping parameter $\rclip$ is set to $1$ in all experiments.}
\label{tab:optimal-parameters}
\centering
\begin{tabular}{l | c c c c | c | c c c c | c | c c c c | c }
\toprule
 & \multicolumn{5}{c|}{ML-10M-top400} & \multicolumn{5}{c|}{ML-10M} & \multicolumn{5}{c}{ML-20M} \\
\midrule
& \multicolumn{4}{c|}{\dpals} & ALS & \multicolumn{4}{c|}{\dpals} & ALS & \multicolumn{4}{c|}{\dpals} & ALS \\
\midrule
$\epsilon$ 	& 0.8 	& 4 	& 8 	& 16 	& - 	& 1 	& 5 	& 10 	& 20 	& - 	& 1 	& 5 	& 10 	& 20 	& -\\
\midrule
$r$ 		& 50 	& 50 	& 50 	& 50 	& 50 	& 32 	& 128 	& 128 	& 128 	& 128 	& 32 	& 32 	& 32 	& 128 	& 128\\
$\lambda$ 	& 90 	& 90 	& 80 	& 80 	& 70 	& 120 	& 80 	& 70 	& 60 	& 70 	& 0.5 	& 0.5 	& 0.1 	& 50 	& 30\\
$\lambda_0$ & 0 	& 0 	& 0 	& 0 	& 0 	& 0 	& 0 	& 0 	& 0 	& 0 	& 2 	& 0.6 	& 0.4 	& 0.4 	& 0.1\\
$\mu$ 		& 0.5 	& 0.5 	& 0.5 	& 0.5 	& 1 	& 0.5 	& 0.5 	& 0.5 	& 0.5 	& 1 	& - 	& - 	& - 	& - 	& -\\
$\nu$ 		& 1 	& 1 	& 1 	& 1 	& 1 	& 1 	& 1 	& 1 	& 1 	& 1 	& - 	& - 	& - 	& - 	& -\\
$T$ 		& 2 	& 2 	& 2 	& 2 	& 15 	& 2 	& 2 	& 2 	& 2 	& 15 	& 1 	& 3 	& 3 	& 1 	& 15\\
$k$ 		& 40 	& 50 	& 50 	& 50 	& - 	& 50 	& 50 	& 50 	& 50 	& - 	& 60 	& 60 	& 100 	& 60 	& -\\
$\sigma_G$ 	& 126.9 & 29.0 	& 11.3 	& 5.86 	& - 	& 125.9 & 27.8 	& 15.5 	& 7.5 	& - 	& 64.0 	& 20.2 	& 14.0 	& 3.5 	& -\\
$\sigma_g$ 	& 63.4 	& 14.5 	& 11.3 	& 5.86 	& - 	& 63.0 	& 13.9 	& 7.7 	& 3.8 	& - 	& 64.0 	& 20.2 	& 14.0 	& 3.5 	& -\\
$\beta$ 	& 1 	& 1 	& 1 	& 1 	& - 	& 0.05 	& 0.4 	& 0.5 	& 0.6 	& - 	& 0.05 	& 0.1 	& 0.05 	& 0.05 	& -\\
$\sigma_p$ 	& 200 	& 200 	& 20 	& 20 	& - 	& 100 	& 20 	& 10 	& 10 	& - 	& 100 	& 100 	& 100 	& 100 	& -\\
\bottomrule
\end{tabular}
\end{table}

\subsection{Standard Deviation}
Finally, Table~\ref{tab:stddev} reports the standard deviation of the \dpals metrics in Figure~\ref{fig:tradeoff}. For each data point, we repeat the experiment 20 times, using the same set of hyper-parameters selected on the validation set, and report the mean and standard deviation of the metric measured on the test set. In all cases, the standard deviation is less than 0.5\% of the mean.

\begin{table}[H]
\footnotesize
\caption{Mean and standard deviation of the \dpals metrics in Figure~\ref{fig:tradeoff}. \label{tab:stddev}}
\centering
\begin{tabular}{l | c c c c | c c c c | c c c c }
\toprule
 & \multicolumn{4}{c|}{ML-10M-top400 (test RMSE)} & \multicolumn{4}{c|}{ML-10M (test RMSE)} & \multicolumn{4}{c}{ML-20M (test Recall@20)} \\
\midrule
$\epsilon$ 	& 0.8 	& 4 	& 8 	& 16 	& 1 	& 5 	& 10 	& 20 	& 1 	& 5 	& 10 	& 20\\
mean    & 0.8855 	& 0.8321 	& 0.8201 	& 0.8147 	& 0.9398 	& 0.8725 	& 0.8530 	& 0.8373 	& 0.3120 	& 0.3330 	& 0.3368 	& 0.3444\\
stddev  & 0.0025 	& 0.0009 	& 0.0011 	& 0.0008 	& 0.0009 	& 0.0006 	& 0.0004 	& 0.0005 	& 0.0016 	& 0.0010 	& 0.0012 	& 0.0013\\
\end{tabular}
\end{table}

\subsection{Additional Experiments}
\label{app:emp-experiments}

\mypar{Convergence plots for \dpals and DPFW} This experiment illustrates the fact that ALS converges faster than FW, both in its exact and private variants, making it more suitable for training private models. Figure~\ref{fig:app:convergence} shows the test error (RMSE) against number of iterations, on the synthetic data set with $n = 20K$ users. We use the vanilla version of \dpals without the heuristics introduced in Section~\ref{sec:practCons}. The hyper-parameters of both methods are tuned on the validation set.

For the non-private baselines, ALS converges significantly faster than FW. For example, the error of ALS after 2 iterations is lower than the error of FW after 40 iterations. For the private models, we compare the two methods with the same sampling rate ($k = 150$) and same noise level ($\sigma = 10$ in Figure~\ref{fig:app:convergence-a} and $\sigma = 20$ in Figure~\ref{fig:app:convergence-b}), and tune other parameters. Since the sampling rate and noise level are fixed, the $\epsilon$ level is directly determined by the number of steps, and the vertical lines show different levels of $\epsilon$.
We can make the following observations. For both methods, in the presence of noise, the error decreases for a few iterations at a rate similar to their exact variants, then plateaus at a fixed error. The fixed error for \dpals the is an order of magnitude lower than DPFW. Furthermore, the error reached by \dpals in 2 iterations is lower than the error of DPFW after 40 iterations. The faster convergence of \dpals, even in the presence of noise, directly translates to a better privacy/utility trade-off as demonstrated in Section~\ref{sec:empEval}.

\begin{figure}[H]
\centering
\begin{subfigure}[b]{0.40\textwidth}
\centering
\includegraphics[width=\textwidth]{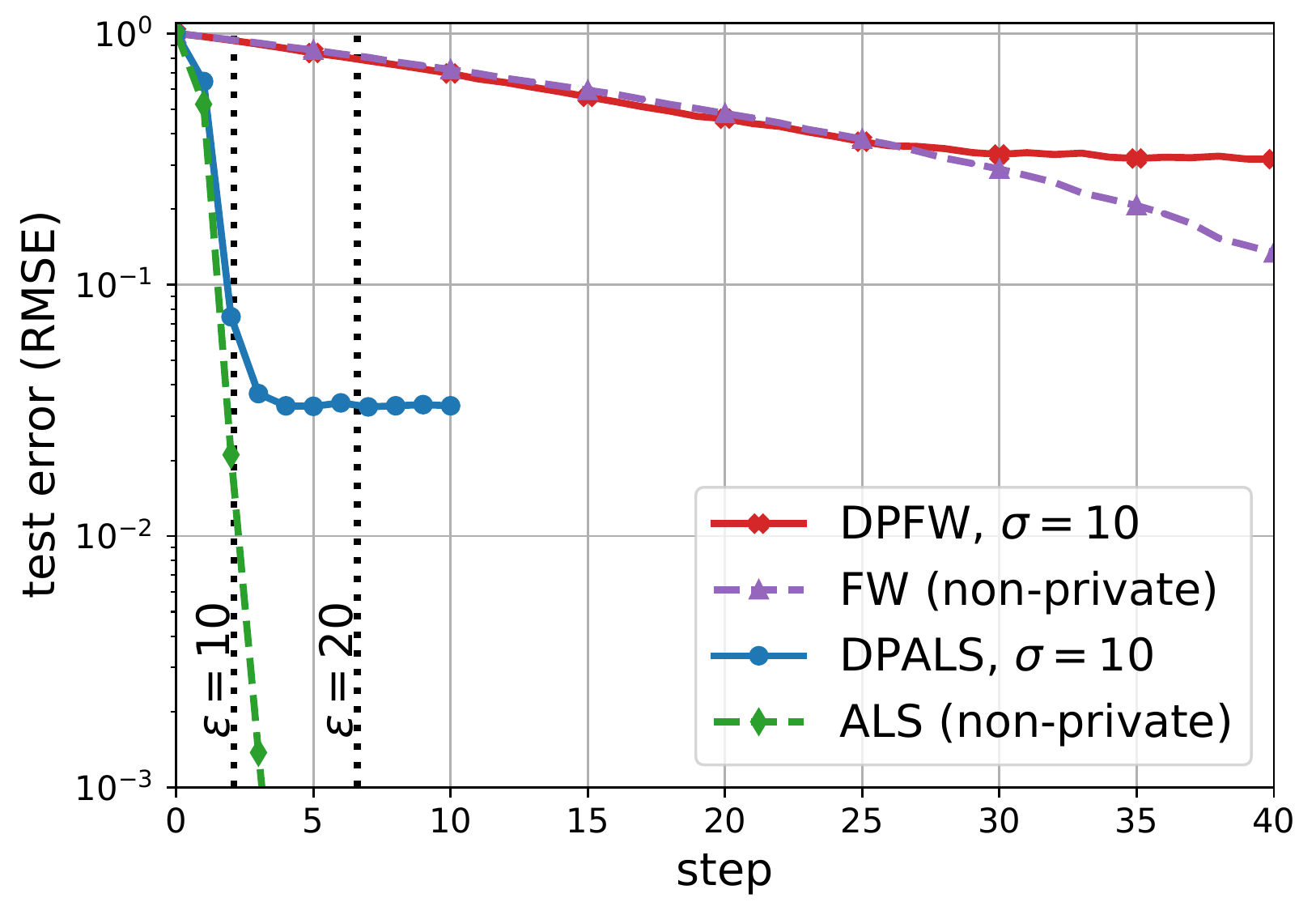}
\caption{$\sigma = 10$\label{fig:app:convergence-a}}
\end{subfigure}
\begin{subfigure}[b]{0.40\textwidth}
\centering
\includegraphics[width=\textwidth]{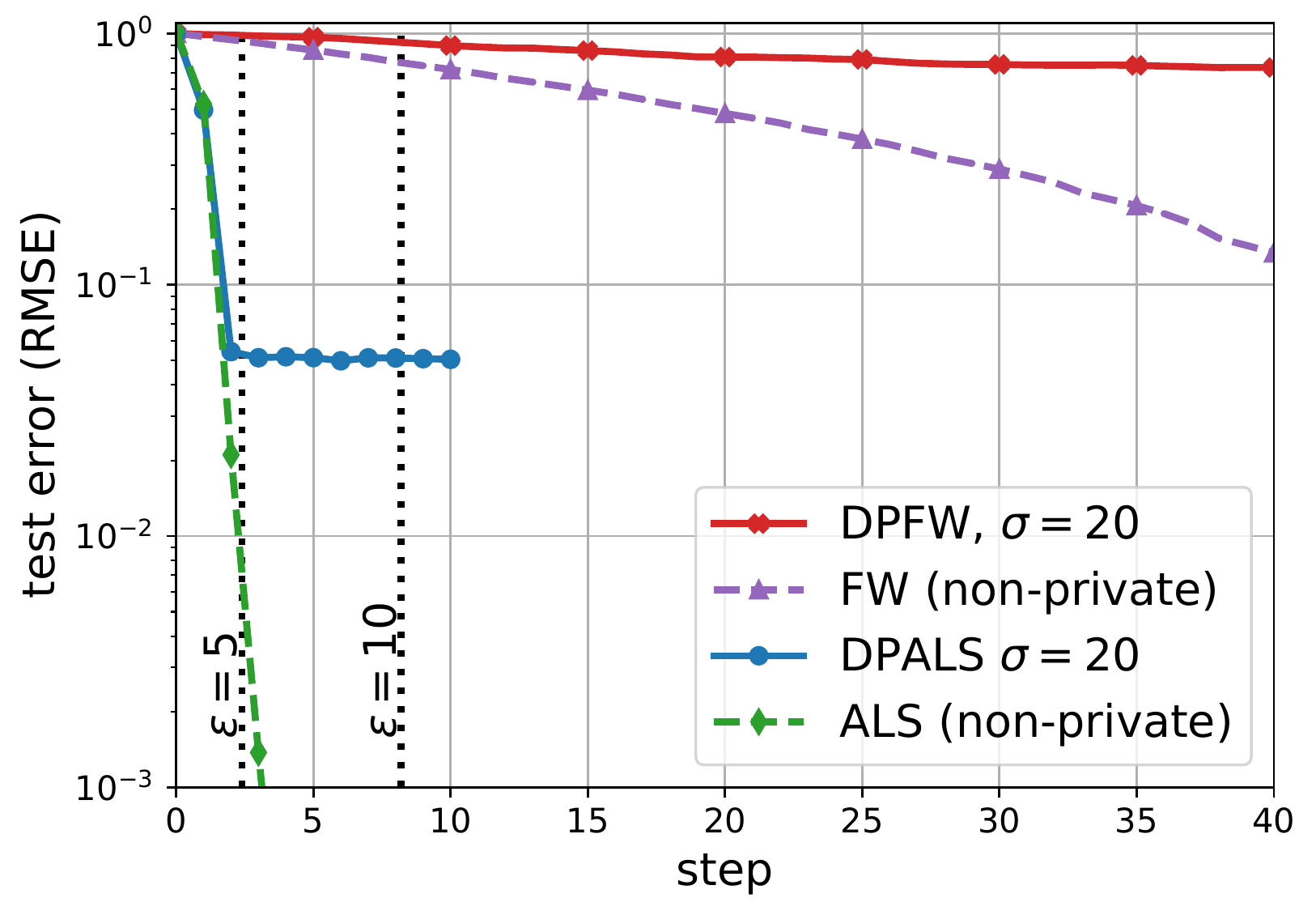}
\caption{$\sigma = 20$\label{fig:app:convergence-b}}
\end{subfigure}
\caption{RMSE against steps on the synthetic data set with $n = 20K$. Dashed lines correspond to the non-private baselines (without noise) and solid lines correspond to the private methods with a fixed noise level (left: $\sigma = 10$, right: $\sigma = 20$).}
\label{fig:app:convergence}
\end{figure}

\mypar{Varying the number of users} This experiment further illustrates the effect of increasing the number of users. We train the \dpals on a subset of the ML-10M-top400 data set, obtained by randomly sampling a subset of $n$ users. Figure~\ref{fig:app:ML10M_400_users} shows the results for different values of $n$, and confirms that increasing the number of users (while keeping the number of movies constant) improves the privacy/utility trade-off. The figure also compares to the DPFW baseline trained on the full data ($n = 69692$). Note that \dpals significantly outperforms \dpfw even when trained on a small fraction of the users ($n = 16000$, or 26.4\% of the total users).

\begin{figure}[H]
\centering
\includegraphics[width=0.40\textwidth]{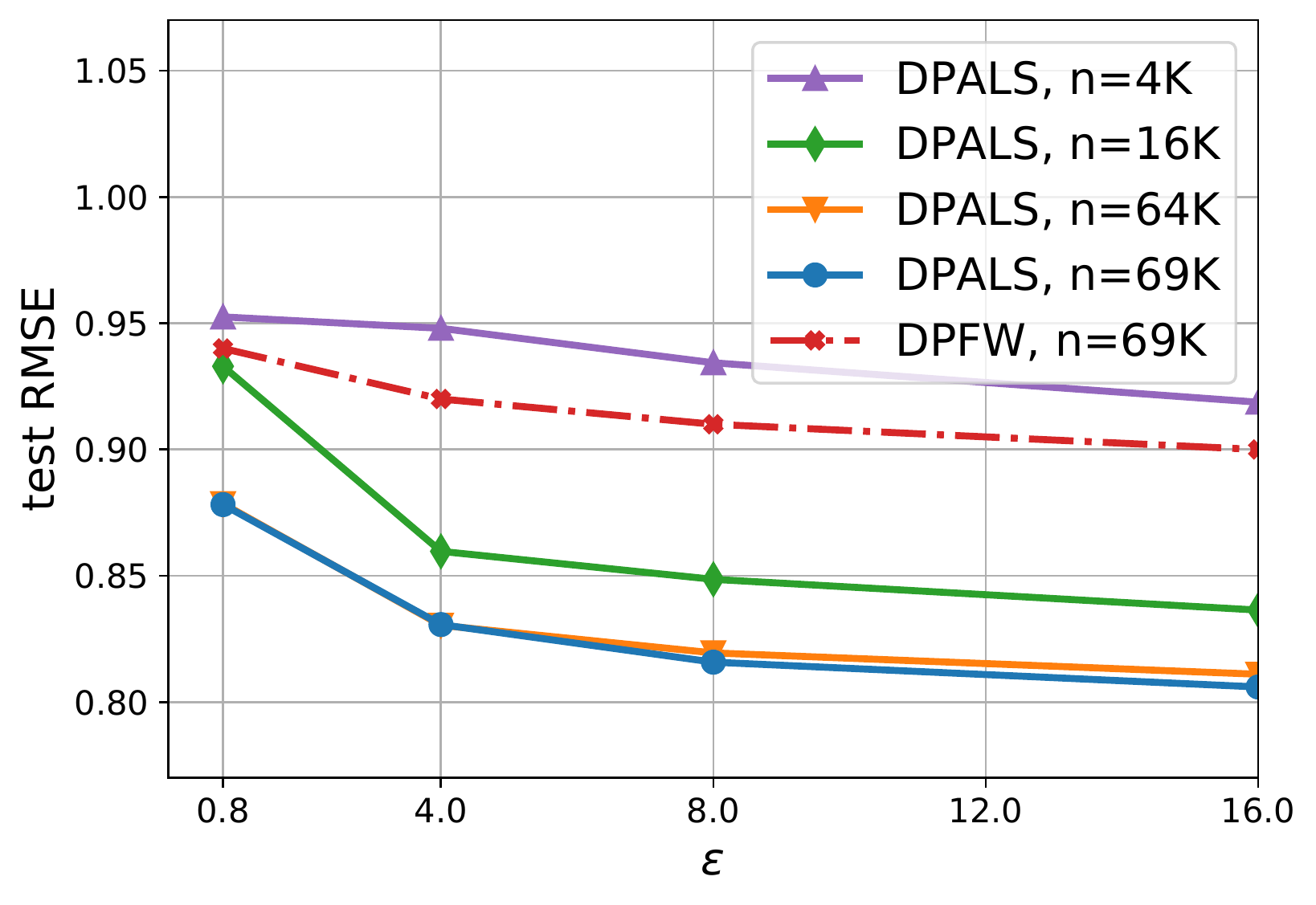}
\caption{\dpals on ML-10M-top400 with a varying number of users, $n$.}
\label{fig:app:ML10M_400_users}
\end{figure}

\mypar{Effect of the rank}
This experiment explores the effect of the rank on the privacy/utility tradeoff.
Figure~\ref{fig:app:effct_of_rank} shows the trade-offs for models of different ranks $r$ on ML-10M and ML-20M. We observe that for non-private ALS, models of higher rank consistently achieve better performance in the range of ranks that we have tried. This is not always the case for the private models. 
For the ML-10M task, the higher rank model ($r=128$) performs well for larger values of $\epsilon$, but not for $\epsilon = 1$. 
On the ML-20M task, the private model with $r=128$ gives the best recall for $\epsilon\approx20$ while $r=32$ performs the best for smaller $\epsilon$.
Therefore, unlike in the non-private ALS algorithm where a higher rank is often more desirable given enough computational and storage resources, when training a private model, one needs to carefully choose the rank to balance model capacity and utility degradation due to privacy.

\begin{figure}[H]
\centering
\begin{subfigure}[b]{0.40\textwidth}
\centering
\includegraphics[width=\textwidth]{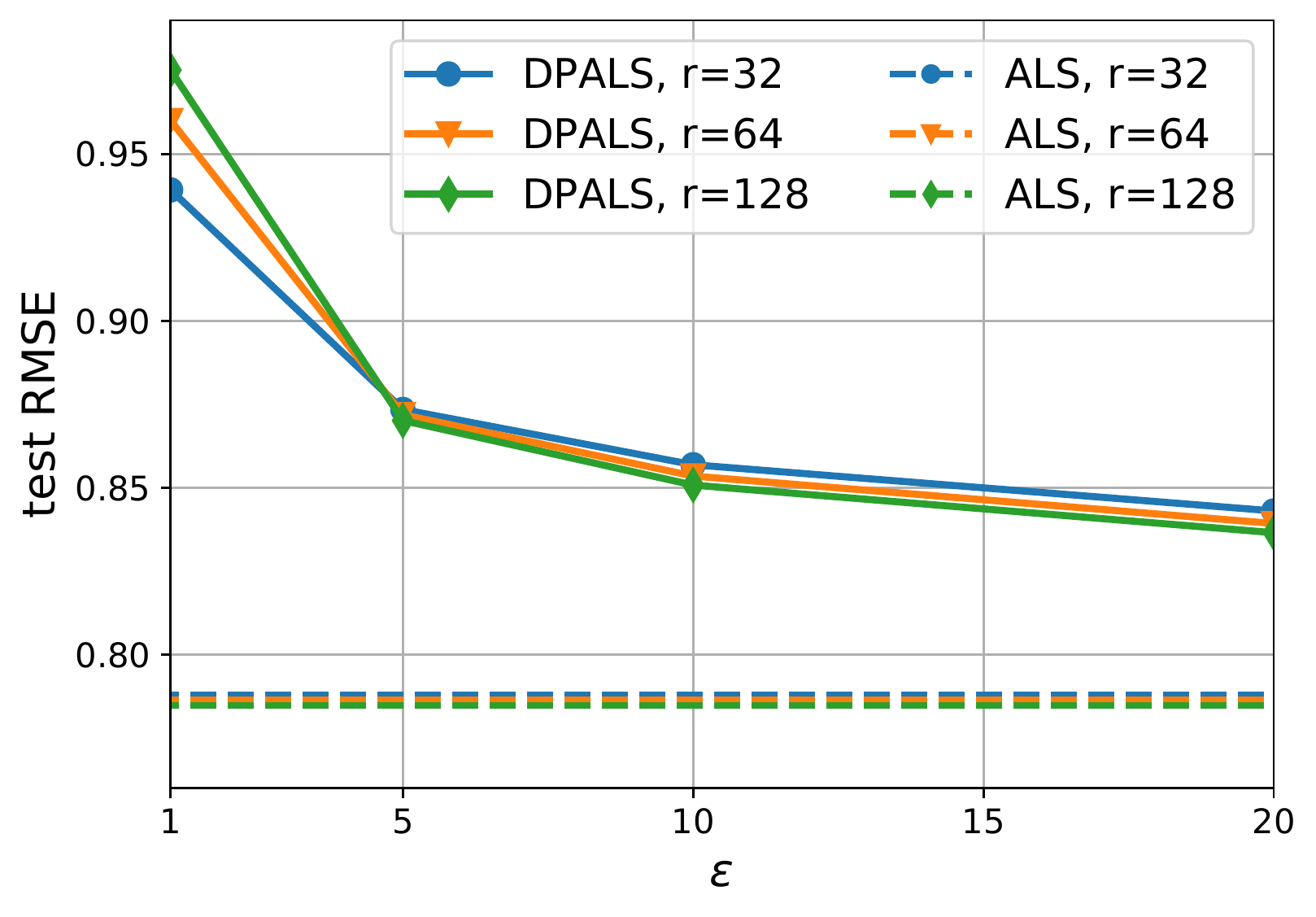}
\caption{RMSE on ML-10M (lower is better)}
\end{subfigure}
\begin{subfigure}[b]{0.40\textwidth}
\centering
\includegraphics[width=\textwidth]{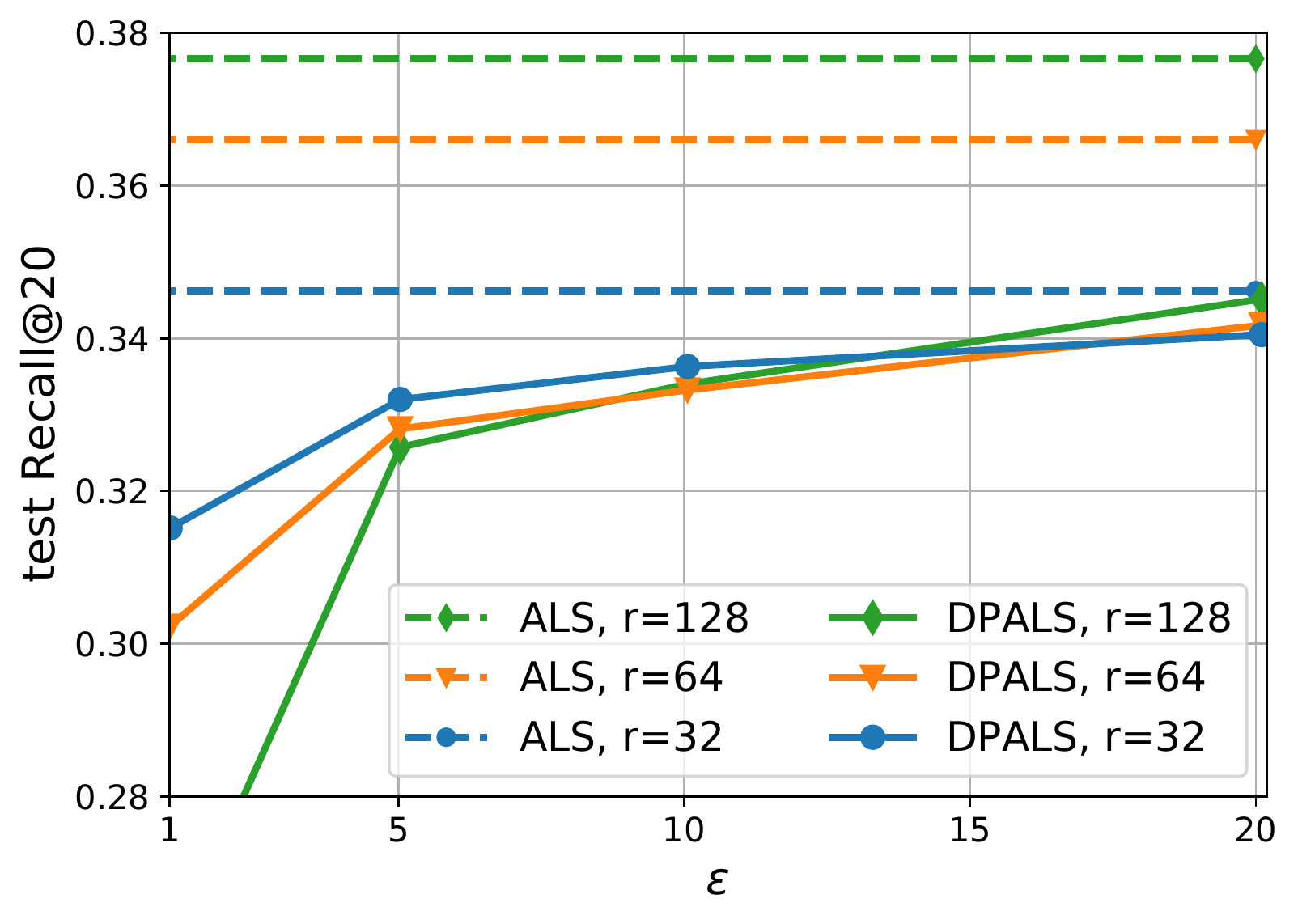}
\caption{Recall@20 on ML-20M (higher is better)}
\end{subfigure}
\caption{Privacy/utility trade-off for models of different ranks. For lower values of $\epsilon$, a lower rank achieves a better privacy/utility.}
\label{fig:app:effct_of_rank}
\end{figure}


\mypar{Training on $\head$ movies} This experiment illustrates the effect of partitioning the set movies into $(\head \sqcup \tail)$ and training only on $\head$ movies. Figure~\ref{fig:app:effect_of_frequent} shows the test RMSE vs movie fraction, at different levels of $\epsilon$. The rank of the model is fixed to $r = 32$, the sample size is fixed to $k = 50$, and other hyper-parameters are re-tuned. The results show that as $\epsilon$ decreases, the optimal fraction of movies decreases. In particular, for $\epsilon = 1$, the optimal fraction is 5\%; note however that this still corresponds to more than 50\% of the ratings, as shown on the right sub-figure.

\begin{figure}[H]
\centering
\begin{subfigure}[b]{0.40\textwidth}
\centering
\includegraphics[width=\textwidth]{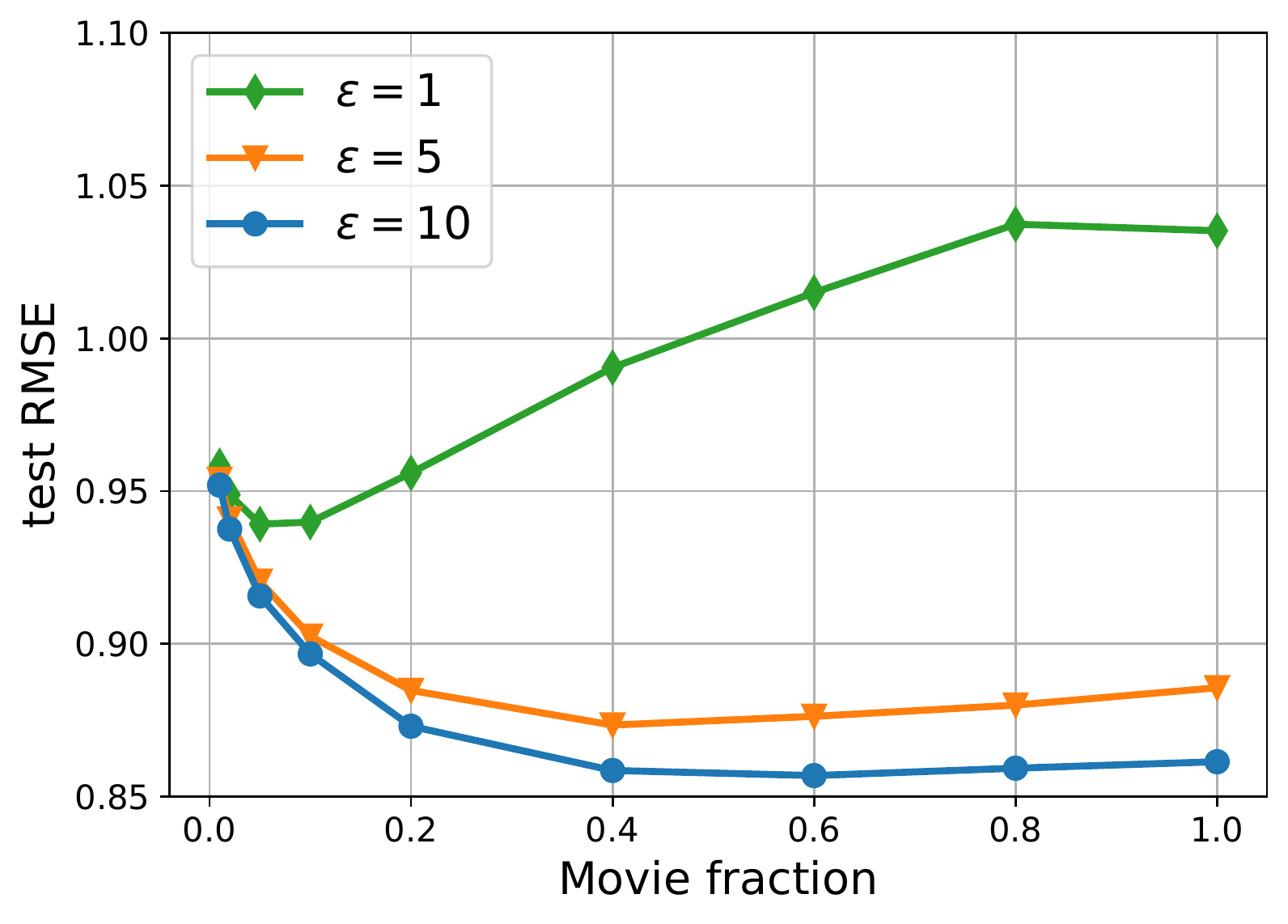}
\end{subfigure}
\begin{subfigure}[b]{0.40\textwidth}
\centering
\includegraphics[width=\textwidth]{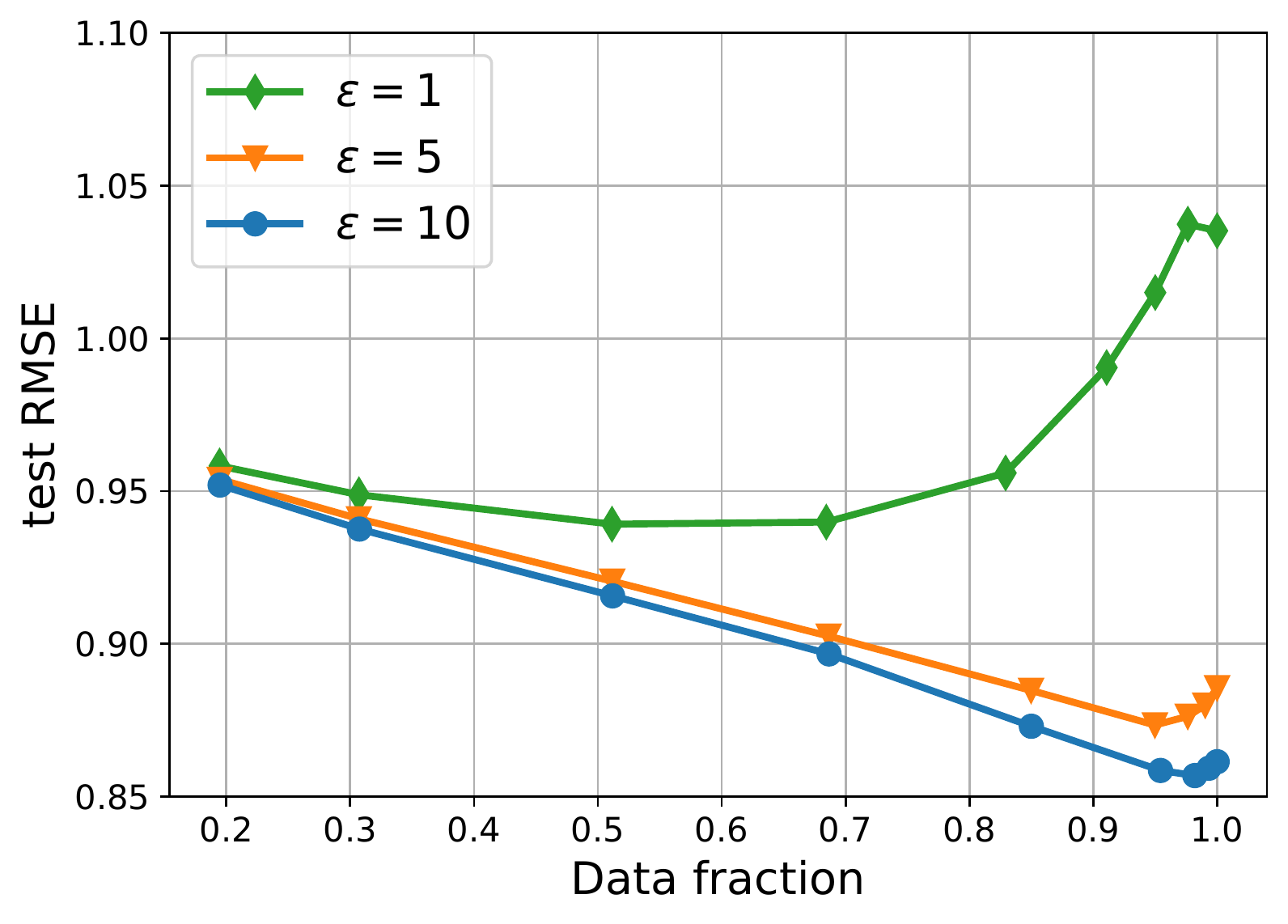}
\end{subfigure}
\caption{RMSE vs movie fraction, for a rank $32$ model on ML-10M, at different privacy levels $\epsilon$. Both figures show the same data, but with a different x axis. The movie fraction (left figure) is defined as $|\head|/m$. The data fraction (right figure) is defined as $|\{(i, j) \in \Omega : j \in \head \}| / |\Omega|$. The right figure emphasizes the long-tail distribution of movie counts -- a small fraction of $\head$ movies corresponds to a large fraction of data.}
\label{fig:app:effect_of_frequent}
\end{figure}

Figure~\ref{fig:app:effect_of_frequent2} shows a similar result for ML-20M. The optimal movie fraction in this example is between 5\% and 10\% depending on the rank.

\begin{figure}[h!]
\centering
\includegraphics[width=0.4\textwidth]{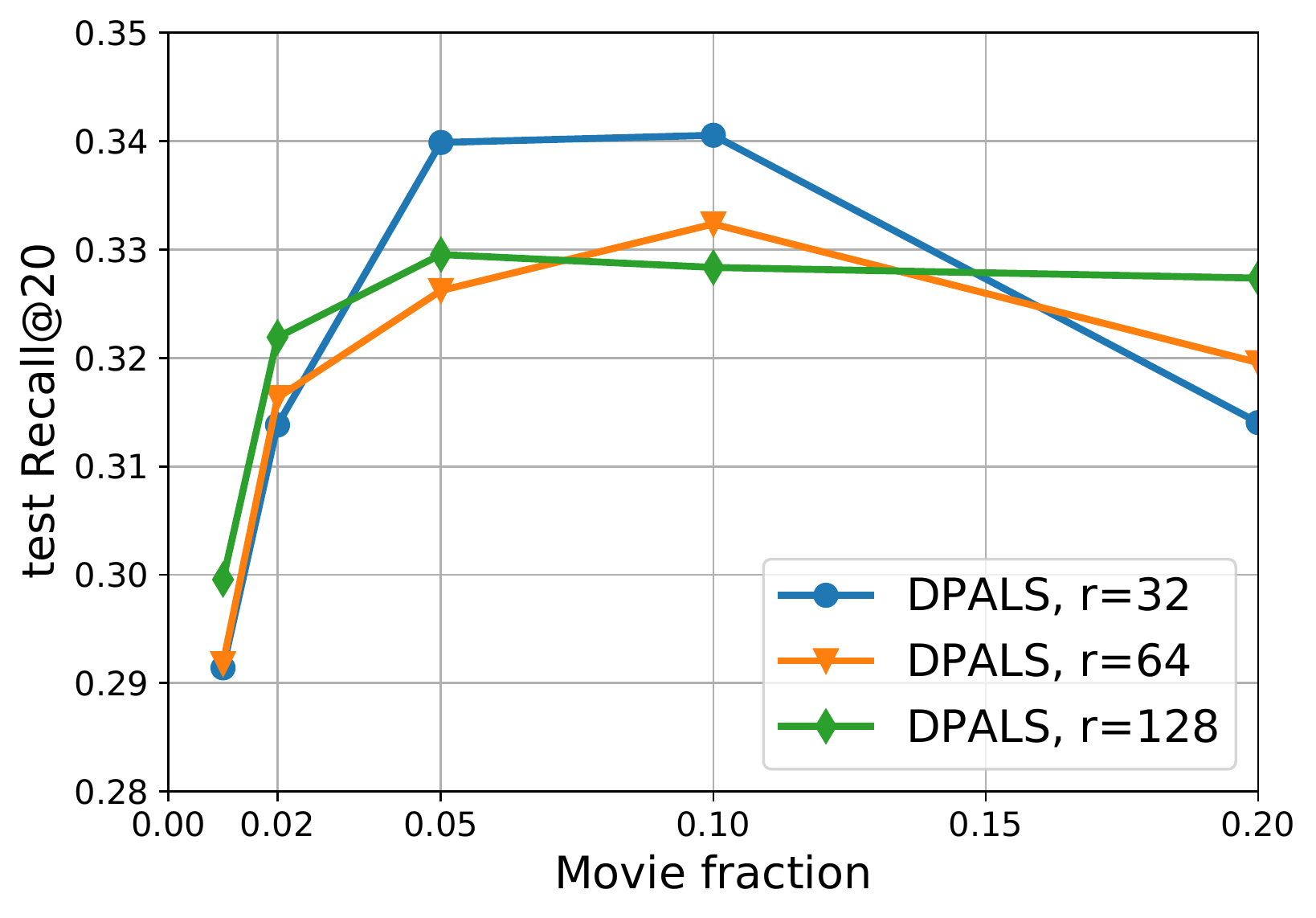}
\caption{Recall@20 vs movie fraction on ML-20M, for $\epsilon = 5$.}
\label{fig:app:effect_of_frequent2}
\end{figure}

\newpage
\mypar{Effect of the regularization exponents}
This experiment illustrates the effect of the regularization exponents $(\nu, \mu)$ in the loss function~\eqref{eq:objective}. We vary $(\nu, \mu)$ for a rank $128$ model with $\epsilon = 10$ on ML-10M (and re-tune other parameters). The results are reported in Figure~\ref{fig:app:effect_of_reg}. This example indicates that a careful tuning of the $\ell_2$ regularization can have a significant impact on utility, and can also make the private models more robust to noise: Notice that with the optimal setting of $(\nu, \mu)$ the model can be trained on a much larger fraction of movies, with only a slight degradation in utility.

\begin{figure}[h!]
\centering
\hspace{.8in}\includegraphics[width=0.54\textwidth]{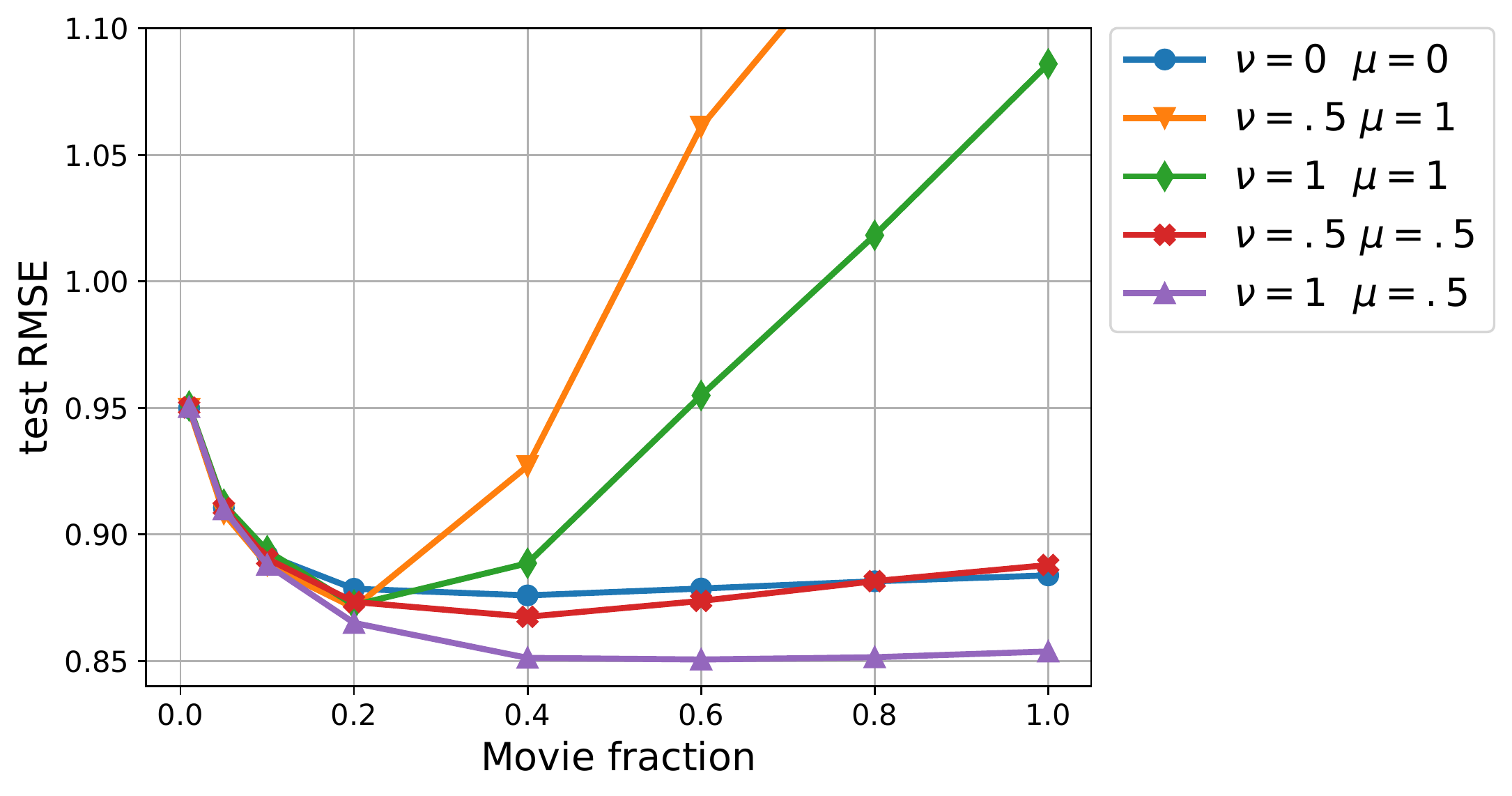}
\caption{RMSE vs movie fraction on ML-10M, for $\epsilon = 10$ and $r = 128$, and for different values of regularization exponents $(\nu, \mu)$.}
\label{fig:app:effect_of_reg}
\end{figure}

%% file: main.bbl
\begin{thebibliography}{54}
\providecommand{\natexlab}[1]{#1}
\providecommand{\url}[1]{\texttt{#1}}
\expandafter\ifx\csname urlstyle\endcsname\relax
  \providecommand{\doi}[1]{doi: #1}\else
  \providecommand{\doi}{doi: \begingroup \urlstyle{rm}\Url}\fi

\bibitem[Abadi et~al.(2016)Abadi, Chu, Goodfellow, McMahan, Mironov, Talwar,
  and Zhang]{abadi2016deep}
Abadi, M., Chu, A., Goodfellow, I., McMahan, H.~B., Mironov, I., Talwar, K.,
  and Zhang, L.
\newblock Deep learning with differential privacy.
\newblock In \emph{Proceedings of the 2016 ACM SIGSAC conference on computer
  and communications security}, pp.\  308--318, 2016.

\bibitem[Bassily et~al.(2014)Bassily, Smith, and Thakurta]{BST14}
Bassily, R., Smith, A., and Thakurta, A.
\newblock Private empirical risk minimization: Efficient algorithms and tight
  error bounds.
\newblock In \emph{Proc. of the 2014 IEEE 55th Annual Symp. on Foundations of
  Computer Science (FOCS)}, 2014.

\bibitem[Bhatia(2013)]{bhatia}
Bhatia, R.
\newblock \emph{Matrix analysis}, volume 169.
\newblock Springer Science \& Business Media, 2013.

\bibitem[Calandrino et~al.(2011)Calandrino, Kilzer, Narayanan, Felten, and
  Shmatikov]{calandrino2011you}
Calandrino, J.~A., Kilzer, A., Narayanan, A., Felten, E.~W., and Shmatikov, V.
\newblock ``you might also like:'' privacy risks of collaborative filtering.
\newblock In \emph{2011 IEEE symposium on security and privacy}, pp.\
  231--246. IEEE, 2011.

\bibitem[Cand{\`e}s \& Recht(2009)Cand{\`e}s and Recht]{candesrecht}
Cand{\`e}s, E.~J. and Recht, B.
\newblock Exact matrix completion via convex optimization.
\newblock \emph{Foundations of Computational mathematics}, 9\penalty0
  (6):\penalty0 717--772, 2009.

\bibitem[Carlini et~al.(2019)Carlini, Liu, Erlingsson, Kos, and
  Song]{carlini2019secret}
Carlini, N., Liu, C., Erlingsson, {\'U}., Kos, J., and Song, D.
\newblock The secret sharer: Evaluating and testing unintended memorization in
  neural networks.
\newblock In \emph{28th {USENIX} Security Symposium ({USENIX} Security 19)},
  pp.\  267--284, 2019.

\bibitem[Carlini et~al.(2020{\natexlab{a}})Carlini, Deng, Garg, Jha,
  Mahloujifar, Mahmoody, Song, Thakurta, and Tramer]{carlini2020attack}
Carlini, N., Deng, S., Garg, S., Jha, S., Mahloujifar, S., Mahmoody, M., Song,
  S., Thakurta, A., and Tramer, F.
\newblock An attack on instahide: Is private learning possible with instance
  encoding?
\newblock \emph{arXiv preprint arXiv:2011.05315}, 2020{\natexlab{a}}.

\bibitem[Carlini et~al.(2020{\natexlab{b}})Carlini, Tramer, Wallace, Jagielski,
  Herbert-Voss, Lee, Roberts, Brown, Song, Erlingsson,
  et~al.]{carlini2020extracting}
Carlini, N., Tramer, F., Wallace, E., Jagielski, M., Herbert-Voss, A., Lee, K.,
  Roberts, A., Brown, T., Song, D., Erlingsson, U., et~al.
\newblock Extracting training data from large language models.
\newblock \emph{arXiv preprint arXiv:2012.07805}, 2020{\natexlab{b}}.

\bibitem[Dekel et~al.(2011)Dekel, Lee, and Linial]{dekel2011eigenvectors}
Dekel, Y., Lee, J.~R., and Linial, N.
\newblock Eigenvectors of random graphs: Nodal domains.
\newblock \emph{Random Structures \& Algorithms}, 39\penalty0 (1):\penalty0
  39--58, 2011.

\bibitem[Dinur \& Nissim(2003)Dinur and Nissim]{dinur2003revealing}
Dinur, I. and Nissim, K.
\newblock Revealing information while preserving privacy.
\newblock In \emph{Proceedings of the twenty-second ACM SIGMOD-SIGACT-SIGART
  symposium on Principles of database systems}, pp.\  202--210, 2003.

\bibitem[Dwork \& Roth(2014)Dwork and Roth]{dwork2014algorithmic}
Dwork, C. and Roth, A.
\newblock The algorithmic foundations of differential privacy.
\newblock \emph{Foundations and Trends in Theoretical Computer Science},
  9\penalty0 (3--4):\penalty0 211--407, 2014.

\bibitem[Dwork et~al.(2006{\natexlab{a}})Dwork, Kenthapadi, McSherry, Mironov,
  and Naor]{ODO}
Dwork, C., Kenthapadi, K., McSherry, F., Mironov, I., and Naor, M.
\newblock Our data, ourselves: Privacy via distributed noise generation.
\newblock In \emph{Advances in Cryptology---EUROCRYPT}, pp.\  486--503,
  2006{\natexlab{a}}.

\bibitem[Dwork et~al.(2006{\natexlab{b}})Dwork, McSherry, Nissim, and
  Smith]{DMNS}
Dwork, C., McSherry, F., Nissim, K., and Smith, A.
\newblock Calibrating noise to sensitivity in private data analysis.
\newblock In \emph{Proc. of the Third Conf. on Theory of Cryptography (TCC)},
  pp.\  265--284, 2006{\natexlab{b}}.

\bibitem[Dwork et~al.(2007)Dwork, McSherry, and Talwar]{dwork2007price}
Dwork, C., McSherry, F., and Talwar, K.
\newblock The price of privacy and the limits of lp decoding.
\newblock In \emph{Proceedings of the thirty-ninth annual ACM Symposium on
  Theory of Computing}, pp.\  85--94, 2007.

\bibitem[Dwork et~al.(2014)Dwork, Talwar, Thakurta, and
  Zhang]{dwork2014analyze}
Dwork, C., Talwar, K., Thakurta, A., and Zhang, L.
\newblock Analyze gauss: optimal bounds for privacy-preserving principal
  component analysis.
\newblock In \emph{Proceedings of the forty-sixth annual ACM symposium on
  Theory of computing}, pp.\  11--20, 2014.

\bibitem[Erd{\H{o}}s et~al.(2013)Erd{\H{o}}s, Knowles, Yau, and
  Yin]{erdHos2013spectral}
Erd{\H{o}}s, L., Knowles, A., Yau, H.-T., and Yin, J.
\newblock Spectral statistics of erd{\H{o}}s--r{\'e}nyi graphs i: Local
  semicircle law.
\newblock \emph{The Annals of Probability}, 41\penalty0 (3B):\penalty0
  2279--2375, 2013.

\bibitem[Ge et~al.(2016)Ge, Lee, and Ma]{ge2016matrix}
Ge, R., Lee, J.~D., and Ma, T.
\newblock Matrix completion has no spurious local minimum.
\newblock \emph{Advances in Neural Information Processing Systems}, pp.\
  2981--2989, 2016.

\bibitem[Hardt \& Price(2013)Hardt and Price]{hardt2013noisy}
Hardt, M. and Price, E.
\newblock The noisy power method: A meta algorithm with applications.
\newblock \emph{arXiv preprint arXiv:1311.2495}, 2013.

\bibitem[Hardt \& Roth(2012)Hardt and Roth]{hardt2012beating}
Hardt, M. and Roth, A.
\newblock Beating randomized response on incoherent matrices.
\newblock In \emph{Proceedings of the forty-fourth annual ACM symposium on
  Theory of computing}, pp.\  1255--1268, 2012.

\bibitem[Hardt \& Roth(2013)Hardt and Roth]{hardt2013beyond}
Hardt, M. and Roth, A.
\newblock Beyond worst-case analysis in private singular vector computation.
\newblock In \emph{Proceedings of the forty-fifth annual ACM symposium on
  Theory of computing}, pp.\  331--340, 2013.

\bibitem[Hardt \& Wootters(2014)Hardt and Wootters]{hardt2014fast}
Hardt, M. and Wootters, M.
\newblock Fast matrix completion without the condition number.
\newblock In \emph{Conference on learning theory}, pp.\  638--678. PMLR, 2014.

\bibitem[Hardt et~al.(2014)Hardt, Meka, Raghavendra, and
  Weitz]{hardt2014computational}
Hardt, M., Meka, R., Raghavendra, P., and Weitz, B.
\newblock Computational limits for matrix completion.
\newblock In \emph{Conference on Learning Theory}, pp.\  703--725. PMLR, 2014.

\bibitem[Harper \& Konstan(2016)Harper and Konstan]{harper16movielens}
Harper, F.~M. and Konstan, J.~A.
\newblock The movielens datasets: History and context.
\newblock \emph{Acm Transactions on Interactive Intelligent Systems (TiiS)},
  5\penalty0 (4):\penalty0 19, 2016.

\bibitem[Hu et~al.(2008)Hu, Koren, and Volinsky]{hu2008ials}
Hu, Y., Koren, Y., and Volinsky, C.
\newblock Collaborative filtering for implicit feedback datasets.
\newblock In \emph{Proceedings of the 2008 Eighth IEEE International Conference
  on Data Mining}, ICDM '08, pp.\  263--272, 2008.

\bibitem[Jain \& Netrapalli(2015)Jain and Netrapalli]{jain2015fast}
Jain, P. and Netrapalli, P.
\newblock Fast exact matrix completion with finite samples.
\newblock In \emph{Conference on Learning Theory}, pp.\  1007--1034. PMLR,
  2015.

\bibitem[Jain et~al.(2013)Jain, Netrapalli, and Sanghavi]{jain2013low}
Jain, P., Netrapalli, P., and Sanghavi, S.
\newblock Low-rank matrix completion using alternating minimization.
\newblock In \emph{Proceedings of the forty-fifth annual ACM symposium on
  Theory of computing}, pp.\  665--674, 2013.

\bibitem[Jain et~al.(2018)Jain, Thakkar, and Thakurta]{jain2018differentially}
Jain, P., Thakkar, O.~D., and Thakurta, A.
\newblock Differentially private matrix completion revisited.
\newblock In \emph{International Conference on Machine Learning}, pp.\
  2215--2224. PMLR, 2018.

\bibitem[Kearns et~al.(2014)Kearns, Pai, Roth, and Ullman]{kearns2014mechanism}
Kearns, M., Pai, M., Roth, A., and Ullman, J.
\newblock Mechanism design in large games: Incentives and privacy.
\newblock In \emph{Proceedings of the 5th conference on Innovations in
  theoretical computer science}, pp.\  403--410, 2014.

\bibitem[Koren \& Bell(2015)Koren and Bell]{koren}
Koren, Y. and Bell, R.
\newblock Advances in collaborative filtering.
\newblock \emph{Recommender systems handbook}, pp.\  77--118, 2015.

\bibitem[Koren et~al.(2009)Koren, Bell, and Volinsky]{koren2009matrix}
Koren, Y., Bell, R., and Volinsky, C.
\newblock Matrix factorization techniques for recommender systems.
\newblock \emph{Computer}, 42\penalty0 (8):\penalty0 30--37, 2009.

\bibitem[Korolova(2010)]{korolova2010privacy}
Korolova, A.
\newblock Privacy violations using microtargeted ads: A case study.
\newblock In \emph{2010 IEEE International Conference on Data Mining
  Workshops}, pp.\  474--482. IEEE, 2010.

\bibitem[Lee et~al.(2013)Lee, Kim, Lebanon, and Singer]{lee13llorma}
Lee, J., Kim, S., Lebanon, G., and Singer, Y.
\newblock Local low-rank matrix approximation.
\newblock In \emph{Proceedings of the 30th International Conference on
  International Conference on Machine Learning - Volume 28}, ICML'13, pp.\
  II--82--II--90. JMLR.org, 2013.

\bibitem[Liang et~al.(2018)Liang, Krishnan, Hoffman, and Jebara]{liang18vae}
Liang, D., Krishnan, R.~G., Hoffman, M.~D., and Jebara, T.
\newblock Variational autoencoders for collaborative filtering.
\newblock WWW '18, pp.\  689–698, 2018.

\bibitem[Liu \& Talwar(2019)Liu and Talwar]{liu2019private}
Liu, J. and Talwar, K.
\newblock Private selection from private candidates.
\newblock In \emph{Proceedings of the 51st Annual ACM SIGACT Symposium on
  Theory of Computing}, pp.\  298--309, 2019.

\bibitem[Liu et~al.(2015)Liu, Wang, and Smola]{liu2015fast}
Liu, Z., Wang, Y.-X., and Smola, A.
\newblock Fast differentially private matrix factorization.
\newblock In \emph{Proceedings of the 9th ACM Conference on Recommender
  Systems}, pp.\  171--178, 2015.

\bibitem[Lu et~al.(2019)Lu, Hong, and Wang]{lu19pagd}
Lu, S., Hong, M., and Wang, Z.
\newblock {PA}-{GD}: On the convergence of perturbed alternating gradient
  descent to second-order stationary points for structured nonconvex
  optimization.
\newblock In \emph{Proceedings of the 36th International Conference on Machine
  Learning}, pp.\  4134--4143, 2019.

\bibitem[Marlin et~al.(2007)Marlin, Zemel, Roweis, and
  Slaney]{marlin07collaborative}
Marlin, B.~M., Zemel, R.~S., Roweis, S., and Slaney, M.
\newblock Collaborative filtering and the missing at random assumption.
\newblock In \emph{Proceedings of the Twenty-Third Conference on Uncertainty in
  Artificial Intelligence}, UAI'07, pp.\  267–275, Arlington, Virginia, USA,
  2007. AUAI Press.

\bibitem[McSherry \& Mironov(2009)McSherry and
  Mironov]{mcsherry2009differentially}
McSherry, F. and Mironov, I.
\newblock Differentially private recommender systems: Building privacy into the
  netflix prize contenders.
\newblock In \emph{Proceedings of the 15th ACM SIGKDD international conference
  on Knowledge discovery and data mining}, pp.\  627--636, 2009.

\bibitem[Meng et~al.(2018)Meng, Wang, Shu, Li, Chen, Liu, and
  Zhang]{meng2018personalized}
Meng, X., Wang, S., Shu, K., Li, J., Chen, B., Liu, H., and Zhang, Y.
\newblock Personalized privacy-preserving social recommendation.
\newblock In \emph{Proceedings of the AAAI Conference on Artificial
  Intelligence}, volume~32, 2018.

\bibitem[Mironov(2017)]{mironov2017renyi}
Mironov, I.
\newblock R{\'e}nyi differential privacy.
\newblock In \emph{2017 IEEE 30th Computer Security Foundations Symposium
  (CSF)}, pp.\  263--275. IEEE, 2017.

\bibitem[Recht(2011)]{recht2011simpler}
Recht, B.
\newblock A simpler approach to matrix completion.
\newblock \emph{Journal of Machine Learning Research}, 12\penalty0 (12), 2011.

\bibitem[Rendle et~al.(2019)Rendle, Zhang, and Koren]{rendle2019baselines}
Rendle, S., Zhang, L., and Koren, Y.
\newblock On the difficulty of evaluating baselines: {A} study on recommender
  systems.
\newblock \emph{CoRR}, abs/1905.01395, 2019.

\bibitem[Rudelson \& Vershynin(2015)Rudelson and
  Vershynin]{rudelson2015delocalization}
Rudelson, M. and Vershynin, R.
\newblock Delocalization of eigenvectors of random matrices with independent
  entries.
\newblock \emph{Duke Mathematical Journal}, 164\penalty0 (13):\penalty0
  2507--2538, 2015.

\bibitem[Sheffet(2019)]{sheffet2019old}
Sheffet, O.
\newblock Old techniques in differentially private linear regression.
\newblock In \emph{Algorithmic Learning Theory}, pp.\  789--827. PMLR, 2019.

\bibitem[Shenbin et~al.(2020)Shenbin, Alekseev, Tutubalina, Malykh, and
  Nikolenko]{shenbin20recvae}
Shenbin, I., Alekseev, A., Tutubalina, E., Malykh, V., and Nikolenko, S.~I.
\newblock Recvae: A new variational autoencoder for top-n recommendations with
  implicit feedback.
\newblock In \emph{Proceedings of the 13th International Conference on Web
  Search and Data Mining}, WSDM '20, pp.\  528–536, 2020.

\bibitem[Shokri et~al.(2017)Shokri, Stronati, Song, and
  Shmatikov]{shokri2017membership}
Shokri, R., Stronati, M., Song, C., and Shmatikov, V.
\newblock Membership inference attacks against machine learning models.
\newblock In \emph{2017 IEEE Symposium on Security and Privacy (SP)}, pp.\
  3--18. IEEE, 2017.

\bibitem[Smith et~al.(2017)Smith, Thakurta, and Upadhyay]{smith2017interaction}
Smith, A., Thakurta, A., and Upadhyay, J.
\newblock Is interaction necessary for distributed private learning?
\newblock In \emph{2017 IEEE Symposium on Security and Privacy (SP)}, pp.\
  58--77. IEEE, 2017.

\bibitem[Song et~al.(2013)Song, Chaudhuri, and Sarwate]{song2013stochastic}
Song, S., Chaudhuri, K., and Sarwate, A.~D.
\newblock Stochastic gradient descent with differentially private updates.
\newblock In \emph{2013 IEEE Global Conference on Signal and Information
  Processing}, pp.\  245--248. IEEE, 2013.

\bibitem[{Sun} \& {Luo}(2015){Sun} and {Luo}]{sun2015guaranteed}
{Sun}, R. and {Luo}, Z.
\newblock Guaranteed matrix completion via nonconvex factorization.
\newblock In \emph{FOCS}, 2015.

\bibitem[Thakkar et~al.(2020)Thakkar, Ramaswamy, Mathews, and
  Beaufays]{thakkar2020understanding}
Thakkar, O., Ramaswamy, S., Mathews, R., and Beaufays, F.
\newblock Understanding unintended memorization in federated learning.
\newblock \emph{arXiv preprint arXiv:2006.07490}, 2020.

\bibitem[Tropp(2015)]{troppmatrix}
Tropp, J.~A.
\newblock An introduction to matrix concentration inequalities.
\newblock \emph{arXiv preprint arXiv:1501.01571}, 2015.

\bibitem[Vershynin(2010)]{vershynin}
Vershynin, R.
\newblock Introduction to the non-asymptotic analysis of random matrices.
\newblock \emph{arXiv preprint arXiv:1011.3027}, 2010.

\bibitem[Vu \& Wang(2015)Vu and Wang]{vu2015random}
Vu, V. and Wang, K.
\newblock Random weighted projections, random quadratic forms and random
  eigenvectors.
\newblock \emph{Random Structures \& Algorithms}, 47\penalty0 (4):\penalty0
  792--821, 2015.

\bibitem[Zhu \& Wang(2020)Zhu and Wang]{zhu2020improving}
Zhu, Y. and Wang, Y.-X.
\newblock Improving sparse vector technique with renyi differential privacy.
\newblock \emph{Advances in Neural Information Processing Systems}, 33, 2020.

\end{thebibliography}
